\documentclass{article}

\PassOptionsToPackage{numbers, compress}{natbib}
\usepackage[main, final]{neurips_2026}

\usepackage[utf8]{inputenc} % allow utf-8 input
\usepackage[T1]{fontenc}    % use 8-bit T1 fonts
\usepackage{hyperref}       % hyperlinks
\usepackage{url}            % simple URL typesetting
\usepackage{booktabs}       % professional-quality tables
\usepackage{tabularx}       % auto-wrapping table columns
\usepackage{amsfonts}       % blackboard math symbols
\usepackage{nicefrac}       % compact symbols for 1/2, etc.
\usepackage{microtype}      % microtypography
\usepackage{xcolor}         % colors

\usepackage{amsmath, amssymb, amsthm}
\usepackage[capitalize,noabbrev]{cleveref}
\usepackage{thmtools, thm-restate}  % restatable propositions (LoMAP style)
\usepackage{algorithm, algorithmic} % algorithm pseudocode
\usepackage{enumitem}               % flexible list formatting
\usepackage{graphicx}               % includegraphics
\usepackage{subcaption}             % subfigure environment
\usepackage{multirow}               % multi-row table cells
\usepackage{pifont}                 % checkmarks (\ding)
\usepackage{makecell}               % multi-line table headers
\usepackage{wrapfig}                % wrap text around figures/tables
\usepackage{caption}                % \captionof for mixed float environments

\newcommand{\cmark}{\ding{51}}%
\newcommand{\xmark}{\ding{55}}%

\newcommand{\valstd}[2]{$#1${\tiny $\pm #2$}}

\newcommand{\NoNumber}[1]{{\def\alglinenumber##1{}\STATE #1}\addtocounter{ALC@line}{-1}}

\newcommand*\diff{\mathop{}\!\mathrm{d}}

\newcommand{\bs}{\boldsymbol}
\newcommand{\E}{\mathbb{E}}
\newcommand{\R}{\mathbb{R}}
\let\cite\citep

\title{Refining Compositional Diffusion \\for Reliable Long-Horizon Planning}

\author{%
  Kyowoon Lee$^{1}$ \quad Yunhao Luo$^{2}$ \quad Anh Tong$^{3}$ \quad Jaesik Choi$^{4,5}$ \\
  \\
  $^{1}$~Ajou University \quad $^{2}$~University of Michigan \quad $^{3}$~Korea University \\
  $^{4}$~KAIST \quad $^{5}$~INEEJI
}

\begin{document}

\maketitle

\begin{abstract}
	Compositional diffusion planning generates long-horizon trajectories by stitching together overlapping short-horizon segments through score composition. However, when local plan distributions are multimodal, existing compositional methods suffer from mode-averaging, where averaging incompatible local modes leads to plans that are neither locally feasible nor globally coherent. We propose \emph{Refining Compositional Diffusion} (RCD), a training-free guidance method that steers compositional sampling toward high-density, globally coherent plans. RCD leverages the \emph{self-reconstruction error} of a pretrained diffusion model as a proxy for the log-density of composed plans, combined with an \emph{overlap consistency} term that enforces consistency at segment boundaries. We show that the combined guidance concentrates sampling on high-density plans that mitigate mode-averaging. Experiments on challenging long-horizon tasks from OGBench, including locomotion, object manipulation, and pixel-based observations, demonstrate that RCD improves compositional planning, with the largest gains on long-horizon tasks where mode-averaging is severe. Project website at \textcolor{magenta}{\url{https://refining-compositional-diffusion.github.io/}}.
\end{abstract}

%=============================================================================
% INTRODUCTION
%=============================================================================
\section{Introduction}

Planning plays a fundamental role in sequential decision-making, enabling agents to reason about future consequences before committing to actions. Classical approaches such as Model Predictive Control~\cite{tassa2012synthesis} and Monte Carlo tree search~\cite{silver2016mastering, silver2017mastering, lee2018deep} achieve strong performance when accurate dynamics models are available, yet constructing such models for high-dimensional, continuous systems remains a significant challenge~\cite{janner2022planning}.
Diffusion probabilistic models have recently emerged as a compelling alternative for trajectory planning, learning to sample directly from the distribution of feasible plans and incorporating task-specific guidance through auxiliary gradient signals~\cite{janner2022planning, ajay2023is, liang2023adaptdiffuser}. However, the effectiveness of a diffusion planner is inherently constrained by the horizon of available training data. Collecting long-horizon demonstrations covering all possible start-to-goal combinations is prohibitively expensive, motivating \emph{compositional} approaches that assemble long plans from short, reusable segments~\cite{mishra2023generative, zhang2023diffcollage, luo2025generative, mishra2026compositional}.

Compositional approaches address this limitation by decomposing a long trajectory into overlapping local segments, each modeled by a single short-horizon diffusion model, and composing their score functions at inference time to sample from an approximation of the joint distribution~\cite{du2020compositional, du2023reduce, zhang2023diffcollage}. This strategy is data efficient and modular, allowing a model trained on short data to plan over arbitrarily long horizons. GSC~\cite{mishra2023generative} pioneered this for trajectory planning by chaining short-horizon diffusion trajectories via score composition. In practice, this composition averages the noise predictions of adjacent segments at every denoising step on their shared overlap regions. CompDiffuser~\cite{luo2025generative} additionally introduced a bidirectional conditioning mechanism that propagates information between adjacent chunks during denoising. However, when the local plan distribution is \emph{multimodal}, this averaging-based composition suffers from \emph{mode-averaging}~\cite{mishra2026compositional}. Adjacent segments independently commit to incompatible modes, and their averaged scores produce composed trajectories that lie in low-density regions of the true distribution. While CDGS~\cite{mishra2026compositional} mitigates this issue through population-based search and DDIM-inversion-based pruning~\cite{song2021denoising}, it incurs substantial inference-time overhead because many candidate trajectories must be repeatedly resampled, ranked, and filtered during the denoising process. More importantly, it improves final outcomes primarily through post hoc candidate selection, rather than by directly steering the denoising process toward more coherent plans.

In this paper, we propose \textbf{R}efining \textbf{C}ompositional \textbf{D}iffusion (\textbf{RCD}), a training-free guidance method that steers compositional sampling toward high-density regions of the local plan distributions while maintaining consistency across overlapping boundaries. We first show that the \emph{self-reconstruction error}, the discrepancy between a predicted clean sample and its reconstruction after partial renoising, provides a differentiable proxy for the local log-density. High-density samples reconstruct faithfully, whereas off-manifold trajectories, including mode-averaging artifacts, exhibit large reconstruction error. We then introduce an \emph{overlap consistency} term that penalizes score disagreement between adjacent segments at shared boundaries. Computed from pretrained local models alone, the combined signal directly steers denoising toward a tilted distribution that concentrates on high-density, globally coherent plans, mitigating mode-averaging without the population-based resampling and ranking overhead of prior search-based approaches, making RCD a plug-and-play component universally applicable to diffusion models.

Our main contributions are as follows:
\textbf{(1)} We analyze the self-reconstruction error as a density proxy for composed trajectories and show its formal connection to the diffusion Evidence Lower Bound (ELBO), complemented by an overlap consistency term that measures score disagreement at segment boundaries.
\textbf{(2)} We propose \emph{Refining Compositional Diffusion} (RCD), a training-free guidance method that combines these signals to steer compositional sampling toward high-density, globally coherent plans using only a pretrained local diffusion model.
\textbf{(3)} We evaluate RCD on OGBench long-horizon tasks spanning locomotion, object manipulation, and pixel-based observations, where it consistently improves success rates over prior compositional methods while running an order of magnitude faster than search-based alternatives, all without additional training or architectural changes to the local diffusion model.

%=============================================================================
% BACKGROUND
%=============================================================================
\section{Background}
\label{sec:background}

\subsection{Planning with Diffusion Models}
\label{sec:diffusion_planning}

Diffusion probabilistic models~\cite{ho2020denoising, song2021scorebased} learn a data distribution by defining a forward process that gradually adds Gaussian noise and a reverse process that recovers the clean data.
Given a clean trajectory $\bs{\tau}^{(0)} \sim p(\bs{\tau})$, the forward process produces a sequence of increasingly noisy versions $\bs{\tau}^{(1)}, \ldots, \bs{\tau}^{(T)}$ via
\begin{align}
	\label{eq:ddpm_forward}
	\bs{\tau}^{(t)} = \sqrt{\alpha_t}\, \bs{\tau}^{(0)} + \sqrt{1 - \alpha_t}\, \bs{\epsilon}, \quad \bs{\epsilon} \sim \mathcal{N}(\mathbf{0}, \mathbf{I}),
\end{align}
where $\alpha_1 > \alpha_2 > \cdots > \alpha_T \approx 0$ is a monotonically decreasing noise schedule.
The reverse process recovers $\bs{\tau}^{(0)}$ by learning a noise-prediction network $\bs{\epsilon}_\theta(\bs{\tau}^{(t)}, t)$ trained to minimize $\E[\| \bs{\epsilon} - \bs{\epsilon}_\theta(\bs{\tau}^{(t)}, t) \|^2]$.
Sampling proceeds iteratively from $\bs{\tau}^{(T)} \sim \mathcal{N}(\mathbf{0}, \mathbf{I})$:
\begin{align}
	\label{eq:ddpm_reverse}
	\bs{\tau}^{(t-1)} = \bs{\mu}_\theta(\bs{\tau}^{(t)}, t) + \sigma_t\, \mathbf{z}, \quad \mathbf{z} \sim \mathcal{N}(\mathbf{0}, \mathbf{I}),
\end{align}
where $\bs{\mu}_\theta$ is the predicted mean derived from $\bs{\epsilon}_\theta$ and $\sigma_t$ is the reverse-process standard deviation.
Diffuser~\cite{janner2022planning} applies this framework to trajectory planning by training a diffusion model over trajectories $\bs{\tau} = (\mathbf{s}_0, \mathbf{a}_0, \mathbf{s}_1, \mathbf{a}_1, \ldots)$ from offline data.
Task-specific objectives such as goal-reaching or return maximization are incorporated via gradient-based guidance at inference time \cite{dhariwal2021diffusion}.

\subsection{Compositional Diffusion via Factor Graphs}
\label{sec:comp_diffusion}
Consider a global plan $\bs{\tau} = (x_1, \ldots, x_N) \in \R^{N \times D}$ decomposed into $M$ overlapping local segments (factors) $y_j$, where each variable $x_i \in \R^D$ represents a state or state-action pair and adjacent segments $y_j$ and $y_{j+1}$ share boundary variables in the overlap $y_j \cap y_{j+1}$.
Following the Bethe approximation~\cite{yedidia2005constructing, mishra2026compositional}, the global distribution factorizes as:
\begin{align}
	\label{eq:bethe}
	p(\bs{\tau}) = \frac{\prod_{j=1}^{M} p(y_j)}{\prod_{i=1}^{N} p(x_i)^{d_i - 1}},
\end{align}
where $d_i$ is the degree of variable $x_i$ in the factor graph ($d_i = 2$ for overlap variables, $d_i = 1$ otherwise).
Taking log-gradients yields the compositional score:
\begin{align}
	\label{eq:comp_score}
	\nabla \log p(\bs{\tau}) = \sum_{j=1}^{M} \nabla \log p(y_j) + \sum_{i=1}^{N}(1 - d_i)\nabla \log p(x_i).
\end{align}
In practice, each $p(y_j)$ is represented by a shared diffusion model conditioned on the noisy states of its neighbors, $\bs{\epsilon}_\theta(y_j^{(t)}, t \mid y_{j-1}^{(t)}, y_{j+1}^{(t)})$~\cite{luo2025generative}. For notational simplicity, we write $\bs{\epsilon}_\theta(y_j^{(t)}, t)$ throughout, with the neighbor conditioning left implicit. The global score is computed by averaging noise predictions in overlapping regions~\cite{mishra2023generative, zhang2023diffcollage, mishra2026compositional}.
Bethe marginal consistency requires $p_{y_j}(x_i) = p_{y_{j+1}}(x_i)$ for all shared variables $x_i \in y_j \cap y_{j+1}$.
When local distributions are multimodal, score averaging blends incompatible modes, pushing samples toward low-density inter-modal regions. This problem worsens with the number of composed segments.

\subsection{Tweedie's Formula for Denoising}
\label{sec:tweedie}

When samples are perturbed by Gaussian noise $\tilde{\bs{\tau}} \sim \mathcal{N}(\bs{\tau}, \sigma^2 \mathbf{I})$, Tweedie's formula~\cite{robbins1992empirical} provides the Bayes-optimal denoised estimate $\E[\bs{\tau} | \tilde{\bs{\tau}}] = \tilde{\bs{\tau}} + \sigma^2 \nabla_{\tilde{\bs{\tau}}} \log p(\tilde{\bs{\tau}})$.
Applying this to the diffusion forward process in \cref{eq:ddpm_forward}, we obtain the posterior mean~\cite{chung2023diffusion}:
\begin{align}
	\label{eq:tweedie_diffusion}
	\E[\bs{\tau}^{(0)} | \bs{\tau}^{(t)}]
	= \frac{1}{\sqrt{\alpha_t}} \bigl( \bs{\tau}^{(t)} + (1 - \alpha_t) \nabla_{\bs{\tau}^{(t)}} \log p(\bs{\tau}^{(t)}) \bigr)
	\approx \frac{1}{\sqrt{\alpha_t}} \bigl( \bs{\tau}^{(t)} - \sqrt{1 - \alpha_t}\, \bs{\epsilon}_\theta(\bs{\tau}^{(t)}, t) \bigr),
\end{align}
where the score function $\mathbf{s}_\theta$ is approximated by $\nabla_{\bs{\tau}^{(t)}} \log p(\bs{\tau}^{(t)}) \approx -\bs{\epsilon}_\theta(\bs{\tau}^{(t)}, t) / \sqrt{1 - \alpha_t}$.
The right-hand side of \cref{eq:tweedie_diffusion} defines the \emph{Tweedie denoised estimate} $\hat{\bs{\tau}}_0^{(t)}$, which serves as a one-step prediction of the clean trajectory from a noisy sample at timestep $t$.

% Training-free Diffusion Guidance moved to Appendix (see \cref{sec:tfg})

%=============================================================================
% REFINING COMPOSITIONAL DIFFUSION
%=============================================================================
\section{Refining Compositional Diffusion}
\label{sec:method}

\begin{figure}[t]
	\centering
	\newlength{\toyheight}\setlength{\toyheight}{3.5cm}%
	\begin{subfigure}[b]{0.25\textwidth}
		\centering
		\includegraphics[height=\toyheight]{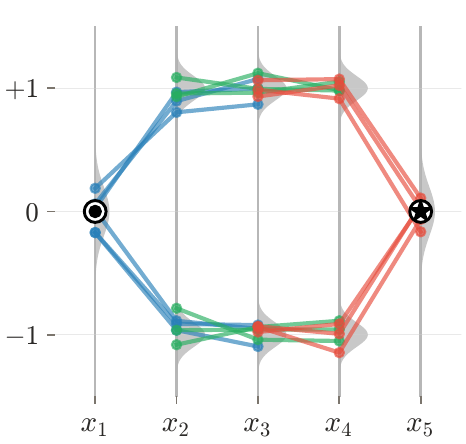}
		\caption{Training Segments}
		\label{fig:toy_a}
	\end{subfigure}\hfill
	\begin{subfigure}[b]{0.25\textwidth}
		\centering
		\includegraphics[height=\toyheight]{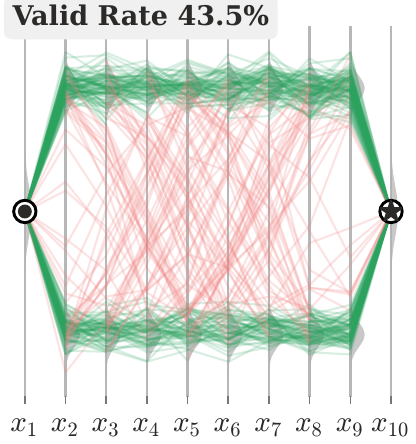}
		\caption{CompDiffuser}
		\label{fig:toy_b}
	\end{subfigure}\hfill
	\begin{subfigure}[b]{0.25\textwidth}
		\centering
		\includegraphics[height=\toyheight]{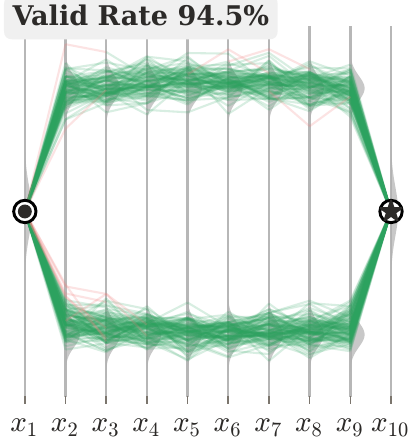}
		\caption{RCD \textbf{(ours)}}
		\label{fig:toy_c}
	\end{subfigure}\hfill
	\begin{subfigure}[b]{0.25\textwidth}
		\centering
		\includegraphics[height=\toyheight]{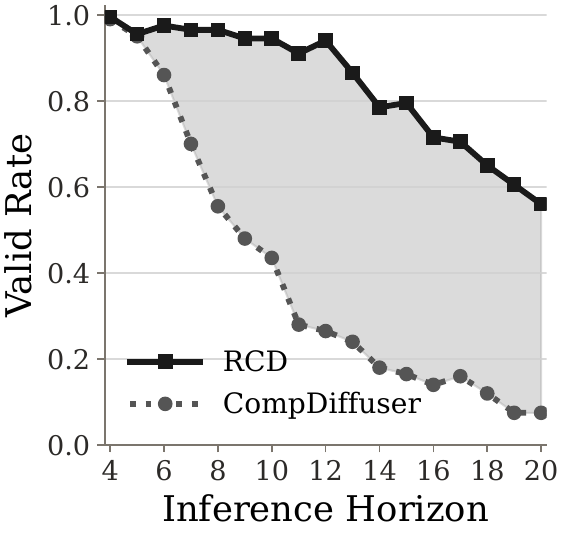}
		\caption{Valid Rate vs.\ Horizon}
		\label{fig:toy_d}
	\end{subfigure}
    \caption{\textbf{Toy illustration of the mode-averaging problem.} \textbf{(a)}~Training data consists of overlapping $l{=}3$ segments, each shown in a distinct color. They are anchored at a fixed start \protect\raisebox{-.08cm}{\includegraphics[height=.3cm]{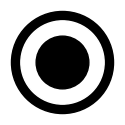}} and goal \protect\raisebox{-.08cm}{\includegraphics[height=.3cm]{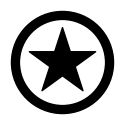}}, and pass through a bimodal distribution with modes at $+1$ and $-1$. \textbf{(b)}~CompDiffuser averages incompatible modes in overlap regions, producing many invalid (\textcolor[HTML]{f28482}{red}) trajectories off both modes. \textbf{(c)}~RCD guides the denoising toward high-density modes, yielding mostly valid (\textcolor[HTML]{2ca25f}{green}) trajectories. \textbf{(d)}~Valid rate drops sharply with horizon for CompDiffuser, while RCD maintains high validity.}
    \vspace{-0.2cm}
	\label{fig:toy_insight}
\end{figure}

% In this section, we present RCD, a training-free guidance method for compositional diffusion planning. We first illustrate the mode-averaging problem inherent in compositional score averaging (\cref{sec:problem}), then introduce the self-reconstruction error as a density proxy (\cref{sec:recon_error}) and the overlap consistency term for detecting local inconsistencies (\cref{sec:overlap}). We finally combine these signals into a unified guidance objective (\cref{sec:rcd_guidance}). An overview of RCD contrasted with existing compositional methods is provided in \cref{fig:overview}.

\subsection{Mode-Averaging Problem in Compositional Diffusion}
\label{sec:problem}

Compositional score averaging as described in \cref{sec:comp_diffusion} provides a reasonable approximation when the Bethe-style factorization is accurate and the composed score field faithfully represents the intended global distribution. In practice, however, both are only approximate, and the resulting artifacts are especially severe when $p(y)$ is multimodal, where incompatible local modes can be averaged together. Consider two adjacent segments $y_k$ and $y_{k+1}$ sharing overlap variables.
If $p(y)$ is bimodal, for instance the overlap variables can take values near $+1$ or $-1$ with equal probability, then $y_k$ and $y_{k+1}$ may independently commit to different modes.
When the scores from the two segments point in opposite directions, their average nearly cancels, leaving the overlap variables stranded in a low-density region between modes (near $0$).
The composed trajectory thus contains transitions that are individually implausible under any local model, leading to plan failure (\cref{fig:toy_b}). The probability of mode mismatch grows with the number of segments $M$.
Under idealized independent mode selection, the probability that all $M{-}1$ adjacent pairs agree on the same mode decays exponentially with $M$. In practice, iterative denoising introduces correlations that slow the decay, but the qualitative trend persists and the valid rate drops sharply as $M$ grows (\cref{fig:toy_d}). \cref{fig:overview} illustrates how RCD addresses this problem.

\subsection{Self-Reconstruction Error as a Density Proxy}
\label{sec:recon_error}

The core observation of RCD is that a pretrained diffusion model already provides an intrinsic density signal. Samples in high-density regions are faithfully reconstructed through a noise-denoise cycle, while mode-averaged samples in low-density regions are not. This \emph{self-reconstruction error} serves as a training-free density proxy for compositional guidance. Given a candidate clean trajectory $\hat{\bs{\tau}}_0$ obtained via the Tweedie estimate in \cref{eq:tweedie_diffusion}, we assess its quality by measuring how well it can be \emph{self-reconstructed} through the local models.

\begin{restatable}[Self-Reconstruction Error]{definition}{DefRecon}
	\label{def:recon}
	Given a candidate trajectory $\hat{\bs{\tau}}_0$, a shared diffusion model $\bs{\epsilon}_\theta$, and a probe noise level $s \in \{1, \ldots, T\}$, the self-reconstruction error is
	\begin{align}
		\mathcal{E}_{\mathrm{recon}}(\hat{\bs{\tau}}_0;\, s) = \E_{\bs{\epsilon}}\!\bigl[\, \| \hat{\bs{\tau}}_0 - \hat{\bs{\tau}}_{0}^{\mathrm{rec}}(\hat{\bs{\tau}}_0, \bs{\epsilon}, s) \|^2 \,\bigr],
		\label{eq:recon_error}
	\end{align}
	where $\hat{\bs{\tau}}_{0}^{\mathrm{rec}}(\hat{\bs{\tau}}_0, \bs{\epsilon}, s)$ is the composed Tweedie reconstruction obtained from the noised sample $\hat{\bs{\tau}}_s = \sqrt{\alpha_s}\, \hat{\bs{\tau}}_0 + \sqrt{1{-}\alpha_s}\, \bs{\epsilon}$: each local segment $\hat{y}_{j,s} = \hat{\bs{\tau}}_s[j]$ is independently denoised via $\hat{y}_{j,0} = (\hat{y}_{j,s} - \sqrt{1{-}\alpha_s}\, \bs{\epsilon}_\theta(\hat{y}_{j,s}, s)) / \sqrt{\alpha_s}$, and the per-segment estimates are merged by the Bethe composition rule in \cref{eq:bethe}.
\end{restatable}

The probe timestep $s$ controls the perturbation magnitude. In practice, we approximate the expectation with a single Monte Carlo sample. Intuitively, when $\hat{\bs{\tau}}_0$ lies near a mode of $p(\bs{\tau})$, the score network $\bs{\epsilon}_\theta$ accurately predicts the injected noise at level $s$, yielding a small reconstruction error.
Conversely, when $\hat{\bs{\tau}}_0$ falls between modes, as occurs under mode-averaging, the score prediction is biased toward a neighboring mode rather than recovering the original inter-modal input, resulting in large reconstruction error.
The following proposition formalizes this connection.

\begin{restatable}[Reconstruction Error as Density Proxy]{proposition}{PropDensity}
	\label{prop:density_proxy}
	Consider the self-reconstruction error $\mathcal{E}_{\mathrm{recon}}(\hat{\bs{\tau}}_0; s)$ with squared-error distance.
	Assume that each local denoiser $\bs{\epsilon}_\theta$ is trained via the standard DDPM objective.
	Then the weighted sum over probe timesteps satisfies
	\begin{align}
		\sum_{s=1}^T \frac{\alpha_s}{1-\alpha_s}\, \mathcal{E}_{\mathrm{recon}}(\hat{\bs{\tau}}_0;\, s)
		= \sum_{s=1}^T \E_{\bs{\epsilon}}\!\left[\, \bigl\| \bs{\epsilon} - \bar{\bs{\epsilon}}_\theta\bigl(\sqrt{\alpha_s}\, \hat{\bs{\tau}}_0 + \sqrt{1{-}\alpha_s}\, \bs{\epsilon},\, s\bigr) \bigr\|_2^2 \,\right]
		\geq -\log p_\theta(\hat{\bs{\tau}}_0),
		\label{eq:elbo_connection}
	\end{align}
	where $\bar{\bs{\epsilon}}_\theta$ is the composed noise prediction and $p_\theta(\hat{\bs{\tau}}_0)$ is the marginal likelihood.
\end{restatable}

The proof can be found in \cref{app:proof_density}. Since $\mathcal{E}_{\mathrm{recon}}$ at a single probe timestep $s$ is one non-negative term of this upper bound on $-\log p_\theta(\hat{\bs{\tau}}_0)$, minimizing it is directionally aligned with maximizing the composed log-density, providing a principled justification for using $\mathcal{E}_{\mathrm{recon}}$ as a guidance objective.

\begin{figure}[t]
	\centering
    \includegraphics[width=0.99\linewidth]{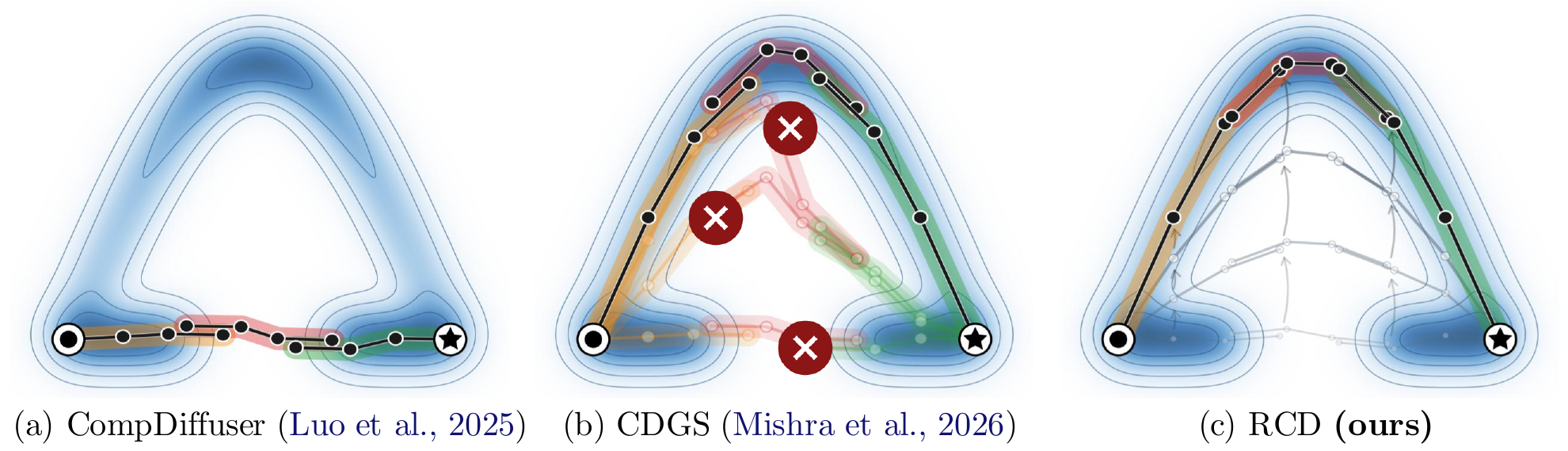}
    \caption{\textbf{Overview of RCD, contrasted with CompDiffuser and CDGS.} Trajectories are planned from start \protect\raisebox{-.08cm}{\includegraphics[height=.3cm]{figures/mark_start_crop.png}} to goal \protect\raisebox{-.08cm}{\includegraphics[height=.3cm]{figures/mark_goal_crop.png}} over a transition distribution \protect\raisebox{-.08cm}{\includegraphics[height=.3cm]{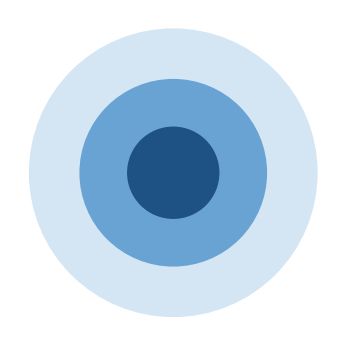}}. \textbf{(a)}~CompDiffuser composes overlapping segments via bidirectional conditioning and autoregressive denoising, but produces mode-averaged trajectories in low-density regions when local distributions are multimodal. \textbf{(b)}~CDGS mitigates this through population-based search and pruning at substantial inference cost. \textbf{(c)}~RCD iteratively refines each denoising step toward high-density, globally coherent plans using the self-reconstruction error and overlap consistency (\cref{sec:rcd_guidance}).}
	\label{fig:overview}
    \vspace{-0.2cm}
\end{figure}

\subsection{Overlap Consistency}
\label{sec:overlap}

The self-reconstruction error captures overall fidelity but may miss \emph{local inconsistencies} at segment boundaries.
Two adjacent segments may each reconstruct well individually while predicting different values for their shared variables $y_k \cap y_{k+1}$.
We introduce the overlap consistency term to detect and penalize such disagreements.

\begin{restatable}[Overlap Consistency]{definition}{DefOverlap}
	\label{def:overlap}
	Let $\hat{y}_{1,0}, \ldots, \hat{y}_{M,0}$ be the per-segment Tweedie estimates from the reconstruction step of \cref{def:recon}.
	The overlap consistency is
	\begin{align}
		\mathcal{E}_{\mathrm{ov}}(\hat{y}_{1:M}) := \frac{1}{M-1} \sum_{k=1}^{M-1} \bigl\| \hat{y}_{k,0}\big|_{y_k \cap y_{k+1}} - \hat{y}_{k+1,0}\big|_{y_k \cap y_{k+1}} \bigr\|_2^2,
		\label{eq:overlap}
	\end{align}
	where $\hat{y}_{j,0}\big|_{y_k \cap y_{k+1}}$ denotes the Tweedie estimate of segment $j$ restricted to the variables shared between segments $k$ and $k{+}1$.
\end{restatable}

The overlap consistency is zero when all segments agree perfectly in their shared regions, a necessary condition for the Bethe approximation to be self-consistent, and grows when independent denoising pushes adjacent segments toward incompatible modes.

\begin{restatable}[Overlap Consistency Measures Score Disagreement]{proposition}{PropOverlap}
	\label{prop:overlap}
	Let $\hat{y}_{k,0}\big|_{y_k \cap y_{k+1}}$ and $\hat{y}_{k+1,0}\big|_{y_k \cap y_{k+1}}$ be the per-segment Tweedie estimates restricted to the overlap region $y_k \cap y_{k+1}$.
	Since both segments receive the same noisy input in the overlap (extracted from $\hat{\bs{\tau}}_s$), the overlap mismatch decomposes as
	\begin{align}
		\bigl\| \hat{y}_{k,0}\big|_{y_k \cap y_{k+1}} - \hat{y}_{k+1,0}\big|_{y_k \cap y_{k+1}} \bigr\|^2
		= \frac{(1{-}\alpha_s)^2}{\alpha_s}\, \bigl\| \mathbf{s}_\theta^{(k)} - \mathbf{s}_\theta^{(k+1)} \bigr\|^2\big|_{y_k \cap y_{k+1}},
		\label{eq:overlap_as_score}
	\end{align}
    where $\mathbf{s}_\theta^{(j)} := -\bs{\epsilon}_\theta^{(j)} / \sqrt{1{-}\alpha_s}$ is the score from segment $j$.
\end{restatable}

The proof can be found in \cref{app:proof_overlap}. \cref{prop:overlap} reveals that the overlap consistency measures the \emph{score disagreement} between adjacent segments.
Bethe marginal consistency requires that both segments produce identical scores in their shared region, so $\mathbf{s}_\theta^{(k)} \neq \mathbf{s}_\theta^{(k+1)}$ signals a violation.
When they disagree, the overlap variables are pulled in different directions, a hallmark of mode-averaging.
Minimizing $\mathcal{E}_{\mathrm{ov}}$ thus encourages the marginal consistency implicit in the Bethe approximation, providing complementary information to the global reconstruction error.

\subsection{RCD Guidance for Compositional Diffusion}
\label{sec:rcd_guidance}

Combining the self-reconstruction error and overlap consistency, we define the RCD guidance objective:
\begin{align}
	\mathcal{E}_{\mathrm{RCD}}(\hat{\bs{\tau}}_0;\, s) := \mathcal{E}_{\mathrm{recon}}(\hat{\bs{\tau}}_0;\, s) + \lambda_{\mathrm{ov}}\, \mathcal{E}_{\mathrm{ov}}(\hat{y}_{1:M}),
	\label{eq:rcd_obj}
\end{align}
where $\lambda_{\mathrm{ov}} \geq 0$ weights the overlap consistency.
Following the training-free guidance framework (\cref{sec:tfg}), RCD treats $\mathcal{E}_{\mathrm{RCD}}$ as the loss function and applies its gradient to steer the reverse process. At each denoising step $t$, we first obtain the Tweedie estimate $\hat{\bs{\tau}}_0^{(t)}$ from the current noisy sample $\bs{\tau}^{(t)}$ via \cref{eq:tweedie_diffusion}, then compute the guidance gradient with respect to $\bs{\tau}^{(t)}$ through the Tweedie estimate:
\begin{align}
	\mathbf{g}^{(t)} = \nabla_{\bs{\tau}^{(t)}} \mathcal{E}_{\mathrm{RCD}}\!\bigl(\hat{\bs{\tau}}_0^{(t)};\, s\bigr).
	\label{eq:rcd_grad}
\end{align}
The gradient is normalized and applied as a correction to the standard reverse step~\cite{shen2024understanding}:
\begin{align}
	\bs{\tau}^{(t-1)} = \bs{\mu}_\theta(\bs{\tau}^{(t)}, t) + \sigma_t\, \mathbf{z} - w\, \sigma_t^2\, \tilde{\mathbf{g}}^{(t)}, \quad \tilde{\mathbf{g}}^{(t)} = \frac{\mathbf{g}^{(t)}}{\|\mathbf{g}^{(t)}\|_\infty + \delta},
	\label{eq:rcd_update}
\end{align}
where $w > 0$ is the guidance weight and $\delta > 0$ ensures numerical stability. The resulting guided reverse process targets a tilted distribution $\tilde{p}(\bs{\tau}) \propto p_\theta(\bs{\tau}) \cdot \exp(-w \cdot \mathcal{E}_{\mathrm{RCD}}(\bs{\tau}))$ that concentrates on high-density, globally coherent plans, with the guidance weight $w$ controlling the sharpness of this concentration.
The full procedure is summarized as pseudo-code in \cref{alg:rcd}, with further practical details provided in \cref{app:practical_alg}.

%=============================================================================
% EXPERIMENTS
%=============================================================================
\section{Experiments}
\label{sec:experiments}

In this section, we present the effectiveness of RCD on long-horizon planning tasks from OGBench~\cite{park2025ogbench}.
Specifically, we demonstrate \textbf{(1)} that RCD produces more physically feasible plans than existing compositional methods, \textbf{(2)} that it further enhances planning performance on locomotion, object manipulation, and pixel-based observation tasks, and \textbf{(3)} that the two guidance components are complementary.
Additional details regarding our experimental setup and implementation are provided in \cref{app:impl_detail}, and further results including additional ablation studies in \cref{app:add_results}.

% \vspace{-0.2cm}
\paragraph{Datasets and Environments.}

We evaluate on OGBench~\cite{park2025ogbench}, a benchmark for offline goal-conditioned reinforcement learning that provides diverse long-horizon tasks with varying difficulty levels.
Our evaluation covers three broad categories:
\textbf{(1) Locomotion:} \texttt{PointMaze}, \texttt{AntMaze}, and \texttt{HumanoidMaze} environments in \texttt{Stitch} datasets across \texttt{Medium}, \texttt{Large}, and \texttt{Giant} maze sizes. \textbf{(2) Object Manipulation:} \texttt{Cube} manipulation tasks (\texttt{Single}, \texttt{Double}, \texttt{Triple}, \texttt{Quadruple}) and \texttt{AntSoccer} (\texttt{Arena}, \texttt{Medium}).
\textbf{(3) Pixel-based Observations:} \texttt{AntMaze-Medium-Stitch} and \texttt{AntMaze-Large-Stitch} with visual observations, where planning operates in a learned latent space.
The \texttt{Stitch} datasets are constructed from short, unconnected trajectory segments that do not individually span start-to-goal pairs, making long-horizon planning critically dependent on compositional stitching. Further details on environments and evaluation protocol are provided in \cref{app:impl_detail}.

% \vspace{-0.2cm}
\paragraph{Baselines.}
We compare with offline goal-conditioned RL methods including goal-conditioned behavioral cloning (GCBC)~\cite{lynch2020learning}, goal-conditioned implicit V-learning (GCIVL) and Q-learning (GCIQL)~\cite{kostrikov2022offline}, Quasimetric RL (QRL)~\cite{wang2023optimal}, Contrastive RL (CRL)~\cite{eysenbach2022contrastive}, and Hierarchical implicit Q-learning (HIQL)~\cite{park2023hiql}. We also include diffusion-based generative planning baselines: Generative Skill Chaining (GSC)~\cite{mishra2023generative}, CompDiffuser (CD)~\cite{luo2025generative}, and CDGS~\cite{mishra2026compositional}. For CD, CDGS, and RCD, we use the same pretrained local diffusion model for fair comparison.

%-----------------------------------------------------------------------------
\subsection{Plan Feasibility Analysis}
\label{sec:plan_quality}

\begin{wrapfigure}{r}{7.0cm}
	\vspace{-0.5cm}
	\centering
	\includegraphics[width=1.0\linewidth]{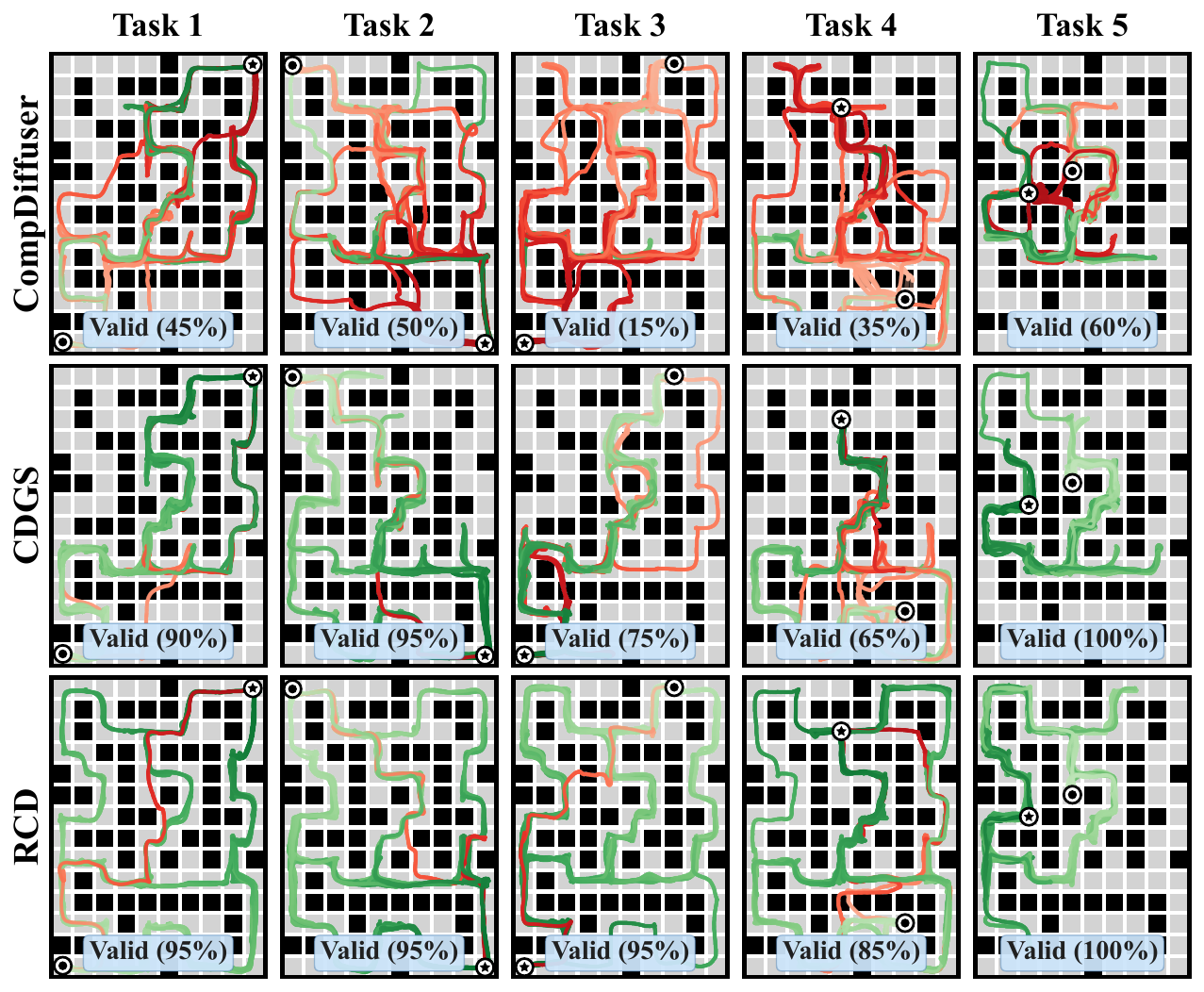}
	\caption{\textbf{Plan quality comparisons of CompDiffuser, CDGS, and RCD.} Each column shows 20 sampled plans in the \texttt{AntMaze-Giant-Stitch} environment for 5 test-time tasks defined in OGBench. Plans that violate environment constraints (wall penetration) are shown in \textcolor[HTML]{f28482}{red}; feasible plans in \textcolor[HTML]{2ca25f}{green}.}
	\label{fig:plan_quality}
	\vspace{-0.3cm}
\end{wrapfigure}

We first assess whether RCD improves the physical feasibility of composed plans.
\cref{fig:plan_quality} shows 20 sampled plans from CompDiffuser, CDGS, and RCD on each of the 5 test-time tasks in \texttt{AntMaze-Giant-Stitch}, together with the \emph{valid plan rate}: the fraction of plans that do not pass through maze walls.
CompDiffuser produces many infeasible plans, with valid rates ranging from 15\% to 60\% across tasks, as mode-averaging at corridor junctions pushes plans into physically impossible regions.
CDGS improves validity through population-based pruning, but remains inconsistent across tasks since pruning only filters already-corrupted trajectories post hoc, and tends to collapse plan diversity toward a single mode.
RCD achieves the highest valid rates across all 5 tasks by directly steering denoising toward high-density regions through the self-reconstruction error and overlap consistency, while preserving diverse plans across different feasible modes.
We observe similar trends on \texttt{PointMaze-Giant-Stitch}; see \cref{app:plan_quality_viz} for additional visualizations.

%-----------------------------------------------------------------------------
\subsection{Enhancing Long-Horizon Planning Performance}
\label{sec:main_results}

\begin{table}[t]
	\centering
	\caption{\textbf{Quantitative results on OGBench locomotion tasks.} Each environment provides 5 goal-conditioned tasks; we evaluate 20 episodes per task and report mean $\pm$ std across 5 seeds.}
	\label{tab:locomotion}
	\vspace{1mm}
	\resizebox{\textwidth}{!}{%
		\footnotesize
		\setlength{\tabcolsep}{4.5pt}
		\begin{tabular}{lll llllllllll}
			\toprule
			\textbf{Env}                           & \textbf{Type}                       & \textbf{Size}   & \textbf{GCBC}   & \textbf{GCIVL} & \textbf{GCIQL}  & \textbf{QRL}    & \textbf{CRL}   & \textbf{HIQL}            & \textbf{GSC}             & \textbf{CD}              & \textbf{CDGS}            & \textbf{RCD}             \\
			\midrule
			\multirow[c]{3}{*}{\texttt{pointmaze}} & \multirow[c]{3}{*}{\texttt{stitch}}
			                                       & \texttt{Medium}                     & \valstd{23}{18} & \valstd{70}{14} & \valstd{21}{9} & \valstd{80}{12} & \valstd{0}{1}   & \valstd{74}{6} & \valstd{\mathbf{100}}{0} & \valstd{\mathbf{100}}{0} & \valstd{\mathbf{100}}{0} & \valstd{\mathbf{100}}{0}                            \\
			                                       &                                     & \texttt{Large}  & \valstd{7}{5}   & \valstd{12}{6} & \valstd{31}{2}  & \valstd{84}{15} & \valstd{0}{0}  & \valstd{13}{6}           & \valstd{\mathbf{100}}{0} & \valstd{\mathbf{100}}{0} & \valstd{\mathbf{100}}{0} & \valstd{\mathbf{100}}{0} \\
			                                       &                                     & \texttt{Giant}  & \valstd{0}{0}   & \valstd{0}{0}  & \valstd{0}{0}   & \valstd{50}{8}  & \valstd{0}{0}  & \valstd{0}{0}            & \valstd{29}{3}           & \valstd{69}{3}           & \valstd{74}{3}           & \valstd{\mathbf{100}}{0} \\
			\midrule
			\multirow[c]{3}{*}{\texttt{antmaze}}   & \multirow[c]{3}{*}{\texttt{stitch}}
			                                       & \texttt{Medium}                     & \valstd{45}{11} & \valstd{44}{6}  & \valstd{29}{6} & \valstd{59}{7}  & \valstd{53}{6}  & \valstd{94}{1} & \valstd{\mathbf{97}}{2}  & \valstd{91}{2}           & \valstd{93}{3}           & \valstd{\mathbf{97}}{2}                             \\
			                                       &                                     & \texttt{Large}  & \valstd{3}{3}   & \valstd{18}{2} & \valstd{7}{2}   & \valstd{18}{2}  & \valstd{11}{2} & \valstd{67}{5}           & \valstd{66}{2}           & \valstd{89}{2}           & \valstd{63}{4}           & \valstd{\mathbf{91}}{2}  \\
			                                       &                                     & \texttt{Giant}  & \valstd{0}{0}   & \valstd{0}{0}  & \valstd{0}{0}   & \valstd{0}{0}   & \valstd{0}{0}  & \valstd{2}{2}            & \valstd{20}{1}           & \valstd{67}{3}           & \valstd{83}{3}           & \valstd{\mathbf{89}}{2}  \\
			\midrule
			\multirow[c]{3}{*}{\makecell{\texttt{humanoid}                                                                                                                                                                                                                                                                                \\\texttt{maze}}} & \multirow[c]{3}{*}{\texttt{stitch}}
			                                       & \texttt{Medium}                     & \valstd{29}{5}  & \valstd{12}{2}  & \valstd{12}{3} & \valstd{18}{2}  & \valstd{36}{2}  & \valstd{88}{2} & \valstd{92}{1}           & \valstd{92}{2}           & \valstd{90}{2}           & \valstd{\mathbf{93}}{2}                             \\
			                                       &                                     & \texttt{Large}  & \valstd{6}{3}   & \valstd{1}{1}  & \valstd{0}{0}   & \valstd{3}{1}   & \valstd{4}{1}  & \valstd{28}{3}           & \valstd{70}{3}           & \valstd{74}{3}           & \valstd{63}{4}           & \valstd{\mathbf{79}}{3}  \\
			                                       &                                     & \texttt{Giant}  & \valstd{0}{0}   & \valstd{0}{0}  & \valstd{0}{0}   & \valstd{0}{0}   & \valstd{0}{0}  & \valstd{3}{2}            & \valstd{5}{1}            & \valstd{44}{4}           & \valstd{35}{4}           & \valstd{\mathbf{62}}{3}  \\
			\bottomrule
		\end{tabular}%
	}
	\vspace{-0.4cm}
    % \vspace{-0.2cm}
\end{table}

\paragraph{Locomotion.}
\label{sec:locomotion_results}

\cref{tab:locomotion} reports success rates on \texttt{PointMaze-Stitch}, \texttt{AntMaze-Stitch}, and \texttt{HumanoidMaze-Stitch} tasks, where all methods plan in the 2D $x$-$y$ position space and execute actions via a learned inverse dynamics model~\cite{luo2025generative}.
Offline GCRL algorithms struggle on \texttt{Giant} mazes, where the required planning horizon far exceeds the length of any single training trajectory.
Compositional diffusion-based planners overcome this limitation by assembling long plans from short segments.
RCD achieves the best performance across nearly all tasks. As indicated by the enhanced plan feasibility in \cref{sec:plan_quality}, RCD translates improved plan quality into higher success rates, particularly on the most challenging \texttt{Giant} tasks where mode-averaging is most severe.

\paragraph{Object Manipulation in High-Dimensional State Spaces.}
\label{sec:manipulation_results}

\begin{wrapfigure}{r}{5.5cm}
	\vspace{-0.4cm}
	\centering
	\begin{minipage}{0.45\linewidth}
		\centering
		\includegraphics[width=\linewidth]{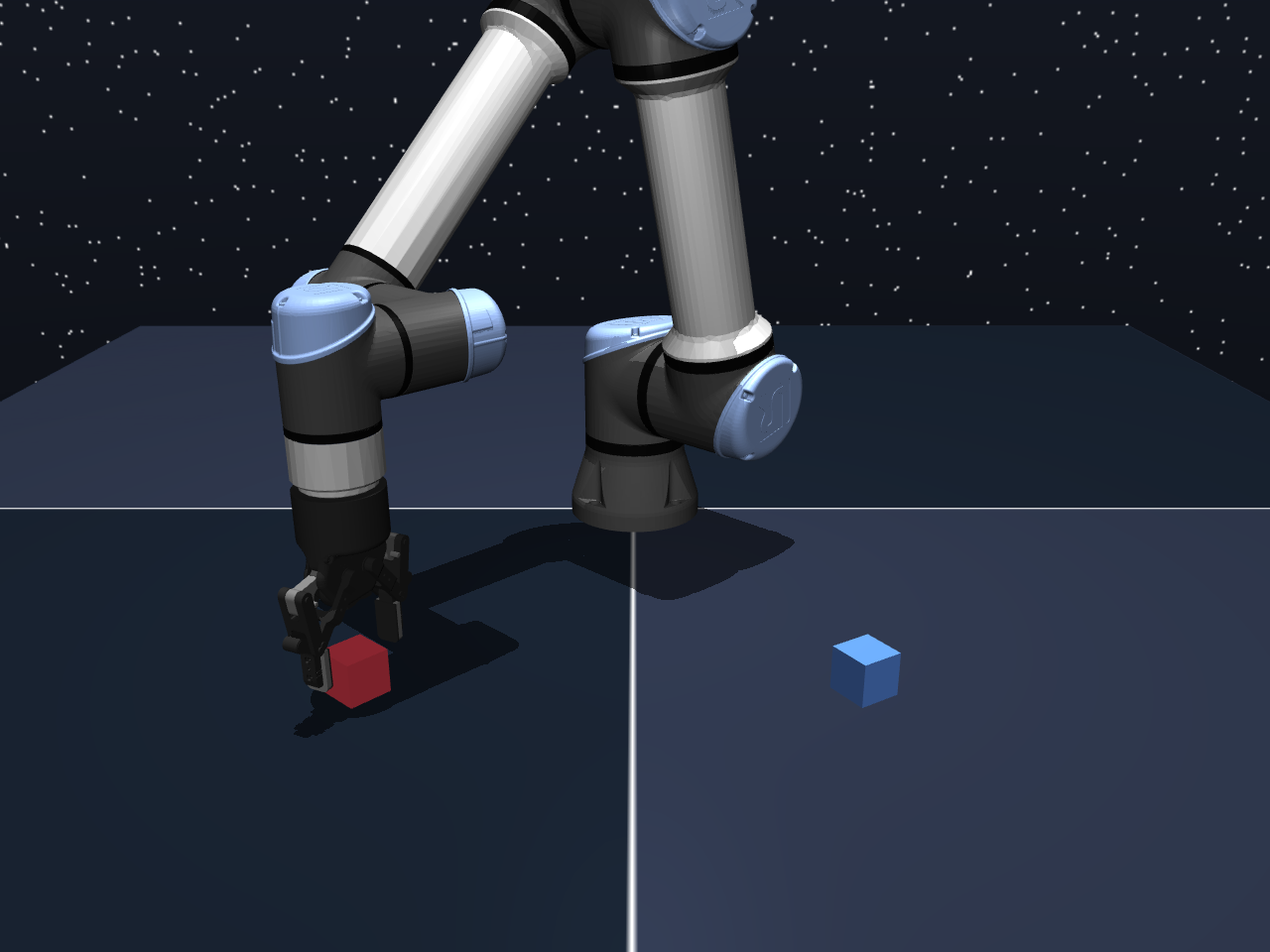}
	\end{minipage}%
	\hfill$\rightarrow$\hfill%
	\begin{minipage}{0.45\linewidth}
		\centering
		\includegraphics[width=\linewidth]{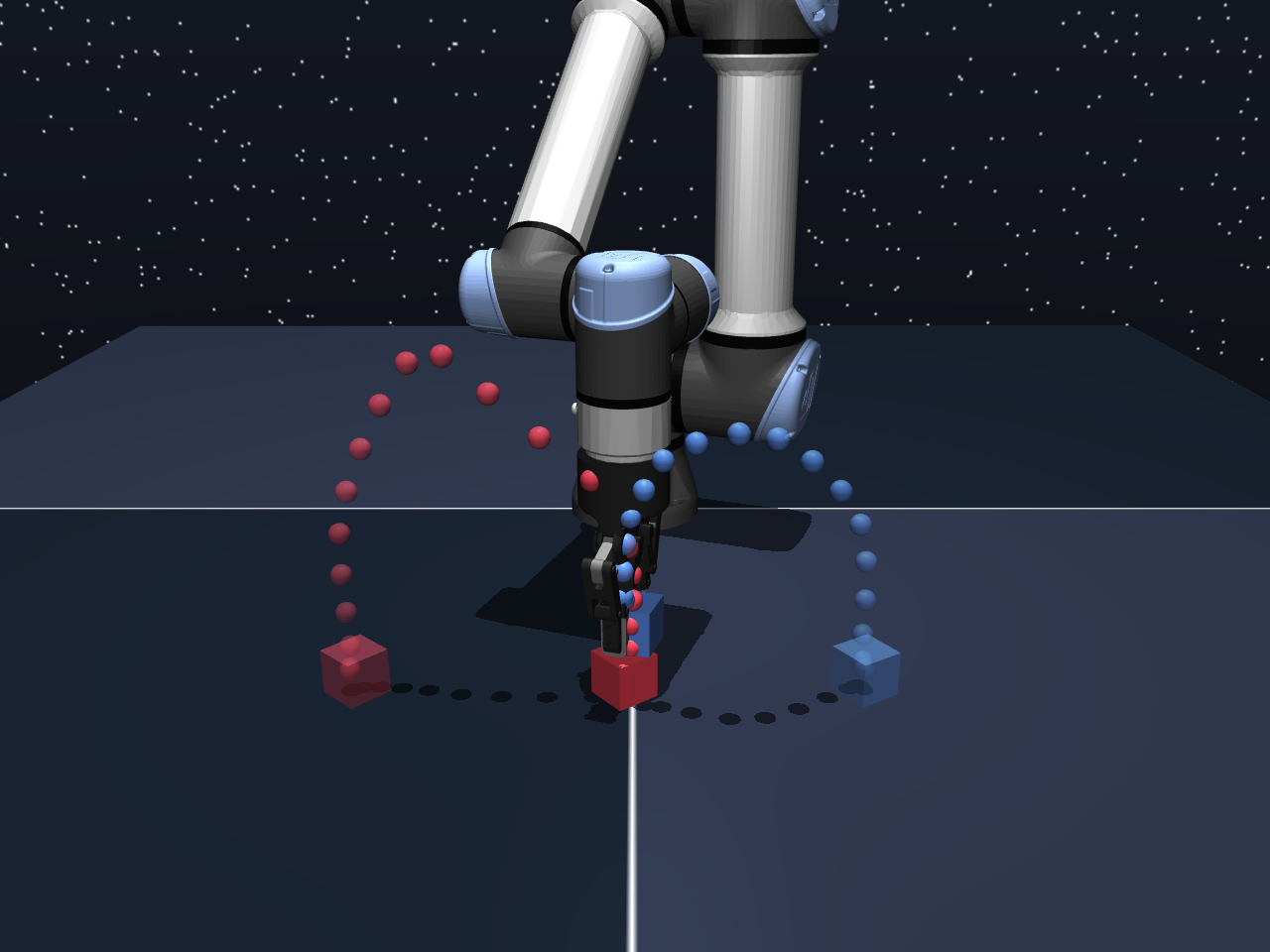}
	\end{minipage}
	\caption{\textbf{A cube manipulation sequence executed by RCD.}}
	\label{fig:cube_example}
	\vspace{-0.3cm}
\end{wrapfigure}

\cref{tab:manipulation} reports success rates on \texttt{Cube} and \texttt{AntSoccer} tasks, which involve planning in high-dimensional state spaces.
For \texttt{Cube}, the planner operates on the full state including the 6-DoF robot arm joint positions and all cube poses, and actions are generated by a DQL-based value-learning policy~\cite{wang2022diffusion}; for \texttt{AntSoccer}, we use the 17D planner that includes the $x$-$y$ positions of the ant and ball along with all 13 joint positions of the ant~\cite{luo2025generative}, with actions produced by a learned inverse dynamics model.
\texttt{Cube} tasks require composing dexterous pick-and-place sequences, ranging from \texttt{Single} cube to \texttt{Quadruple} with increasing horizon requirements.
\texttt{AntSoccer} requires coordinating ant locomotion with ball dribbling to reach a goal location.
Among offline GCRL methods, GCIQL performs best on \texttt{Cube}, while most baselines struggle on \texttt{AntSoccer} due to the long planning horizon.
RCD achieves the best or competitive performance across all tasks, demonstrating that the RCD guidance generalizes effectively to high-dimensional planning spaces.

\paragraph{Visual AntMaze.}
\label{sec:pixel_results}

\begin{wrapfigure}{r}{6.5cm}
	\vspace{-0.5cm}
    % \vspace{-0.4cm}
	\centering
	\captionof{table}{\textbf{Quantitative results on OGBench visual antmaze tasks.}}
	\label{tab:pixel}
	\vspace{-0.2cm}
	\footnotesize
	\setlength{\tabcolsep}{4.5pt}
	\begin{tabular}{ll llll}
		\toprule
		\textbf{Env} & \textbf{Size}   & \textbf{GSC}   & \textbf{CD}    & \textbf{CDGS}  & \textbf{RCD}            \\
		\midrule
		\multirow[c]{2}{*}{\makecell{\texttt{visual}                                                                \\\texttt{antmaze}}}
		             & \texttt{Medium} & \valstd{40}{4} & \valstd{55}{3} & \valstd{26}{4} & \valstd{\mathbf{63}}{3} \\
		             & \texttt{Large}  & \valstd{8}{3}  & \valstd{15}{3} & \valstd{11}{3} & \valstd{\mathbf{18}}{3} \\
		\bottomrule
	\end{tabular}
	\vspace{0.2cm}
	\includegraphics[width=0.95\linewidth]{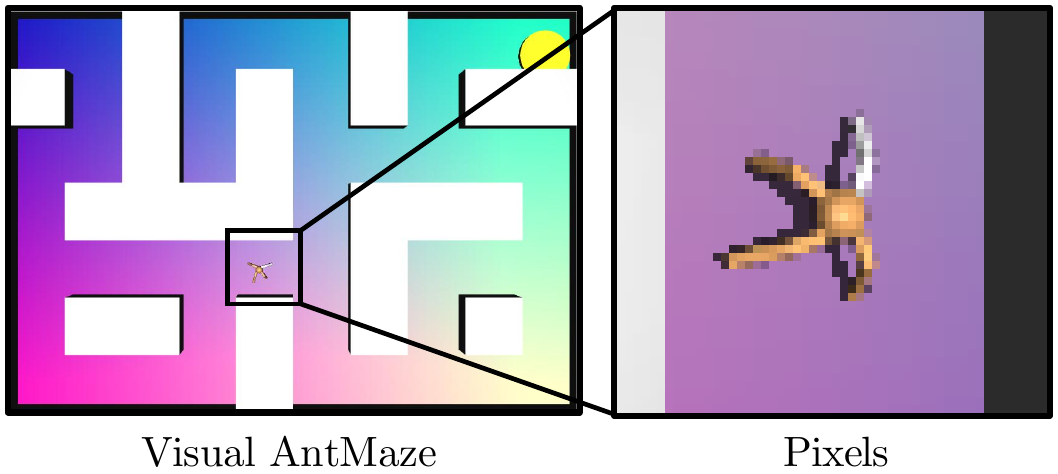}
    \vspace{-0.2cm}
	% \captionof{figure}{\textbf{Task example in the visual antmaze.}}
	\label{fig:visual_antmaze}
	% \vspace{-0.8cm}
    \vspace{-0.4cm}
\end{wrapfigure}

To evaluate RCD beyond state-based settings, we test on \texttt{AntMaze-Stitch} tasks with pixel observations, where the agent observes $64 \times 64$ RGB images and planning operates in a 16-dimensional latent space learned by a variational autoencoder. Actions are produced by an inverse dynamics model that maps consecutive latent states to actions.
\cref{tab:pixel} reports results for GSC, CD, CDGS, and RCD.
RCD achieves improved performance on both \texttt{Medium} and \texttt{Large} mazes, confirming that the self-reconstruction error provides a meaningful density signal even in a compressed latent space where the original pixel-level structure is abstracted away.

%-----------------------------------------------------------------------------
\subsection{Additional Results}
\label{sec:ablation}

\begin{table}[t]
	\centering
	\caption{\textbf{Quantitative results on OGBench object manipulation tasks.} Each environment provides 5 goal-conditioned tasks; we evaluate 20 episodes per task and report mean $\pm$ std across 5 seeds.}
	\label{tab:manipulation}
	\vspace{1mm}
	\resizebox{\textwidth}{!}{%
		\footnotesize
		\setlength{\tabcolsep}{4.5pt}
		\begin{tabular}{lll llllllllll}
			\toprule
			\textbf{Env}                           & \textbf{Type}                       & \textbf{Size}      & \textbf{GCBC}  & \textbf{GCIVL} & \textbf{GCIQL} & \textbf{QRL}   & \textbf{CRL}   & \textbf{HIQL}            & \textbf{GSC}             & \textbf{CD}              & \textbf{CDGS}            & \textbf{RCD}            \\
			\midrule
			\multirow[c]{4}{*}{\texttt{cube}}      & \multirow[c]{4}{*}{\texttt{play}}
			                                       & \texttt{Single}                     & \valstd{6}{2}      & \valstd{53}{4} & \valstd{68}{6} & \valstd{5}{1}  & \valstd{19}{2} & \valstd{15}{3} & \valstd{\mathbf{100}}{0} & \valstd{\mathbf{100}}{0} & \valstd{\mathbf{100}}{0} & \valstd{\mathbf{100}}{0}                           \\
			                                       &                                     & \texttt{Double}    & \valstd{1}{1}  & \valstd{36}{3} & \valstd{40}{5} & \valstd{1}{0}  & \valstd{10}{2} & \valstd{6}{2}            & \valstd{70}{4}           & \valstd{75}{3}           & \valstd{78}{4}           & \valstd{\mathbf{85}}{3} \\
			                                       &                                     & \texttt{Triple}    & \valstd{1}{1}  & \valstd{1}{0}  & \valstd{3}{1}  & \valstd{0}{0}  & \valstd{4}{1}  & \valstd{3}{1}            & \valstd{45}{4}           & \valstd{51}{4}           & \valstd{\mathbf{60}}{4}  & \valstd{\mathbf{60}}{3} \\
			                                       &                                     & \texttt{Quadruple} & \valstd{0}{0}  & \valstd{0}{0}  & \valstd{0}{0}  & \valstd{0}{0}  & \valstd{0}{0}  & \valstd{0}{0}            & \valstd{30}{4}           & \valstd{48}{4}           & \valstd{51}{4}           & \valstd{\mathbf{55}}{3} \\
			\midrule
			\multirow[c]{2}{*}{\texttt{antsoccer}} & \multirow[c]{2}{*}{\texttt{stitch}}
			                                       & \texttt{Arena}                      & \valstd{24}{8}     & \valstd{21}{3} & \valstd{2}{0}  & \valstd{1}{1}  & \valstd{1}{0}  & \valstd{15}{1} & \valstd{65}{3}           & \valstd{69}{3}           & \valstd{61}{3}           & \valstd{\mathbf{72}}{3}                            \\
			                                       &                                     & \texttt{Medium}    & \valstd{2}{1}  & \valstd{1}{0}  & \valstd{0}{0}  & \valstd{0}{0}  & \valstd{0}{0}  & \valstd{4}{1}            & \valstd{12}{2}           & \valstd{17}{3}           & \valstd{11}{2}           & \valstd{\mathbf{20}}{2} \\
			\bottomrule
		\end{tabular}%
	}
	\vspace{-0.2cm}
\end{table}

We present an ablation on the two RCD guidance components, an analysis of replanning, and a planning-time comparison with baselines below. Additional ablations on the guidance weight, the overlap consistency weight, and the probe level are provided in \cref{app:add_ablation}, confirming that RCD is robust to the choice of these hyperparameters.

\begin{wrapfigure}{r}{5.5cm}
	\vspace{-0.4cm}
	\centering
	\captionof{table}{\textbf{Ablation on guidance components.} Success rates with different combinations of $\mathcal{E}_{\mathrm{recon}}$ and $\mathcal{E}_{\mathrm{ov}}$.}
    \vspace{-0.1cm}
	\label{tab:ablation_components}
	\vspace{1mm}
	\footnotesize
	\begin{tabular}{cc ll}
		\toprule
		$\mathcal{E}_{\mathrm{recon}}$ & $\mathcal{E}_{\mathrm{ov}}$ & \makecell{\texttt{pointmaze}                           \\\texttt{Giant}} & \makecell{\texttt{antmaze}\\\texttt{Giant}} \\
		\midrule
		\xmark                         & \cmark                      & \valstd{59}{3}               & \valstd{56}{3}          \\
		\cmark                         & \xmark                      & \valstd{95}{2}               & \valstd{79}{3}          \\
		\cmark                         & \cmark                      & \valstd{\mathbf{100}}{0}     & \valstd{\mathbf{89}}{2} \\
		\bottomrule
	\end{tabular}
	\vspace{-0.3cm}
\end{wrapfigure}

\paragraph{Ablation on Guidance Components.}
RCD combines two guidance signals: the self-reconstruction error ($\mathcal{E}_{\mathrm{recon}}$) and the overlap consistency ($\mathcal{E}_{\mathrm{ov}}$). \cref{tab:ablation_components} isolates their contributions on \texttt{PointMaze-Giant-Stitch} and \texttt{AntMaze-Giant-Stitch}.
The self-reconstruction error alone already provides a large improvement over the unguided baseline by steering samples toward high-density regions.
Adding the overlap consistency further improves performance, particularly on \texttt{AntMaze-Giant-Stitch}, confirming that the two signals are complementary. $\mathcal{E}_{\mathrm{recon}}$ addresses global density while $\mathcal{E}_{\mathrm{ov}}$ targets local boundary agreement.

\begin{wraptable}{r}{7.5cm}
	\vspace{-0.4cm}
	\caption{\textbf{Ablation study on replanning.} Success rates on \texttt{pointmaze} and \texttt{antmaze} \texttt{stitch}, comparing CD and RCD with (\cmark) and without (\xmark) replanning.}
	\label{tab:ablation_replan}
	\centering
	\resizebox{0.53\textwidth}{!}{%
		\begin{tabular}{ll ll ll}
			\toprule
			\multirow{2}{*}[-0.7ex]{\textbf{Env}} & \multirow{2}{*}[-0.7ex]{\textbf{Size}} & \multicolumn{2}{c}{\textbf{CD}} & \multicolumn{2}{c}{\textbf{RCD}}                                              \\
			\cmidrule(lr){3-4} \cmidrule(lr){5-6}
			                                      &                                        & \textbf{\xmark}                 & \textbf{\cmark}                  & \textbf{\xmark} & \textbf{\cmark}          \\
			\midrule
			\multirow[c]{3}{*}{\texttt{pointmaze}}
			                                      & \texttt{Medium}                        & \valstd{100}{0}                 & \valstd{100}{0}                  & \valstd{100}{0} & \valstd{100}{0}          \\
			                                      & \texttt{Large}                         & \valstd{100}{0}                 & \valstd{100}{0}                  & \valstd{100}{0}  & \valstd{100}{0} \\
			                                      & \texttt{Giant}                         & \valstd{55}{5}                  & \valstd{69}{3}                   & \valstd{98}{1}  & \valstd{\mathbf{100}}{0} \\
			\midrule
			\multirow[c]{3}{*}{\texttt{antmaze}}
			                                      & \texttt{Medium}                        & \valstd{90}{3}                  & \valstd{91}{2}                   & \valstd{91}{2}  & \valstd{\mathbf{97}}{2}  \\
			                                      & \texttt{Large}                         & \valstd{73}{3}                  & \valstd{89}{2}                   & \valstd{76}{3}  & \valstd{\mathbf{91}}{2}  \\
			                                      & \texttt{Giant}                         & \valstd{29}{4}                  & \valstd{67}{3}                   & \valstd{69}{3}  & \valstd{\mathbf{89}}{2}  \\
			\midrule
			\multicolumn{2}{c}{\textbf{Average}}  & $74.5$                                 & $86.0$                          & $89.0$                           & $\mathbf{96.2}$                            \\
			\bottomrule
		\end{tabular}%
	}
	\vspace{-0.2cm}
\end{wraptable}

\paragraph{Ablation Study on Replanning.}
Replanning periodically regenerates a plan from the current state during execution to mitigate execution drift (see \cref{app:eval_protocol} for details).
\cref{tab:ablation_replan} compares CD and RCD with and without replanning on \texttt{PointMaze-Stitch} and \texttt{AntMaze-Stitch} tasks.
Replanning improves both methods, with the largest gains on \texttt{Giant} mazes where long plans are most susceptible to drift.
Notably, RCD without replanning already outperforms CD with replanning on the most challenging \texttt{Giant} tasks, highlighting the importance of refining plans during generation where mode-averaging is most severe.

\begin{wraptable}{r}{7.5cm}
	\vspace{-0.4cm}
	\caption{\textbf{Planning-time comparison with baselines.} Success rates and planning times (sec.).}
	\label{tab:planning_time}
	\centering
	\resizebox{0.53\textwidth}{!}{%
		\begin{tabular}{l ll ll}
			\toprule
			\multirow{2}{*}[-0.7ex]{\textbf{Method}} & \multicolumn{2}{c}{\texttt{pointmaze-Giant}} & \multicolumn{2}{c}{\texttt{antmaze-Giant}}                                                        \\
			\cmidrule(lr){2-3} \cmidrule(lr){4-5}
			                                         & \textbf{Succ.}                               & \textbf{Time (s)}                          & \textbf{Succ.}          & \textbf{Time (s)}          \\
			\midrule
			CD                                       & \valstd{69}{3}                               & \valstd{8.0}{0.5}                          & \valstd{67}{3}          & \valstd{8.9}{0.7}          \\
			CDGS                                     & \valstd{74}{3}                               & \valstd{44.7}{3.9}                         & \valstd{83}{3}          & \valstd{56.7}{7.4}         \\
			RCD                                      & \valstd{\mathbf{100}}{0}                     & \valstd{\mathbf{4.9}}{0.1}                 & \valstd{\mathbf{89}}{2} & \valstd{\mathbf{4.6}}{0.3} \\
			\bottomrule
		\end{tabular}%
	}
	\vspace{-0.2cm}
\end{wraptable}

\paragraph{Planning-Time Comparison.}
\cref{tab:planning_time} compares the success rate and planning time of RCD with CompDiffuser (CD) and CDGS on \texttt{PointMaze-Giant-Stitch} and \texttt{AntMaze-Giant-Stitch}, measured on a single NVIDIA H100 GPU.
RCD applies its guidance on top of the parallel compositional score averaging (\cref{eq:comp_score}). Despite introducing this guidance step, its planning time remains below CD, which adopts an autoregressive variant for bidirectional conditioning, and is an order of magnitude below CDGS (population-based search). Intuitively, RCD is fast because all $M$ segments are evaluated in a single batched call at each denoising step, whereas CD must sweep segments sequentially to propagate bidirectional conditioning and CDGS runs multiple sequential resampling and pruning passes per step.

%=============================================================================
% RELATED WORK
%=============================================================================
\section{Related Work}
\label{sec:related}

\paragraph{Planning with Diffusion Models.}
Diffusion probabilistic models~\cite{ho2020denoising, song2021scorebased} have emerged as a powerful framework for trajectory planning in reinforcement learning.
Diffuser~\cite{janner2022planning} pioneered this direction by training an unconditional diffusion model on offline trajectories and guiding it toward high-return regions via a learned value function.
Decision Diffuser~\cite{ajay2023is} introduced classifier-free guidance, conditioning directly on reward or constraint signals, while AdaptDiffuser~\cite{liang2023adaptdiffuser} progressively fine-tuned the model with synthetic high-quality data.
Subsequent work extended diffusion planning to hierarchical settings~\cite{chen2024simple, li2023hierarchical}, multi-agent coordination~\cite{zhu2023madiff}, latent trajectory spaces~\cite{xie2025latent}, and inference-time tree search over diffusion rollouts~\cite{yoon2025monte, yoon2025fast}.
Monolithic diffusion planners, however, are limited by the training data horizon, motivating \emph{compositional} approaches that stitch short segments at inference time.
GSC~\cite{mishra2023generative} chains skills via score composition, CompDiffuser~\cite{luo2025generative} introduces bidirectional neighbor conditioning for smoother stitching, and CDGS~\cite{mishra2026compositional} combines compositional diffusion with population-based guided search to mitigate mode-averaging.
A complementary line of work improves the reliability of plans produced by a diffusion planner.
LoMAP~\cite{lee2025local} projects intermediate diffusion samples onto a local approximation of the data manifold to suppress off-manifold artifacts, while RGG~\cite{lee2023refining} detects infeasible plans via a separately trained, time-dependent classifier and resamples them, which requires additional training data.
RCD shares the goal of refining diffusion-generated plans but targets the compositional setting, where mode-averaging is the dominant failure mode.
Unlike RGG, which trains a time-dependent classifier on additional labeled data, RCD is fully training-free and uses only the pretrained local denoiser as an intrinsic density proxy.

\vspace{-0.05cm}
\paragraph{Trajectory Stitching.}
The ability to compose novel long-horizon behaviors from short trajectory fragments is a long-standing challenge in offline RL.
One broad category of approaches relies on offline data augmentation, where short or reward-suboptimal trajectories are stitched or synthesized and a downstream policy is retrained on the enriched data, spanning generative augmentation~\cite{li2024diffstitch, jackson2024policy, lee2024gta, lee2025scots, chen2025extendable}, model-based stitching~\cite{char2022bats, hepburn2022model, zhou2023free}, and clustering-based augmentation~\cite{ghugare2024closing}.
RCD instead follows the \emph{inference-time} compositional diffusion line of CompDiffuser~\cite{luo2025generative} and CDGS~\cite{mishra2026compositional}, which compose local diffusion models directly at sampling time and avoid regenerating the dataset.

%=============================================================================
% CONCLUSION
%=============================================================================
\section{Conclusion}
\label{sec:conclusion}

We presented Refining Compositional Diffusion (RCD), a training-free guidance method that addresses the mode-averaging problem in compositional diffusion planning.
By leveraging the self-reconstruction error of a pretrained diffusion model as a density proxy and combining it with an overlap consistency term that penalizes score disagreement at segment boundaries, RCD steers compositional sampling toward high-density, globally coherent plans without additional training.
We established theoretical connections between the reconstruction error and the composed model's evidence lower bound, and showed that the RCD-guided reverse process targets a tilted distribution concentrating on globally coherent plans.
Empirically, RCD achieves consistent improvements over existing compositional methods across OGBench tasks spanning locomotion, object manipulation, and pixel-based observations, with particularly pronounced gains on the most challenging long-horizon tasks where mode-averaging is most severe.

\vspace{-0.05cm}
\paragraph{Limitations.}
The current formulation operates within the Bethe factor graph framework of~\cite{yedidia2005constructing, mishra2026compositional}, which assumes chain-structured overlapping segments. Extending RCD guidance to other compositional structures, such as hierarchical or temporal abstractions, or factor graphs with loops, is an interesting future direction. Another practical aspect, shared with prior compositional planners such as CompDiffuser~\cite{luo2025generative} and CDGS~\cite{mishra2026compositional}, is that the number of local segments used to determine the total horizon is pre-specified at inference time. Adaptively choosing it is another promising direction, for instance via tree-search-style horizon expansion~\cite{yoon2025monte} or by monitoring the self-reconstruction error along candidate horizons.

\vspace{-0.05cm}
\paragraph{Impact Statement.}
This paper advances compositional diffusion planning by introducing a training-free guidance method that refines long-horizon plans toward high-density, globally coherent trajectories. While we do not identify direct negative societal impacts stemming from this research, practitioners deploying diffusion-based planners in real-world decision-making systems are encouraged to carefully assess safety and reliability prior to deployment, particularly in settings where unsafe actions could carry significant consequences.

\section*{Acknowledgements}

This work was supported by the Institute of Information \& Communications Technology Planning \& Evaluation (IITP) grants funded by the Korean government (MSIT) (grant Nos. RS-2019-II190075, Artificial Intelligence Graduate School Program (KAIST); RS-2019-II190079, Artificial Intelligence Graduate School Program (Korea University); RS-2022-II220984, Development of Artificial Intelligence Technology for Personalized Plug-and-Play Explanation and Verification of Explanation; RS-2024-00457882, AI Research Hub Project; IITP-2026-RS-2023-00255968, Artificial Intelligence Convergence Innovation Human Resources Development), the National Research Foundation of Korea (NRF) grant funded by the Korean government (MSIT) (No. RS-2025-23525509), and the InnoCORE program of the Ministry of Science and ICT (N10250156).

\bibliography{references}

@inproceedings{tassa2012synthesis,
  title={Synthesis and stabilization of complex behaviors through online trajectory optimization},
  author={Tassa, Yuval and Erez, Tom and Todorov, Emanuel},
  booktitle={2012 IEEE/RSJ International Conference on Intelligent Robots and Systems},
  pages={4906--4913},
  year={2012},
}

@article{silver2016mastering,
  title={Mastering the game of Go with deep neural networks and tree search},
  author={Silver, David and Huang, Aja and Maddison, Chris J and Guez, Arthur and Sifre, Laurent and Van Den Driessche, George and Schrittwieser, Julian and Antonoglou, Ioannis and Panneershelvam, Veda and Lanctot, Marc and others},
  journal={nature},
  volume={529},
  number={7587},
  pages={484--489},
  year={2016},
  publisher={Nature Publishing Group}
}

@article{silver2017mastering,
  title={Mastering the game of go without human knowledge},
  author={Silver, David and Schrittwieser, Julian and Simonyan, Karen and Antonoglou, Ioannis and Huang, Aja and Guez, Arthur and Hubert, Thomas and Baker, Lucas and Lai, Matthew and Bolton, Adrian and others},
  journal={nature},
  volume={550},
  number={7676},
  pages={354--359},
  year={2017},
  publisher={Nature Publishing Group}
}

@inproceedings{lee2018deep,
  title={Deep reinforcement learning in continuous action spaces: a case study in the game of simulated curling},
  author={Lee, Kyowoon and Kim, Sol-A and Choi, Jaesik and Lee, Seong-Whan},
  booktitle={International Conference on Machine Learning (ICML)},
  year={2018},
}

@inproceedings{janner2022planning,
  title={Planning with diffusion for flexible behavior synthesis},
  author={Janner, Michael and Du, Yilun and Tenenbaum, Joshua B and Levine, Sergey},
  booktitle={International Conference on Machine Learning (ICML)},
  year={2022}
}

@inproceedings{
    ajay2023is,
    title={Is Conditional Generative Modeling all you need for Decision Making?},
    author={Anurag Ajay and Yilun Du and Abhi Gupta and Joshua B. Tenenbaum and Tommi S. Jaakkola and Pulkit Agrawal},
    booktitle={International Conference on Learning Representations (ICLR)},
    year={2023}
}

@inproceedings{liang2023adaptdiffuser,
  title={AdaptDiffuser: Diffusion Models as Adaptive Self-evolving Planners},
  author={Liang, Zhixuan and Mu, Yao and Ding, Mingyu and Ni, Fei and Tomizuka, Masayoshi and Luo, Ping},
  booktitle={International Conference on Machine Learning (ICML)},
  year={2023},
}

@inproceedings{luo2025generative,
  title={Generative Trajectory Stitching through Diffusion Composition},
  author={Luo, Yunhao and Mishra, Utkarsh A and Du, Yilun and Xu, Danfei},
  booktitle={Advances in Neural Information Processing Systems (NeurIPS)},
  year={2025}
}

@inproceedings{
mishra2026compositional,
title={Compositional Diffusion with Guided search for Long-Horizon Planning},
author={Utkarsh Aashu Mishra and David He and Yongxin Chen and Danfei Xu},
booktitle={International Conference on Learning Representations (ICLR)},
year={2026},
}

@inproceedings{zhang2023diffcollage,
  title={Diffcollage: Parallel generation of large content with diffusion models},
  author={Zhang, Qinsheng and Song, Jiaming and Huang, Xun and Chen, Yongxin and Liu, Ming-Yu},
  booktitle={2023 IEEE/CVF Conference on Computer Vision and Pattern Recognition (CVPR)},
  pages={10188--10198},
  year={2023},
  organization={IEEE}
}

@inproceedings{du2020compositional,
  title={Compositional visual generation with energy based models},
  author={Du, Yilun and Li, Shuang and Mordatch, Igor},
  booktitle={Advances in Neural Information Processing Systems (NeurIPS)},
  year={2020}
}

@inproceedings{du2023reduce,
  title={Reduce, reuse, recycle: Compositional generation with energy-based diffusion models and mcmc},
  author={Du, Yilun and Durkan, Conor and Strudel, Robin and Tenenbaum, Joshua B and Dieleman, Sander and Fergus, Rob and Sohl-Dickstein, Jascha and Doucet, Arnaud and Grathwohl, Will Sussman},
  booktitle={International Conference on Machine Learning (ICML)},
  year={2023},
}

@inproceedings{mishra2023generative,
  title={Generative skill chaining: Long-horizon skill planning with diffusion models},
  author={Mishra, Utkarsh Aashu and Xue, Shangjie and Chen, Yongxin and Xu, Danfei},
  booktitle={Conference on Robot Learning},
  pages={2905--2925},
  year={2023},
  organization={PMLR}
}

@inproceedings{song2021denoising,
  title={Denoising Diffusion Implicit Models},
  author={Song, Jiaming and Meng, Chenlin and Ermon, Stefano},
  booktitle={International Conference on Learning Representations (ICLR)},
  year={2021}
}

@inproceedings{ho2020denoising,
  title={Denoising diffusion probabilistic models},
  author={Ho, Jonathan and Jain, Ajay and Abbeel, Pieter},
  booktitle={Advances in Neural Information Processing Systems (NeurIPS)},
  year={2020}
}

@article{yedidia2005constructing,
  title={Constructing free-energy approximations and generalized belief propagation algorithms},
  author={Yedidia, Jonathan S and Freeman, William T and Weiss, Yair},
  journal={IEEE Transactions on information theory},
  volume={51},
  number={7},
  pages={2282--2312},
  year={2005},
  publisher={IEEE}
}

@inproceedings{
  song2021scorebased,
  title={Score-Based Generative Modeling through Stochastic Differential Equations},
  author={Yang Song and Jascha Sohl-Dickstein and Diederik P Kingma and Abhishek Kumar and Stefano Ermon and Ben Poole},
  booktitle={International Conference on Learning Representations (ICLR)},
  year={2021},
}

@inproceedings{dhariwal2021diffusion,
  title={Diffusion models beat gans on image synthesis},
  author={Dhariwal, Prafulla and Nichol, Alexander},
  booktitle={Advances in Neural Information Processing Systems (NeurIPS)},
  year={2021}
}

@inproceedings{chung2023diffusion,
  title={Diffusion posterior sampling for general noisy inverse problems},
  author={Chung, Hyungjin and Kim, Jeongsol and Mccann, Michael T and Klasky, Marc L and Ye, Jong Chul},
  booktitle={International Conference on Learning Representations (ICLR)},
  year={2023},
}

@incollection{robbins1992empirical,
  title={An empirical Bayes approach to statistics},
  author={Robbins, Herbert E},
  booktitle={Breakthroughs in Statistics: Foundations and basic theory},
  pages={388--394},
  year={1992},
  publisher={Springer}
}

@inproceedings{song2023loss,
  title={Loss-guided diffusion models for plug-and-play controllable generation},
  author={Song, Jiaming and Zhang, Qinsheng and Yin, Hongxu and Mardani, Morteza and Liu, Ming-Yu and Kautz, Jan and Chen, Yongxin and Vahdat, Arash},
  booktitle={International Conference on Machine Learning (ICML)},
  year={2023},
}

@inproceedings{yang2024guidance,
  title={Guidance with spherical gaussian constraint for conditional diffusion},
  author={Yang, Lingxiao and Ding, Shutong and Cai, Yifan and Yu, Jingyi and Wang, Jingya and Shi, Ye},
  booktitle={International Conference on Machine Learning (ICML)},
  year={2024}
}

@inproceedings{he2024manifold,
  title={Manifold preserving guided diffusion},
  author={He, Yutong and Murata, Naoki and Lai, Chieh-Hsin and Takida, Yuhta and Uesaka, Toshimitsu and Kim, Dongjun and Liao, Wei-Hsiang and Mitsufuji, Yuki and Kolter, J Zico and Salakhutdinov, Ruslan and others},
  booktitle={International Conference on Learning Representations (ICLR)},
  year={2024}
}

@inproceedings{park2025ogbench,
  title={Ogbench: Benchmarking offline goal-conditioned rl},
  author={Park, Seohong and Frans, Kevin and Eysenbach, Benjamin and Levine, Sergey},
  booktitle={International Conference on Learning Representations (ICLR)},
  year={2025}
}

@inproceedings{kostrikov2022offline,
  title={Offline reinforcement learning with implicit q-learning},
  author={Kostrikov, Ilya and Nair, Ashvin and Levine, Sergey},
  booktitle={International Conference on Learning Representations (ICLR)},
  year={2022}
}

@inproceedings{lynch2020learning,
  title={Learning latent plans from play},
  author={Lynch, Corey and Khansari, Mohi and Xiao, Ted and Kumar, Vikash and Tompson, Jonathan and Levine, Sergey and Sermanet, Pierre},
  booktitle={Conference on robot learning},
  year={2020},
}

@article{ghosh2019learning,
  title={Learning to reach goals via iterated supervised learning},
  author={Ghosh, Dibya and Gupta, Abhishek and Reddy, Ashwin and Fu, Justin and Devin, Coline and Eysenbach, Benjamin and Levine, Sergey},
  journal={arXiv preprint arXiv:1912.06088},
  year={2019}
}

@inproceedings{wang2023optimal,
  title={Optimal goal-reaching reinforcement learning via quasimetric learning},
  author={Wang, Tongzhou and Torralba, Antonio and Isola, Phillip and Zhang, Amy},
  booktitle={International Conference on Machine Learning (ICML)},
  year={2023}
}

@inproceedings{eysenbach2022contrastive,
  title={Contrastive learning as goal-conditioned reinforcement learning},
  author={Eysenbach, Benjamin and Zhang, Tianjun and Levine, Sergey and Salakhutdinov, Russ R},
  booktitle={Advances in Neural Information Processing Systems (NeurIPS)},
  year={2022}
}

@inproceedings{park2023hiql,
  title={Offline goal-conditioned rl with latent states as actions},
  author={Park, Seohong and Ghosh, Dibya and Eysenbach, Benjamin and Levine, Sergey},
  booktitle={Advances in Neural Information Processing Systems (NeurIPS)},
  year={2023}
}

@article{wang2022diffusion,
  title={Diffusion policies as an expressive policy class for offline reinforcement learning},
  author={Wang, Zhendong and Hunt, Jonathan J and Zhou, Mingyuan},
  journal={arXiv preprint arXiv:2208.06193},
  year={2022}
}

@article{chen2024simple,
  title={Simple hierarchical planning with diffusion},
  author={Chen, Chang and Deng, Fei and Kawaguchi, Kenji and Gulcehre, Caglar and Ahn, Sungjin},
  journal={arXiv preprint arXiv:2401.02644},
  year={2024}
}

@inproceedings{li2023hierarchical,
  title={Hierarchical diffusion for offline decision making},
  author={Li, Wenhao and Wang, Xiangfeng and Jin, Bo and Zha, Hongyuan},
  booktitle={International Conference on Machine Learning (ICML)},
  year={2023},
}

@article{zhu2023madiff,
  title={Madiff: Offline multi-agent learning with diffusion models},
  author={Zhu, Zhengbang and Liu, Minghuan and Mao, Liyuan and Kang, Bingyi and Xu, Minkai and Yu, Yong and Ermon, Stefano and Zhang, Weinan},
  journal={arXiv preprint arXiv:2305.17330},
  year={2023}
}

@inproceedings{lee2025local,
  title={Local Manifold Approximation and Projection for Manifold-Aware Diffusion Planning},
  author={Lee, Kyowoon and Choi, Jaesik},
  booktitle={International Conference on Machine Learning (ICML)},
  year={2025}
}

@inproceedings{lee2023refining,
  title={Refining Diffusion Planner for Reliable Behavior Synthesis by Automatic Detection of Infeasible Plans},
  author={Lee, Kyowoon and Kim, Seongun and Choi, Jaesik},
  booktitle={Advances in Neural Information Processing Systems (NeurIPS)},
  year={2023},
}

@inproceedings{lee2025scots,
  title={State-covering trajectory stitching for diffusion planners},
  author={Lee, Kyowoon and Choi, Jaesik},
  booktitle={Advances in Neural Information Processing Systems (NeurIPS)},
  year={2025}
}

@inproceedings{li2024diffstitch,
  title={Diffstitch: Boosting offline reinforcement learning with diffusion-based trajectory stitching},
  author={Li, Guanghe and Shan, Yixiang and Zhu, Zhengbang and Long, Ting and Zhang, Weinan},
  booktitle={International Conference on Machine Learning (ICML)},
  year={2024}
}

@article{jackson2024policy,
  title={Policy-guided diffusion},
  author={Jackson, Matthew Thomas and Matthews, Michael Tryfan and Lu, Cong and Ellis, Benjamin and Whiteson, Shimon and Foerster, Jakob},
  journal={arXiv preprint arXiv:2404.06356},
  year={2024}
}

@article{lee2024gta,
  title={Gta: Generative trajectory augmentation with guidance for offline reinforcement learning},
  author={Lee, Jaewoo and Yun, Sujin and Yun, Taeyoung and Park, Jinkyoo},
  journal={arXiv preprint arXiv:2405.16907},
  year={2024}
}

@inproceedings{yoon2025monte,
  title={Monte Carlo Tree Diffusion for System 2 Planning},
  author={Yoon, Jaesik and Cho, Hyeonseo and Baek, Doojin and Bengio, Yoshua and Ahn, Sungjin},
  booktitle={International Conference on Machine Learning (ICML)},
  year={2025}
}

@article{yoon2025fast,
  title={Fast Monte Carlo Tree Diffusion: 100x Speedup via Parallel Sparse Planning},
  author={Yoon, Jaesik and Cho, Hyeonseo and Bengio, Yoshua and Ahn, Sungjin},
  journal={arXiv preprint arXiv:2506.09498},
  year={2025}
}

@article{chen2025extendable,
  title={Extendable long-horizon planning via hierarchical multiscale diffusion},
  author={Chen, Chang and Hamed, Hany and Baek, Doojin and Kang, Taegu and Bengio, Yoshua and Ahn, Sungjin},
  journal={arXiv e-prints},
  pages={arXiv--2503},
  year={2025}
}

@article{char2022bats,
  title={Bats: Best action trajectory stitching},
  author={Char, Ian and Mehta, Viraj and Villaflor, Adam and Dolan, John M and Schneider, Jeff},
  journal={arXiv preprint arXiv:2204.12026},
  year={2022}
}

@article{hepburn2022model,
  title={Model-based trajectory stitching for improved offline reinforcement learning},
  author={Hepburn, Charles A and Montana, Giovanni},
  journal={arXiv preprint arXiv:2211.11603},
  year={2022}
}

@inproceedings{zhou2023free,
  title={Free from bellman completeness: Trajectory stitching via model-based return-conditioned supervised learning},
  author={Zhou, Zhaoyi and Zhu, Chuning and Zhou, Runlong and Cui, Qiwen and Gupta, Abhishek and Du, Simon Shaolei},
  booktitle={International Conference on Learning Representations (ICLR)},
  year={2023}
}

@inproceedings{ghugare2024closing,
  title={Closing the Gap between TD Learning and Supervised Learning--A Generalisation Point of View},
  author={Ghugare, Raj and Geist, Matthieu and Berseth, Glen and Eysenbach, Benjamin},
  booktitle={International Conference on Learning Representations (ICLR)},
  year={2024}
}

@inproceedings{shen2024understanding,
  title={Understanding and improving training-free loss-based diffusion guidance},
  author={Shen, Yifei and Jiang, Xinyang and Yang, Yifan and Wang, Yezhen and Han, Dongqi and Li, Dongsheng},
  booktitle={Advances in Neural Information Processing Systems (NeurIPS)},
  year={2024}
}

@inproceedings{chi2023diffusionpolicy,
	title={Diffusion Policy: Visuomotor Policy Learning via Action Diffusion},
	author={Chi, Cheng and Feng, Siyuan and Du, Yilun and Xu, Zhenjia and Cousineau, Eric and Burchfiel, Benjamin and Song, Shuran},
	booktitle={Proceedings of Robotics: Science and Systems (RSS)},
	year={2023}
}

@inproceedings{xie2025latent,
  title={Latent diffusion planning for imitation learning},
  author={Xie, Amber and Rybkin, Oleh and Sadigh, Dorsa and Finn, Chelsea},
  booktitle={International Conference on Machine Learning (ICML)},
  year={2025}
}

@inproceedings{feng2026ada,
title={Ada-Diffuser: Latent-Aware Adaptive Diffusion for Decision-Making},
author={Fan Feng and Selena Ge and Minghao Fu and Zijian Li and Yujia Zheng and Zeyu Tang and Yingyao Hu and Biwei Huang and Kun Zhang},
booktitle={International Conference on Learning Representations (ICLR)},
year={2026}
}

@inproceedings{
park2025flow,
title={Flow Q-Learning},
author={Seohong Park and Qiyang Li and Sergey Levine},
booktitle={International Conference on Machine Learning (ICML)},
year={2025}
}

@inproceedings{
kim2026deas,
title={{DEAS}: {DE}tached value learning with Action Sequence for Scalable Offline {RL}},
author={Changyeon Kim and Haeone Lee and Younggyo Seo and Kimin Lee and Yuke Zhu},
booktitle={International Conference on Learning Representations (ICLR)},
year={2026}
}

@inproceedings{
park2026scalable,
title={Scalable Offline Model-Based {RL} with Action Chunks},
author={Kwanyoung Park and Seohong Park and Youngwoon Lee and Sergey Levine},
booktitle={International Conference on Learning Representations (ICLR)},
year={2026}
}

@inproceedings{
li2026decoupled,
title={Decoupled Q-Chunking},
author={Qiyang Li and Seohong Park and Sergey Levine},
booktitle={International Conference on Learning Representations (ICLR)},
year={2026}
}

@inproceedings{agrawalla2025floq,
title={floq: Training critics via flow-matching for scaling compute in value-based rl},
author={Agrawalla, Bhavya and Nauman, Michal and Agrawal, Khush and Kumar, Aviral},
booktitle={International Conference on Learning Representations (ICLR)},
year={2026}
}

@article{frans2025diffusion,
  title={Diffusion guidance is a controllable policy improvement operator},
  author={Frans, Kevin and Park, Seohong and Abbeel, Pieter and Levine, Sergey},
  journal={arXiv preprint arXiv:2505.23458},
  year={2025}
}

@inproceedings{ki2025prior,
title={Prior-guided diffusion planning for offline reinforcement learning},
author={Ki, Donghyeon and Oh, JunHyeok and Shim, Seong-Woong and Lee, Byung-Jun},
booktitle={Advances in Neural Information Processing Systems (NeurIPS)},
year={2025}
}

@inproceedings{chen2024diffusion,
  title={Diffusion forcing: Next-token prediction meets full-sequence diffusion},
  author={Chen, Boyuan and Mart{\'\i} Mons{\'o}, Diego and Du, Yilun and Simchowitz, Max and Tedrake, Russ and Sitzmann, Vincent},
  booktitle={Advances in Neural Information Processing Systems (NeurIPS)},
  year={2024}
}

@article{cai2025flood,
  title={FloodDiffusion: Tailored Diffusion Forcing for Streaming Motion Generation},
  author={Cai, Yiyi and Wu, Yuhan and Li, Kunhang and Zhou, You and Zheng, Bo and Liu, Haiyang},
  journal={arXiv preprint arXiv:2512.03520},
  year={2025}
}

@inproceedings{dong2024diffuserlite,
  title={Diffuserlite: Towards real-time diffusion planning},
  author={Dong, Zibin and Hao, Jianye and Yuan, Yifu and Ni, Fei and Wang, Yitian and Li, Pengyi and Zheng, Yan},
  booktitle={Advances in Neural Information Processing Systems (NeurIPS)},
  year={2024}
}

@inproceedings{zheng2025diffusion,
  title={Diffusion-based planning for autonomous driving with flexible guidance},
  author={Zheng, Yinan and Liang, Ruiming and Zheng, Kexin and Zheng, Jinliang and Mao, Liyuan and Li, Jianxiong and Gu, Weihao and Ai, Rui and Li, Shengbo Eben and Zhan, Xianyuan and others},
  booktitle={International Conference on Learning Representations (ICLR)},
  year={2025}
}

@inproceedings{shaoul2025multi,
  title={Multi-robot motion planning with diffusion models},
  author={Shaoul, Yorai and Mishani, Itamar and Vats, Shivam and Li, Jiaoyang and Likhachev, Maxim},
  booktitle={International Conference on Learning Representations (ICLR)},
  year={2025}
}

@inproceedings{liang2025simultaneous,
  title={Simultaneous multi-robot motion planning with projected diffusion models},
  author={Liang, Jinhao and Christopher, Jacob K and Koenig, Sven and Fioretto, Ferdinando},
    booktitle={International Conference on Machine Learning (ICML)},
    year={2025}
}

@inproceedings{luo2024potential,
  title={Potential based diffusion motion planning},
  author={Luo, Yunhao and Sun, Chen and Tenenbaum, Joshua B and Du, Yilun},
  booktitle={International Conference on Machine Learning (ICML)},
  year={2024}
}

@inproceedings{lu2025makes,
  title={What makes a good diffusion planner for decision making?},
  author={Lu, Haofei and Han, Dongqi and Shen, Yifei and Li, Dongsheng},
  booktitle={International Conference on Learning Representations (ICLR)},
  year={2025},
}

@inproceedings{dong2024cleandiffuser,
  title={Cleandiffuser: An easy-to-use modularized library for diffusion models in decision making},
  author={Dong, Zibin and Yuan, Yifu and Hao, Jianye and Ni, Fei and Ma, Yi and Li, Pengyi and Zheng, Yan},
  booktitle={Advances in Neural Information Processing Systems (NeurIPS), Datasets and Benchmarks Track},
  year={2024}
}

@inproceedings{wang2025inference,
  title={Inference-time policy steering through human interactions},
  author={Wang, Yanwei and Wang, Lirui and Du, Yilun and Sundaralingam, Balakumar and Yang, Xuning and Chao, Yu-Wei and P{\'e}rez-D’Arpino, Claudia and Fox, Dieter and Shah, Julie},
  booktitle={IEEE International Conference on Robotics and Automation (ICRA)},
  year={2025}
}

@inproceedings{ren2026driftlite,
  title={Driftlite: Lightweight drift control for inference-time scaling of diffusion models},
  author={Ren, Yinuo and Gao, Wenhao and Ying, Lexing and Rotskoff, Grant M and Han, Jiequn},
  booktitle={International Conference on Learning Representations (ICLR)},
  year={2026}
}

@inproceedings{park2024foundation,
  title={Foundation policies with hilbert representations},
  author={Park, Seohong and Kreiman, Tobias and Levine, Sergey},
  booktitle={International Conference on Machine Learning (ICML)},
  year={2024}
}

@inproceedings{park2026dual,
  title={Dual Goal Representations},
  author={Park, Seohong and Mann, Deepinder and Levine, Sergey},
  booktitle={International Conference on Learning Representations (ICLR)},
  year={2026}
}

@inproceedings{park2026transitive,
  title={Transitive RL: Value Learning via Divide and Conquer},
  author={Park, Seohong and Oberai, Aditya and Atreya, Pranav and Levine, Sergey},
  booktitle={International Conference on Learning Representations (ICLR)},
  year={2026}
}

@inproceedings{farebrother2025temporal,
  title={Temporal difference flows},
  author={Farebrother, Jesse and Pirotta, Matteo and Tirinzoni, Andrea and Munos, R{\'e}mi and Lazaric, Alessandro and Touati, Ahmed},
  booktitle={International Conference on Machine Learning (ICML)},
  year={2025}
}

@inproceedings{park2025horizon,
  title={Horizon reduction makes rl scalable},
  author={Park, Seohong and Frans, Kevin and Mann, Deepinder and Eysenbach, Benjamin and Kumar, Aviral and Levine, Sergey},
  booktitle={Advances in Neural Information Processing Systems (NeurIPS)},
  year={2025}
}

@inproceedings{ahn2025option,
  title={Option-aware temporally abstracted value for offline goal-conditioned reinforcement learning},
  author={Ahn, Hongjoon and Choi, Heewoong and Han, Jisu and Moon, Taesup},
  booktitle={Advances in Neural Information Processing Systems (NeurIPS)},
  year={2025}
}

@inproceedings{baek2025graph,
  title={Graph-assisted stitching for offline hierarchical reinforcement learning},
  author={Baek, Seungho and Park, Taegeon and Park, Jongchan and Oh, Seungjun and Kim, Yusung},
  booktitle={International Conference on Machine Learning (ICML)},
  year={2025}
}

@inproceedings{haramati2026hierarchical,
  title={Hierarchical Entity-centric Reinforcement Learning with Factored Subgoal Diffusion},
  author={Haramati, Dan and Qi, Carl and Daniel, Tal and Zhang, Amy and Tamar, Aviv and Konidaris, George},
  booktitle={International Conference on Learning Representations (ICLR)},
  year={2026}
}

@inproceedings{jeon2025tree,
  title={Tree-Guided Diffusion Planner},
  author={Jeon, Hyeonseong and Min, Cheolhong and Park, Jaesik},
  booktitle={Advances in Neural Information Processing Systems (NeurIPS)},
  year={2025}
}

@inproceedings{feng2024resisting,
  title={Resisting stochastic risks in diffusion planners with the trajectory aggregation tree},
  author={Feng, Lang and Gu, Pengjie and An, Bo and Pan, Gang},
  booktitle={International Conference on Machine Learning (ICML)},
  year={2024}
}

@inproceedings{zhang2026compositional,
  title={Compositional Visual Planning via Inference-Time Diffusion Scaling},
  author={Zhang, Yixin and Luo, Yunhao and Mishra, Utkarsh Aashu and Shin, Woo Chul and Chen, Yongxin and Xu, Danfei},
  booktitle={International Conference on Learning Representations (ICLR)},
  year={2026}
}

@article{lu2026improving,
  title={Improving Diffusion Planners by Self-Supervised Action Gating with Energies},
  author={Lu, Yuan and Han, Dongqi and Wang, Yansen and Li, Dongsheng},
  journal={arXiv preprint arXiv:2603.02650},
  year={2026}
}

@article{opryshko2025test,
  title={Test-time graph search for goal-conditioned reinforcement learning},
  author={Opryshko, Evgenii and Quan, Junwei and Voelcker, Claas and Du, Yilun and Gilitschenski, Igor},
  journal={arXiv preprint arXiv:2510.07257},
  year={2025}
}

@article{dong2026proximal,
  title={Proximal Action Replacement for Behavior Cloning Actor-Critic in Offline Reinforcement Learning},
  author={Dong, Jinzong and Huang, Wei and Zhang, Jianshu and Chen, Zhuo and Yuan, Xinzhe and Gu, Qinying and Jiang, Zhaohui and Ye, Nanyang},
  journal={arXiv preprint arXiv:2602.07441},
  year={2026}
}

@article{clark2025you,
  title={What Do You Need for Diverse Trajectory Composition in Diffusion Planning?},
  author={Clark, Quentin and Shkurti, Florian},
  journal={arXiv preprint arXiv:2505.18083},
  year={2025}
}

@inproceedings{nandiraju2025hdflow,
  title={HDFlow: Hierarchical Diffusion-Flow Planning for Long-horizon Robotic Assembly},
  author={Nandiraju, Gireesh and Ju, Yuanliang and Xu, Chaoyi and Wang, He},
  booktitle={NeurIPS 2025 Workshop on Embodied World Models for Decision Making},
  year={2025}
}
\bibliographystyle{neurips_2026}

%=============================================================================
% APPENDIX
%=============================================================================
\newpage
\appendix
\crefalias{section}{appendix}
\crefalias{subsection}{appendix}
\onecolumn
\hypersetup{linkcolor=black}
\section*{\huge Appendix}
\vspace{5mm}
\section*{\LARGE Table of Contents}
\vspace{-5mm}
\noindent\rule{\linewidth}{0.8pt}
\vspace{-7mm}
\begin{itemize}[label=, leftmargin=*]
	\item \hyperref[app:notation]{\textbf{Appendix A.} Notation \dotfill \pageref{app:notation}}
	\item \hyperref[app:ext_related]{\textbf{Appendix B.} Extended Related Work \dotfill \pageref{app:ext_related}}
	\item \hyperref[sec:tfg]{\textbf{Appendix C.} Training-free Diffusion Guidance \dotfill \pageref{sec:tfg}}
	\item \hyperref[app:proofs]{\textbf{Appendix D.} Proofs \dotfill \pageref{app:proofs}}
	      \begin{itemize}[label=, leftmargin=1.5em]
		      \item \hyperref[app:proof_density]{D.1. Proof of \Cref{prop:density_proxy} \dotfill \pageref{app:proof_density}}
		      \item \hyperref[app:proof_overlap]{D.2. Proof of \Cref{prop:overlap} \dotfill \pageref{app:proof_overlap}}
		      \item \hyperref[app:proof_guided]{D.3. Proof of \Cref{prop:guided} \dotfill \pageref{app:proof_guided}}
            \item \hyperref[app:local_decomposition]{D.4. Decomposition into Local Reconstruction Errors \dotfill \pageref{app:local_decomposition}}
	      \end{itemize}
	\item \hyperref[app:practical_alg]{\textbf{Appendix E.} Algorithm \dotfill \pageref{app:practical_alg}}
	\item \hyperref[app:impl_detail]{\textbf{Appendix F.} Implementation Details \dotfill \pageref{app:impl_detail}}
	      \begin{itemize}[label=, leftmargin=1.5em]
		      \item \hyperref[app:envs]{F.1. Environments and Datasets \dotfill \pageref{app:envs}}
		      \item \hyperref[app:arch]{F.2. Diffusion Model Architecture and Training \dotfill \pageref{app:arch}}
		      \item \hyperref[app:rcd_hyperparams]{F.3. Hyperparameters \dotfill \pageref{app:rcd_hyperparams}}
		      \item \hyperref[app:eval_protocol]{F.4. Evaluation Protocol \dotfill \pageref{app:eval_protocol}}
		      \item \hyperref[app:compute]{F.5. Compute Resources \dotfill \pageref{app:compute}}
		      \item \hyperref[app:low_level]{F.6. Low-Level Controllers \dotfill \pageref{app:low_level}}
	      \end{itemize}
	\item \hyperref[app:add_results]{\textbf{Appendix G.} Additional Results \dotfill \pageref{app:add_results}}
	      \begin{itemize}[label=, leftmargin=1.5em]
		      \item \hyperref[app:plan_quality_viz]{G.1. Additional Plan Quality Visualizations \dotfill \pageref{app:plan_quality_viz}}
		      \item \hyperref[app:add_ablation]{G.2. Additional Ablation Studies \dotfill \pageref{app:add_ablation}}
            \item \hyperref[app:proxy_feasibility]{G.3. Reconstruction Error and Plan Feasibility \dotfill \pageref{app:proxy_feasibility}}
            \item \hyperref[app:unconditional_composition]{G.4. RCD with Unconditional Score Composition \dotfill \pageref{app:unconditional_composition}}
            \item \hyperref[app:ood_composition]{G.5. Composition with Out-of-Distribution Constraints \dotfill \pageref{app:ood_composition}}
            \item \hyperref[app:failure_analysis]{G.6. Failure Analysis \dotfill \pageref{app:failure_analysis}}
	      \end{itemize}
	\item \hyperref[app:baseline_sources]{\textbf{Appendix H.} Baseline Result Sources \dotfill \pageref{app:baseline_sources}}
\end{itemize}
\vspace{-5mm}
\noindent\rule{\linewidth}{0.8pt}
\newpage

\section{Notation}
\label{app:notation}

\begin{table}[ht]
	\centering
	\caption{\textbf{Table of notation.}}
	\label{tab:notation}
	\begin{tabularx}{\linewidth}{lX}
		\toprule
		\textbf{Notation} & \textbf{Description} \\
		\midrule
		\multicolumn{2}{l}{\textbf{Factor Graph \& Plan Structure}} \\
		$\bs{\tau} = (x_1, \ldots, x_N) \in \R^{N \times D}$ & Global plan consisting of $N$ variable nodes \\
		$x_i \in \R^D$ & Variable node $i$ (state or state-action pair) \\
		$D$ & Per-variable dimension (state or state-action) \\
		$y_j$ & Factor node $j$ (local segment, contiguous subsequence of $\bs{\tau}$) \\
		$M$ & Number of factors (local segments) \\
		$N$ & Number of variable nodes \\
		$y_j \cap y_{j+1}$ & Overlap region between adjacent segments $j$ and $j{+}1$ \\
		$d_i$ & Degree of variable $x_i$ ($d_i = 2$ for overlap, $d_i = 1$ otherwise) \\
		\midrule
		\multicolumn{2}{l}{\textbf{Diffusion Process}} \\
		$\bs{\tau}^{(t)}$, $y_j^{(t)}$ & Noisy sample at diffusion timestep $t$ \\
		$\bs{\tau}^{(0)}$, $y_j^{(0)}$ & Clean sample \\
		$\alpha_t$ & Cumulative noise schedule (DDPM) \\
		$\bs{\epsilon}$ & Gaussian noise \\
		$T$ & Total number of diffusion timesteps \\
		$\bs{\epsilon}_\theta(y_j^{(t)}, t)$ & Noise prediction network; neighbor conditioning implicit (see \cref{sec:comp_diffusion}) \\
		$\mathbf{s}_\theta(\cdot, t)$ & Score function $\mathbf{s}_\theta := -\bs{\epsilon}_\theta/\sqrt{1{-}\alpha_t}$ (see \cref{sec:tweedie}) \\
		$\bar{\bs{\epsilon}}_\theta$ & Composed noise prediction (averaging in overlaps) \\
		$\hat{\bs{\tau}}_0^{(t)}$, $\hat{y}_{j,0}^{(t)}$ & Tweedie denoised estimate \\
		$\hat{\bs{\tau}}_0^{\mathrm{rec}}$ & Composed Tweedie reconstruction \\
		$\bs{\mu}_\theta$, $\sigma_t$ & Predicted mean and standard deviation of reverse process \\
		\midrule
		\multicolumn{2}{l}{\textbf{RCD-Specific}} \\
		$\mathcal{E}_{\mathrm{recon}}(\hat{\bs{\tau}}_0;\, s)$ & Self-reconstruction error at probe level $s$ \\
		$\mathcal{E}_{\mathrm{ov}}(\hat{y}_{1:M})$ & Overlap consistency \\
		$\mathcal{E}_{\mathrm{RCD}}$ & Combined RCD guidance objective \\
		$s$ & Probe timestep for self-reconstruction \\
		$w$ & Guidance weight \\
		$\lambda_{\mathrm{ov}}$ & Weight for overlap consistency \\
		$\tilde{p}(\bs{\tau})$ & RCD-guided (tilted) distribution \\
		$\delta$ & Numerical stability constant for gradient normalization \\
		\bottomrule
	\end{tabularx}
\end{table}

\section{Extended Related Work}
\label{app:ext_related}

\paragraph{Diffusion and Flow Models for Offline Decision Making.}
Beyond trajectory planning, diffusion and flow models have been widely adopted as policy and value classes in offline RL and imitation learning. Diffusion Policy~\cite{chi2023diffusionpolicy} and Latent Diffusion Planning~\cite{xie2025latent} train diffusion- (or latent-diffusion-) based visuomotor policies, and Ada-Diffuser~\cite{feng2026ada} extends this line with a causal diffusion framework that identifies latent dynamics from short temporal observation blocks and applies to both planning and policy learning. Diffusion Q-Learning~\cite{wang2022diffusion} and more recent flow-matching-based variants such as FQL~\cite{park2025flow}, DEAS~\cite{kim2026deas}, MAC~\cite{park2026scalable}, and DQC~\cite{li2026decoupled} combine diffusion/flow policy classes with Q-learning. floq~\cite{agrawalla2025floq} instead applies flow matching on the critic side. On a different axis, training-time data augmentations such as Proximal Action Replacement~\cite{dong2026proximal} progressively replace low-value actions in the offline buffer with high-value ones from a stable actor to lift the performance ceiling of behavior-cloning-regularized actor-critic methods. A complementary line treats guidance itself as a policy-improvement step~\cite{frans2025diffusion, jackson2024policy}, including Prior-Guided Diffusion Planning~\cite{ki2025prior}, which replaces the standard Gaussian noise prior of a behavior-cloned diffusion planner with a state-conditioned learnable prior concentrated on high-value trajectories. Sequence-level diffusion models have also been extended along the time axis: Diffusion Forcing~\cite{chen2024diffusion} unifies next-token prediction with full-sequence denoising, FloodDiffusion~\cite{cai2025flood} streams motion generation online, and DiffuserLite~\cite{dong2024diffuserlite} employs a coarse-to-fine planning refinement process to accelerate inference. Structured variants of diffusion planners include Hierarchical Diffusion for Offline Decision Making~\cite{li2023hierarchical}, which uses a two-level trajectory-level diffusion with classifier-free guidance, and Potential-Based Diffusion Motion Planning~\cite{luo2024potential}, which parametrizes motion-planning potentials with diffusion models and composes constraints by summing the learned potentials, and HDFlow~\cite{nandiraju2025hdflow}, which couples a high-level diffusion planner over latent subgoals with a low-level rectified-flow trajectory generator for long-horizon robotic assembly. Diffusion planners have further been deployed across domains: Diffusion-Based Planning for Autonomous Driving~\cite{zheng2025diffusion} brings related ideas to driving-policy generation, and multi-robot motion planning methods~\cite{shaoul2025multi, liang2025simultaneous} compose single-robot diffusion planners under collision and kinematic constraints. An empirical study of design principles by \citet{lu2025makes} systematically evaluates over $6{,}000$ diffusion-planner configurations for offline RL, \citet{clark2025you} identifies shift equivariance and local receptive fields as the architectural ingredients that enable diffusion planners to compose diverse trajectories from short BC-trained segments, and CleanDiffuser~\cite{dong2024cleandiffuser} provides a modular library that unifies many of the above components. A broader set of inference-time methods such as Inference-Time Policy Steering~\cite{wang2025inference} and DriftLite~\cite{ren2026driftlite} modify a pretrained diffusion model at test time without additional training, but target settings orthogonal to compositional long-horizon planning.

\paragraph{Offline Goal-Conditioned RL and Horizon Reduction.}
Offline goal-conditioned RL~\cite{ghosh2019learning, wang2023optimal, eysenbach2022contrastive, kostrikov2022offline, park2023hiql, park2025ogbench} studies how to train a policy that can reach any target state from any starting state using only a fixed, pre-collected dataset. Representation-learning approaches such as HILP~\cite{park2024foundation}, Dual Goal Representations~\cite{park2026dual}, Transitive RL~\cite{park2026transitive}, and Temporal Difference Flows~\cite{farebrother2025temporal} learn geometry-aware state embeddings tailored for long-horizon value estimation or temporal-distance prediction. Horizon-reduction strategies explicitly shrink the effective planning horizon and shift reasoning to a hierarchical or graph-based representation: SHARSA~\cite{park2025horizon}, option-aware temporal abstraction~\cite{ahn2025option}, Graph-Assisted Stitching~\cite{baek2025graph}, hierarchical RL with factored subgoal diffusion~\cite{haramati2026hierarchical}, and Test-Time Graph Search~\cite{opryshko2025test}, which builds a weighted graph over offline dataset states and assembles a subgoal sequence at inference time for a frozen goal-conditioned policy, all trade long reasoning chains for shorter, compositional segments. Along the data-centric axis, trajectory-stitching methods synthesize longer trajectories \emph{offline}, either via model-based return-conditioned supervised learning~\cite{zhou2023free} or generative augmentation~\cite{li2024diffstitch, jackson2024policy, lee2024gta, lee2025scots, chen2025extendable}, after which a downstream policy is retrained on the enriched dataset.

\paragraph{Inference-Time Search and Plan Refinement for Diffusion Planners.}
A recent line of work explores explicit search or population-based selection over diffusion-generated plans. Monte Carlo Tree Diffusion~\cite{yoon2025monte} and its faster variant~\cite{yoon2025fast} perform MCTS over denoising trajectories, and Tree-Guided Diffusion Planner~\cite{jeon2025tree} balances exploration and exploitation through a bi-level tree-search sampling procedure. Within the compositional diffusion setting, CDGS~\cite{mishra2026compositional} augments segment-level score composition with population-based guided search and plan pruning. Plan-refinement methods instead detect or project infeasible trajectories: RGG~\cite{lee2023refining} trains a time-dependent classifier to flag infeasible plans, LoMAP~\cite{lee2025local} projects intermediate samples onto a local approximation of the data manifold, TAT~\cite{feng2024resisting} aggregates multiple diffusion-generated trajectories into a dynamic tree and filters unreliable predictions via weighted majority voting over shared states, SAGE~\cite{lu2026improving} re-ranks sampled plans at inference time using a JEPA-based latent prediction error as a feasibility energy combined with value estimates, and AdaptDiffuser~\cite{liang2023adaptdiffuser} generates synthetic trajectories via reward-gradient guidance, filters them with a discriminator, and fine-tunes the diffusion planner on the filtered high-quality samples in a self-evolving loop. In contrast, RCD uses the pretrained diffusion model's own self-reconstruction error as an intrinsic density proxy, directly at the composed-trajectory level where mode-averaging originates, and requires no auxiliary classifier, manifold estimator, or search tree. Compositional Visual Planning~\cite{zhang2026compositional} addresses long-horizon video planning by enforcing boundary residual constraints across overlapping chunks via message passing. RCD instead introduces a self-reconstruction density signal with a formal ELBO connection that explicitly targets mode-averaging, and applies gradient guidance on a single composed trajectory rather than residual-based message passing.

\section{Training-free Diffusion Guidance}
\label{sec:tfg}

Training-free guidance methods steer diffusion sampling toward a target condition $y$ without retraining the model.
Starting from classifier guidance~\cite{dhariwal2021diffusion}, which modifies the score via $\nabla_{\bs{\tau}^{(t)}} \log p(\bs{\tau}^{(t)} | y) = \nabla_{\bs{\tau}^{(t)}} \log p(\bs{\tau}^{(t)}) + \nabla_{\bs{\tau}^{(t)}} \log p(y | \bs{\tau}^{(t)})$, recent methods avoid the need for a time-dependent classifier by operating directly on the Tweedie estimate.
Given a differentiable loss function $\mathcal{L}(\hat{\bs{\tau}}_0^{(t)})$ defined on the clean-data support, the likelihood term is approximated as~\cite{chung2023diffusion, song2023loss}:
\begin{align}
	\label{eq:tfg}
	\nabla_{\bs{\tau}^{(t)}} \log p(y | \bs{\tau}^{(t)}) \approx -\gamma\, \nabla_{\bs{\tau}^{(t)}} \mathcal{L}\!\bigl(\hat{\bs{\tau}}_0^{(t)}\bigr),
\end{align}
where $\gamma > 0$ is a guidance scale and the gradient is taken with respect to $\bs{\tau}^{(t)}$ through the Tweedie estimate.
This approach requires only a pretrained diffusion model and a differentiable objective, making it applicable to a wide range of tasks without additional training~\cite{yang2024guidance, he2024manifold}.

\section{Proofs}
\label{app:proofs}

\subsection{Proof of \Cref{prop:density_proxy} (Reconstruction Error as Density Proxy)}
\label{app:proof_density}

We provide the full derivation of the ELBO connection.

\paragraph{Per-segment ELBO.}
For a single segment $y_j$ with clean sample $y_j^{(0)}$, the per-segment reconstruction error at probe timestep $s$ is:
\begin{align}
	\E_{\bs{\epsilon}}\!\left[\|y_j^{(0)} - \hat{y}_{j,0}^{(s)}\|^2\right]
	 & = \E_{\bs{\epsilon}}\!\left[\left\| y_j^{(0)} - \frac{\sqrt{\alpha_s}\, y_j^{(0)} + \sqrt{1{-}\alpha_s}\, \bs{\epsilon} - \sqrt{1{-}\alpha_s}\, \bs{\epsilon}_\theta(y_j^{(s)}, s)}{\sqrt{\alpha_s}} \right\|^2\right] \notag \\
	 & = \frac{1-\alpha_s}{\alpha_s}\, \E_{\bs{\epsilon}}\!\left[\| \bs{\epsilon} - \bs{\epsilon}_\theta(y_j^{(s)}, s) \|^2\right].
	\label{eq:persegment_recon}
\end{align}
Multiplying both sides by $\alpha_s / (1{-}\alpha_s)$ and summing over $s = 1, \ldots, T$:
\begin{align}
	\sum_{s=1}^T \frac{\alpha_s}{1{-}\alpha_s}\, \E_{\bs{\epsilon}}\!\left[\|y_j^{(0)} - \hat{y}_{j,0}^{(s)}\|^2\right]
	= \sum_{s=1}^T \E_{\bs{\epsilon}}\!\left[\| \bs{\epsilon} - \bs{\epsilon}_\theta(y_j^{(s)}, s) \|^2\right].
\end{align}
The right-hand side is the negative ELBO for the local diffusion model~\cite{ho2020denoising}, providing an upper bound on $-\log p_\theta(y_j^{(0)})$.

\paragraph{Composed trajectory.}
Let $\bar{\bs{\epsilon}}_\theta(\hat{\bs{\tau}}_s, s)$ denote the \emph{composed noise prediction}: equal to $\bs{\epsilon}_\theta^{(j)}$ on non-overlap variables $x_i \in y_j$ and to $\tfrac{1}{2}\bigl(\bs{\epsilon}_\theta^{(k)} + \bs{\epsilon}_\theta^{(k+1)}\bigr)$ on overlap variables $x_i \in y_k \cap y_{k+1}$.
By the definition of the composed Tweedie reconstruction $\hat{\bs{\tau}}_0^{\mathrm{rec}}$ in \cref{def:recon}, a direct algebraic manipulation with $\hat{\bs{\tau}}_s = \sqrt{\alpha_s}\,\hat{\bs{\tau}}_0 + \sqrt{1{-}\alpha_s}\,\bs{\epsilon}$ yields
\begin{align}
	\hat{\bs{\tau}}_0 - \hat{\bs{\tau}}_0^{\mathrm{rec}}
	= -\sqrt{\tfrac{1-\alpha_s}{\alpha_s}}\bigl(\bs{\epsilon} - \bar{\bs{\epsilon}}_\theta(\hat{\bs{\tau}}_s, s)\bigr).
\end{align}
Squaring and taking the expectation over $\bs{\epsilon}$ gives
\begin{align}
	\frac{\alpha_s}{1{-}\alpha_s}\, \mathcal{E}_{\mathrm{recon}}(\hat{\bs{\tau}}_0;\, s)
	= \E_{\bs{\epsilon}}\!\left[\bigl\| \bs{\epsilon} - \bar{\bs{\epsilon}}_\theta\bigl(\sqrt{\alpha_s}\,\hat{\bs{\tau}}_0 + \sqrt{1{-}\alpha_s}\,\bs{\epsilon},\, s\bigr) \bigr\|^2\right].
\end{align}
Summing over $s = 1, \ldots, T$, the right-hand side is the DDPM simple loss $L_{\mathrm{simple}}$ for the composed denoiser $\bar{\bs{\epsilon}}_\theta$.
By the standard variational argument~\cite{ho2020denoising}, $L_{\mathrm{simple}}$ upper-bounds $-\log p_\theta(\hat{\bs{\tau}}_0)$ up to a timestep re-weighting, where $p_\theta$ is the distribution implicitly defined by the composed denoiser $\bar{\bs{\epsilon}}_\theta$.
This establishes the inequality in \cref{eq:elbo_connection}. \qed

\subsection{Proof of \Cref{prop:overlap} (Overlap Consistency Measures Score Disagreement)}
\label{app:proof_overlap}

Both segments $k$ and $k{+}1$ receive the same noisy trajectory $\hat{\bs{\tau}}_s$ in the overlap region $y_k \cap y_{k+1}$.
Let $z := \hat{\bs{\tau}}_s\big|_{y_k \cap y_{k+1}}$ be the common noisy overlap.
The Tweedie estimates from each segment in the overlap are:
\begin{align}
	\hat{y}_{k,0}\big|_{y_k \cap y_{k+1}}   & = \frac{z - \sqrt{1{-}\alpha_s}\, \bs{\epsilon}_\theta^{(k)}(y_k^{(s)}, s)}{\sqrt{\alpha_s}},       \\
	\hat{y}_{k+1,0}\big|_{y_k \cap y_{k+1}} & = \frac{z - \sqrt{1{-}\alpha_s}\, \bs{\epsilon}_\theta^{(k+1)}(y_{k+1}^{(s)}, s)}{\sqrt{\alpha_s}},
\end{align}
where $\bs{\epsilon}_\theta^{(j)}(y_j^{(s)}, s)$ is the noise prediction from segment $j$ evaluated on its full noisy input $y_j^{(s)}$.
Taking the difference:
\begin{align}
	\hat{y}_{k,0}\big|_{y_k \cap y_{k+1}} - \hat{y}_{k+1,0}\big|_{y_k \cap y_{k+1}}
	= \frac{\sqrt{1{-}\alpha_s}}{\sqrt{\alpha_s}}\, \bigl(\bs{\epsilon}_\theta^{(k+1)} - \bs{\epsilon}_\theta^{(k)}\bigr)\big|_{y_k \cap y_{k+1}}.
\end{align}
Squaring yields:
\begin{align}
	\|\hat{y}_{k,0} - \hat{y}_{k+1,0}\|^2\big|_{y_k \cap y_{k+1}} = \frac{1{-}\alpha_s}{\alpha_s} \|\bs{\epsilon}_\theta^{(k+1)} - \bs{\epsilon}_\theta^{(k)}\|^2\big|_{y_k \cap y_{k+1}}.
\end{align}
Substituting the score identity $\bs{\epsilon}_\theta^{(j)} = -\sqrt{1{-}\alpha_s}\,\mathbf{s}_\theta^{(j)}$ from \cref{sec:tweedie} yields \cref{eq:overlap_as_score}. \qed

\paragraph{Connection to marginal consistency.}
In the Bethe approximation, marginal consistency requires $p_k(x_i) = p_{k+1}(x_i)$ for all $x_i \in y_k \cap y_{k+1}$.
Taking gradients: $\nabla_{x_i} \log p_k(x_i) = \nabla_{x_i} \log p_{k+1}(x_i)$, i.e., scores from both segments must agree at overlap variables.
$\mathcal{E}_{\mathrm{ov}}$ is a relaxation of this hard constraint, penalizing violations via the squared score difference.

\subsection{Proof of \Cref{prop:guided} (Guided Distribution)}
\label{app:proof_guided}

\begin{restatable}[Guided Distribution]{proposition}{PropGuided}
	\label{prop:guided}
	Let $p_\theta(\bs{\tau})$ denote the composed distribution obtained via score averaging.
	The RCD-guided reverse process targets the modified distribution
	\begin{align}
		\label{eq:tilted}
		\tilde{p}(\bs{\tau}) \propto p_\theta(\bs{\tau}) \cdot \exp\bigl(-w \cdot \mathcal{E}_{\mathrm{RCD}}(\bs{\tau})\bigr),
	\end{align}
	which up-weights high-density, globally coherent plans and down-weights mode-averaged plans.
\end{restatable}

\begin{proof}
	Consider the reverse-time SDE for the composed distribution $p_\theta(\bs{\tau})$~\cite{song2021scorebased}:
	\begin{align}
		\diff \bs{\tau} = \bigl[f(\bs{\tau}, t) - g(t)^2 \nabla_{\bs{\tau}} \log p_t(\bs{\tau})\bigr] \diff t + g(t) \diff \bar{\bs{w}},
	\end{align}
	where $f$ and $g$ are the drift and diffusion coefficients and $\bar{\bs{w}}$ is a reverse-time Wiener process.

	The RCD guidance modifies the score by adding $-w \nabla_{\bs{\tau}} \mathcal{E}_{\mathrm{RCD}}(\bs{\tau})$:
	\begin{align}
		\diff \bs{\tau} = \bigl[f(\bs{\tau}, t) - g(t)^2 \bigl(\nabla_{\bs{\tau}} \log p_t(\bs{\tau}) - w \nabla_{\bs{\tau}} \mathcal{E}_{\mathrm{RCD}}(\bs{\tau})\bigr)\bigr] \diff t + g(t) \diff \bar{\bs{w}}.
	\end{align}
	This is the reverse SDE for the modified distribution:
	\begin{align}
		\tilde{p}_t(\bs{\tau}) \propto p_t(\bs{\tau}) \cdot \exp\bigl(-w \cdot \mathcal{E}_{\mathrm{RCD}}(\bs{\tau})\bigr),
	\end{align}
	since $\nabla_{\bs{\tau}} \log \tilde{p}_t(\bs{\tau}) = \nabla_{\bs{\tau}} \log p_t(\bs{\tau}) - w \nabla_{\bs{\tau}} \mathcal{E}_{\mathrm{RCD}}(\bs{\tau})$.

	At $t = 0$, the marginal is $\tilde{p}(\bs{\tau}) \propto p_\theta(\bs{\tau}) \cdot \exp(-w \cdot \mathcal{E}_{\mathrm{RCD}}(\bs{\tau}))$.
	Since $\mathcal{E}_{\mathrm{RCD}} \geq 0$ with equality for perfectly reconstructed, globally coherent plans, the exponential factor $\exp(-w\, \mathcal{E}_{\mathrm{RCD}})$ up-weights plans with low reconstruction error and high overlap consistency, while down-weighting mode-averaged plans.
\end{proof}

\subsection{Decomposition into Local Reconstruction Errors}
\label{app:local_decomposition}

The reconstruction and overlap terms are complementary because averaging can hide opposing reconstruction errors at segment boundaries. We make this relation explicit for the chain structure in \cref{sec:comp_diffusion}, where each variable belongs to at most two segments. Fix a candidate trajectory $\hat{\bs{\tau}}_0$ and a probe timestep $s$. For a noise sample $\bs{\epsilon}\sim\mathcal{N}(\mathbf{0},\mathbf{I})$, let $r_{\mathrm{recon}}(\bs{\epsilon})=\|\hat{\bs{\tau}}_0-\hat{\bs{\tau}}_0^{\mathrm{rec}}\|^2$ and let $r_{\mathrm{ov}}(\bs{\epsilon})$ be the overlap term in \cref{eq:overlap} for the same reconstruction. For a segment $j$, define
\begin{align}
 r_j(\bs{\epsilon}) = \sum_{i:x_i\in y_j}\frac{1}{d_i}\bigl\|\hat{\bs{\tau}}_0\big|_{x_i}-\hat{y}_{j,0}\big|_{x_i}\bigr\|^2,
\end{align}
where $d_i\in\{1,2\}$ is the degree of variable $x_i$ and $\cdot\big|_{x_i}$ denotes restriction to the coordinates of the global variable $x_i$. Thus overlap coordinates receive weight $1/2$ in each segment. For a shared variable with target $x$ and two local reconstructions $a$ and $b$,
\begin{align}
 \frac{1}{2}\|x-a\|^2+\frac{1}{2}\|x-b\|^2
 = \left\|x-\frac{a+b}{2}\right\|^2+\frac{1}{4}\|a-b\|^2.
\end{align}
Summing over shared and non-shared coordinates gives the samplewise identity
\begin{align}
 r_{\mathrm{recon}}(\bs{\epsilon})+\frac{M-1}{4}\,r_{\mathrm{ov}}(\bs{\epsilon})
 = \sum_{j=1}^M r_j(\bs{\epsilon}).
 \label{eq:local_decomposition}
\end{align}
Taking the expectation over $\bs{\epsilon}$ and using \cref{eq:recon_error} gives
\begin{align}
 \mathcal{E}_{\mathrm{recon}}(\hat{\bs{\tau}}_0;\,s)+\frac{M-1}{4}\,\E_{\bs{\epsilon}}[r_{\mathrm{ov}}(\bs{\epsilon})]
 = \sum_{j=1}^M\E_{\bs{\epsilon}}[r_j(\bs{\epsilon})].
\end{align}

The factor $M-1$ follows from the normalization of the overlap term in \cref{eq:overlap}. This decomposition shows how overlap consistency accounts for local reconstruction errors that cancel in the merged estimate, but does not establish calibrated density under the global trajectory distribution.

%=============================================================================
% PRACTICAL ALGORITHM
%=============================================================================
\section{Algorithm}
\label{app:practical_alg}

Each RCD guidance step requires one additional forward pass through the score network (to compute the reconstruction at probe level $s$) and one backward pass for the gradient.
The expectation in \cref{eq:recon_error} is approximated with a single Monte Carlo noise sample.
The complete procedure is summarized in \cref{alg:rcd}.

\begin{algorithm}[H]
	\caption{\textbf{Refining Compositional Diffusion}}
	\label{alg:rcd}
	\begin{algorithmic}[1]
		\STATE \textbf{Require:} Shared diffusion model $\bs{\epsilon}_\theta$, guidance weight $w$, overlap weight $\lambda_{\mathrm{ov}}$, probe level $s$
		\STATE Observe current state $\bs{s}$; initialize $\bs{\tau}^{(T)} \sim \mathcal{N}(\mathbf{0}, \mathbf{I})$
		\FOR{$t = T, \ldots, 1$}
		\NoNumber{\small{\color{gray}\texttt{// compositional denoising via score averaging over segments $\{y_j\}$}}}
		\STATE $\bar{\bs{\epsilon}}_\theta(\bs{\tau}^{(t)}, t) \gets$ average per-segment predictions $\bs{\epsilon}_\theta(y_j^{(t)}, t)$, sharing in overlaps
		\STATE $\hat{\bs{\tau}}_0^{(t)} \gets \bigl(\bs{\tau}^{(t)} - \sqrt{1{-}\alpha_t}\, \bar{\bs{\epsilon}}_\theta(\bs{\tau}^{(t)}, t)\bigr) / \sqrt{\alpha_t}$ \hfill {\small\color{gray}\texttt{// Tweedie estimate}}
		\NoNumber{\small{\color{gray}\texttt{// self-reconstruction: perturb and reconstruct}}}
		\STATE Sample $\bs{\epsilon} \sim \mathcal{N}(\mathbf{0}, \mathbf{I})$
		\STATE $\hat{\bs{\tau}}_s \gets \sqrt{\alpha_s}\, \hat{\bs{\tau}}_0^{(t)} + \sqrt{1{-}\alpha_s}\, \bs{\epsilon}$
		\FOR{each segment $j = 1, \ldots, M$}
		\STATE $\hat{y}_{j,0} \gets \bigl(\hat{y}_{j,s} - \sqrt{1{-}\alpha_s}\, \bs{\epsilon}_\theta(\hat{y}_{j,s}, s)\bigr) / \sqrt{\alpha_s}$ \hfill {\small\color{gray}\texttt{// per-segment Tweedie}}
		\ENDFOR
		\NoNumber{\small{\color{gray}\texttt{// compose per-segment estimates via overlap averaging (Eq.~\ref{eq:bethe})}}}
		\STATE $\hat{\bs{\tau}}_0^{\mathrm{rec}} \gets$ merge $\{\hat{y}_{j,0}\}_{j=1}^{M}$, averaging in overlap regions
		\NoNumber{\small{\color{gray}\texttt{// RCD guidance objective (Eq.~\ref{eq:rcd_obj})}}}
		\STATE $\mathcal{E}_{\mathrm{recon}} \gets \| \hat{\bs{\tau}}_0^{(t)} - \hat{\bs{\tau}}_0^{\mathrm{rec}} \|^2$ \hfill {\small\color{gray}\texttt{// self-reconstruction error}}
		\STATE $\mathcal{E}_{\mathrm{ov}} \gets \frac{1}{M{-}1} \sum_{k=1}^{M-1} \| \hat{y}_{k,0} - \hat{y}_{k+1,0} \|^2\big|_{y_k \cap y_{k+1}}$ \hfill {\small\color{gray}\texttt{// overlap consistency}}
		\STATE $\mathbf{g}^{(t)} \gets \nabla_{\bs{\tau}^{(t)}} \bigl(\mathcal{E}_{\mathrm{recon}} + \lambda_{\mathrm{ov}}\, \mathcal{E}_{\mathrm{ov}}\bigr)$
		\STATE $\tilde{\mathbf{g}}^{(t)} \gets \mathbf{g}^{(t)} / \|\mathbf{g}^{(t)}\|_\infty$ \hfill {\small\color{gray}\texttt{// normalized guidance}}
		\NoNumber{\small{\color{gray}\texttt{// reverse step with guidance (Eq.~\ref{eq:rcd_update})}}}
		\STATE $\bs{\tau}^{(t-1)} \gets \bs{\mu}_\theta(\bs{\tau}^{(t)}, t) + \sigma_t\, \mathbf{z} - w\, \sigma_t^2\, \tilde{\mathbf{g}}^{(t)}, \quad \mathbf{z} \sim \mathcal{N}(\mathbf{0}, \mathbf{I})$
		\STATE $\bs{\tau}^{(t-1)}_{\,\bs{s}_0} \gets \bs{s}$ \hfill {\small\color{gray}\texttt{// constrain start state}}
		\ENDFOR
		\STATE \textbf{return} $\bs{\tau}^{(0)}$
	\end{algorithmic}
\end{algorithm}

%=============================================================================
% IMPLEMENTATION DETAILS
%=============================================================================
\section{Implementation Details}
\label{app:impl_detail}

\subsection{Environments and Datasets}
\label{app:envs}

We evaluate on OGBench~\cite{park2025ogbench}, which provides diverse long-horizon goal-conditioned tasks with \texttt{stitch} and \texttt{play} datasets.
\textbf{Locomotion.} \texttt{pointmaze}, \texttt{antmaze}, and \texttt{humanoid maze} require an agent to navigate from a start position to a goal in mazes of increasing size (\texttt{Medium}, \texttt{Large}, \texttt{Giant}). The \texttt{stitch} datasets consist of short, disconnected trajectory segments that do not individually span start-to-goal pairs, requiring methods to compose multiple segments at inference time.
\textbf{Object Manipulation.} \texttt{cube} tasks require a 6-DoF UR5e robot arm to pick and place one to four cubes into target configurations. The \texttt{play} datasets contain diverse, unstructured manipulation trajectories.
\texttt{antsoccer} requires an ant agent to dribble a ball to a goal location, combining locomotion and object interaction.
\textbf{Visual.} \texttt{visual antmaze} provides $64 \times 64$ RGB image observations. We pretrain a VAE to encode observations into a 16-dimensional latent space, and all planning and inverse dynamics operate in this latent space.

\cref{tab:env_datasets} lists all evaluation environments with their corresponding OGBench dataset names and the maximum environment steps per episode.

\begin{table}[ht]
	\centering
	\caption{\textbf{Evaluation environments and datasets in OGBench.} We follow the evaluation setup of CompDiffuser~\cite{luo2025generative}.}
	\label{tab:env_datasets}
	\vspace{1mm}
	\setlength{\tabcolsep}{8pt}
	\begin{tabular}{llllc}
		\toprule
		\textbf{Environment}                   & \textbf{Type}                       & \textbf{Size}                 & \textbf{Dataset Name}        & \textbf{Env Steps} \\
		\midrule
		\multirow[c]{3}{*}{\texttt{pointmaze}} & \multirow[c]{3}{*}{\texttt{stitch}}
		                                       & \texttt{Medium}                     & pointmaze-medium-stitch-v0    & 1000                                              \\
		                                       &                                     & \texttt{Large}                & pointmaze-large-stitch-v0    & 1000               \\
		                                       &                                     & \texttt{Giant}                & pointmaze-giant-stitch-v0    & 1000               \\
		\midrule
		\multirow[c]{3}{*}{\texttt{antmaze}}   & \multirow[c]{3}{*}{\texttt{stitch}}
		                                       & \texttt{Medium}                     & antmaze-medium-stitch-v0      & 1000                                              \\
		                                       &                                     & \texttt{Large}                & antmaze-large-stitch-v0      & 2000               \\
		                                       &                                     & \texttt{Giant}                & antmaze-giant-stitch-v0      & 2000               \\
		\midrule
		\multirow[c]{3}{*}{\makecell[l]{\texttt{humanoid}                                                                                                                \\\texttt{maze}}} & \multirow[c]{3}{*}{\texttt{stitch}}
		                                       & \texttt{Medium}                     & humanoidmaze-medium-stitch-v0 & 5000                                              \\
		                                       &                                     & \texttt{Large}                & humanoidmaze-large-stitch-v0 & 5000               \\
		                                       &                                     & \texttt{Giant}                & humanoidmaze-giant-stitch-v0 & 8000               \\
		\midrule
		\multirow[c]{2}{*}{\texttt{antsoccer}} & \multirow[c]{2}{*}{\texttt{stitch}}
		                                       & \texttt{Arena}                      & antsoccer-arena-stitch-v0     & 5000                                              \\
		                                       &                                     & \texttt{Medium}               & antsoccer-medium-stitch-v0   & 5000               \\
		\midrule
		\multirow[c]{4}{*}{\texttt{cube}}      & \multirow[c]{4}{*}{\texttt{play}}
		                                       & \texttt{Single}                     & cube-single-play-v0           & 2000                                              \\
		                                       &                                     & \texttt{Double}               & cube-double-play-v0          & 2000               \\
		                                       &                                     & \texttt{Triple}               & cube-triple-play-v0          & 2000               \\
		                                       &                                     & \texttt{Quadruple}            & cube-quadruple-play-v0       & 2000               \\
		\midrule
		\multirow[c]{2}{*}{\makecell[l]{\texttt{visual}                                                                                                                  \\\texttt{antmaze}}} & \multirow[c]{2}{*}{\texttt{stitch}}
		                                       & \texttt{Medium}                     & antmaze-medium-stitch-v0      & 2000                                              \\
		                                       &                                     & \texttt{Large}                & antmaze-large-stitch-v0      & 2000               \\
		\bottomrule
	\end{tabular}
\end{table}

\subsection{Diffusion Model Architecture and Training}
\label{app:arch}

Our implementation is built on top of the public CompDiffuser codebase\footnote{\url{https://github.com/devinluo27/comp_diffuser_release}}~\cite{luo2025generative}, and we use the same pretrained local diffusion model as CompDiffuser for all experiments.
For \texttt{pointmaze}, \texttt{antmaze}, and \texttt{humanoid maze}, whose planners operate in the 2D $x$-$y$ space, the model architecture follows a 1D temporal U-Net with residual blocks, group normalization, and sinusoidal timestep embeddings. For \texttt{antsoccer} and \texttt{cube}, we use a DiT-based architecture~\cite{luo2025generative} with environment-specific hidden dimensions and patch sizes.
Training uses the standard DDPM~\cite{ho2020denoising} objective with a linear noise schedule.
All models are trained until convergence using the Adam optimizer with learning rate $2 \times 10^{-4}$.

\subsection{Hyperparameters}
\label{app:rcd_hyperparams}

\cref{tab:impl_hyperparams} lists the diffusion model and RCD hyperparameters used across all environments.
We use a \emph{single fixed set of RCD hyperparameters} ($w{=}0.25$, $\lambda_{\mathrm{ov}}{=}0.5$, $s/T{=}0.4$) across every benchmark in this paper, including locomotion mazes of three different scales, high-dimensional object manipulation, and visual AntMaze. The hyperparameter ablations in \cref{app:add_ablation} confirm that RCD is robust to the choice of these hyperparameters.

\begin{table}[ht]
	\centering
	\caption{\textbf{Hyperparameters across environments.}}
	\label{tab:impl_hyperparams}
	\resizebox{\textwidth}{!}{%
		\begin{tabular}{@{}l cccccc@{}}
			\toprule
			\textbf{Parameter}                     & \textbf{PointMaze} & \textbf{AntMaze} & \makecell{\textbf{Humanoid}                       \\\textbf{Maze}} & \textbf{Cube} & \textbf{AntSoccer} & \makecell{\textbf{Visual}\\\textbf{AntMaze}} \\
			\midrule
			Architecture                           & U-Net              & U-Net            & U-Net                       & DiT  & DiT  & U-Net \\
			Segment length $H$                     & 160                & 160              & 336                         & 160  & 160  & 160   \\
			Overlap size                           & 64                 & 64               & 128                         & 56   & 56   & 64    \\
			Diffusion steps $T$                    & 512               & 512              & 1000                         & 512  & 512  & 512   \\
			\midrule
			Probe ratio $s/T$                      & 0.40               & 0.40             & 0.40                        & 0.40 & 0.40 & 0.40  \\
			Guidance weight $w$                    & 0.25               & 0.25             & 0.25                        & 0.25 & 0.25 & 0.25  \\
			Overlap weight $\lambda_{\mathrm{ov}}$ & 0.5                & 0.5              & 0.5                         & 0.5  & 0.5  & 0.5   \\
			\bottomrule
		\end{tabular}%
	}
\end{table}

\begin{table}[ht]
	\centering
    \caption{\textbf{Number of composed segments $M$ per environment.} For locomotion mazes and \texttt{antsoccer}, values follow the same setting as CompDiffuser~\cite{luo2025generative}.}
	\label{tab:env_n_segments}
	\vspace{1mm}
	\setlength{\tabcolsep}{8pt}
	\begin{tabular}{llllc}
		\toprule
		\textbf{Environment}                   & \textbf{Type}                       & \textbf{Size}                 & \textbf{Dataset Name}        & \textbf{\# Segments $M$} \\
		\midrule
		\multirow[c]{3}{*}{\texttt{pointmaze}} & \multirow[c]{3}{*}{\texttt{stitch}}
		                                       & \texttt{Medium}                     & pointmaze-medium-stitch-v0    & 3                                                       \\
		                                       &                                     & \texttt{Large}                & pointmaze-large-stitch-v0    & 6                        \\
		                                       &                                     & \texttt{Giant}                & pointmaze-giant-stitch-v0    & 8                        \\
		\midrule
		\multirow[c]{3}{*}{\texttt{antmaze}}   & \multirow[c]{3}{*}{\texttt{stitch}}
		                                       & \texttt{Medium}                     & antmaze-medium-stitch-v0      & 3                                                       \\
		                                       &                                     & \texttt{Large}                & antmaze-large-stitch-v0      & 6                        \\
		                                       &                                     & \texttt{Giant}                & antmaze-giant-stitch-v0      & 9                        \\
		\midrule
		\multirow[c]{3}{*}{\makecell[l]{\texttt{humanoid}                                                                                                                      \\\texttt{maze}}} & \multirow[c]{3}{*}{\texttt{stitch}}
		                                       & \texttt{Medium}                     & humanoidmaze-medium-stitch-v0 & 4                                                       \\
		                                       &                                     & \texttt{Large}                & humanoidmaze-large-stitch-v0 & 6                        \\
		                                       &                                     & \texttt{Giant}                & humanoidmaze-giant-stitch-v0 & 11                       \\
		\midrule
		\multirow[c]{2}{*}{\texttt{antsoccer}} & \multirow[c]{2}{*}{\texttt{stitch}}
		                                       & \texttt{Arena}                      & antsoccer-arena-stitch-v0     & 5                                                       \\
		                                       &                                     & \texttt{Medium}               & antsoccer-medium-stitch-v0   & 6                        \\
		\midrule
		\multirow[c]{4}{*}{\texttt{cube}}      & \multirow[c]{4}{*}{\texttt{play}}
		                                       & \texttt{Single}                     & cube-single-play-v0           & 2                                                       \\
		                                       &                                     & \texttt{Double}               & cube-double-play-v0          & 5                        \\
		                                       &                                     & \texttt{Triple}               & cube-triple-play-v0          & 10                       \\
		                                       &                                     & \texttt{Quadruple}            & cube-quadruple-play-v0       & 10                       \\
		\midrule
		\multirow[c]{2}{*}{\makecell[l]{\texttt{visual}                                                                                                                        \\\texttt{antmaze}}} & \multirow[c]{2}{*}{\texttt{stitch}}
		                                       & \texttt{Medium}                     & antmaze-medium-stitch-v0      & 2                                                       \\
		                                       &                                     & \texttt{Large}                & antmaze-large-stitch-v0      & 5                        \\
		\bottomrule
	\end{tabular}
\end{table}

\subsection{Evaluation Protocol}
\label{app:eval_protocol}

We follow the OGBench evaluation protocol~\cite{park2025ogbench}.
Each environment specifies 5 test-time start-goal pairs.
For each pair, we run 20 episodes and report the binary success rate (reaching within a threshold distance of the goal).
Results are averaged over 5 random seeds, and we report mean $\pm$ standard deviation.
For replanning experiments, we adopt the adaptive replanning strategy proposed by CompDiffuser~\cite{luo2025generative}: replanning is triggered when the agent deviates from its current subgoal beyond a threshold distance, and a receding scheme progressively reduces the number of composed segments based on how much of the current plan has been executed to encourage faster convergence toward the goal.

\subsection{Compute Resources}
\label{app:compute}

All experiments, including training of the local diffusion models and all evaluations are conducted on a single NVIDIA H100 GPU.

\subsection{Low-Level Controllers}
\label{app:low_level}

The diffusion planner generates state-space trajectories, and a separate low-level controller converts these into executable actions.
\textbf{Locomotion and AntSoccer.} Following CompDiffuser~\cite{luo2025generative}, we train an MLP-based inverse dynamics model that takes consecutive states $(s_t, s_{t+1})$ as input and predicts the action $a_t$. The same model architecture and training procedure from CompDiffuser is used.
\textbf{Cube Manipulation.} We use a DQL-based value-learning policy~\cite{wang2022diffusion} that takes the current state and a planned subgoal as input and outputs actions. The DQL policy is trained on the same offline dataset used for the diffusion planner.
\textbf{Visual AntMaze.} We train an inverse dynamics model in the VAE latent space that maps consecutive latent states $(z_t, z_{t+1})$ to actions, following a similar architecture to the state-based inverse dynamics but operating on 16-dimensional latent representations.

%=============================================================================
% ADDITIONAL RESULTS
%=============================================================================
\section{Additional Results}
\label{app:add_results}

%-----------------------------------------------------------------------------
\subsection{Additional Plan Quality Visualizations}
\label{app:plan_quality_viz}

\cref{fig:plan_quality_appendix} extends the plan quality analysis of \cref{sec:plan_quality} to \texttt{PointMaze-Giant-Stitch}, confirming that the improvement in plan feasibility observed on \texttt{AntMaze} generalizes across environments.
CompDiffuser produces valid rates as low as 10\% on the hardest task (Task 3), while CDGS shows improvement on some tasks but remains inconsistent (15--100\%).
RCD achieves 85--100\% valid rates across all 5 tasks, consistent with the \texttt{AntMaze-Giant-Stitch} case presented in \cref{sec:plan_quality}.

\begin{figure}[ht]
	\centering
	\includegraphics[width=0.95\linewidth]{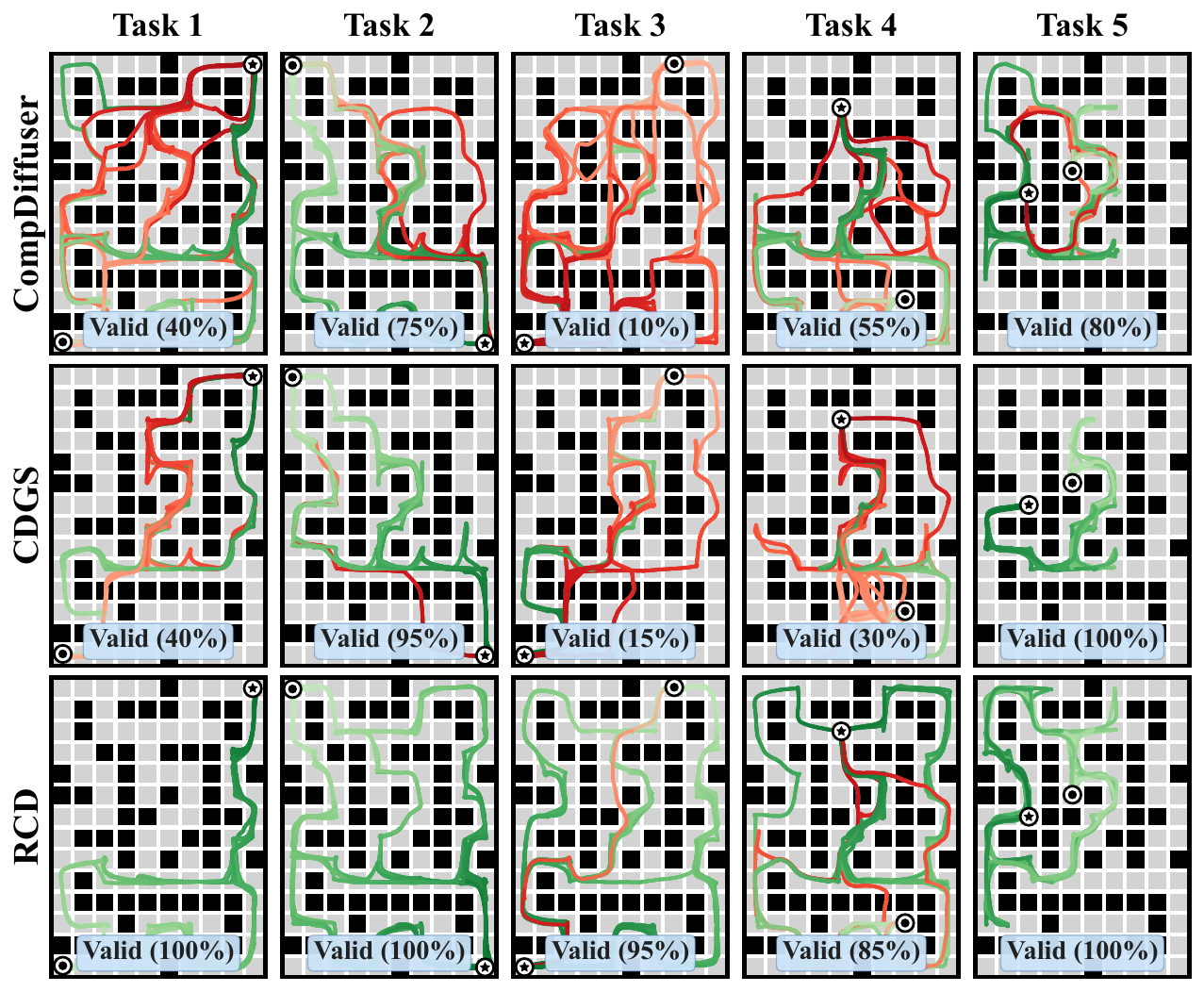}
	\caption{\textbf{Plan quality comparisons on \texttt{PointMaze-Giant-Stitch}.} Each column shows 20 sampled plans from CompDiffuser, CDGS, and RCD for 5 test-time tasks defined in OGBench. Plans that violate environment constraints (wall penetration) are shown in \textcolor[HTML]{f28482}{red}; feasible plans in \textcolor[HTML]{2ca25f}{green}.}
	\label{fig:plan_quality_appendix}
\end{figure}

%-----------------------------------------------------------------------------
\subsection{Additional Ablation Studies}
\label{app:add_ablation}

We vary the three RCD-specific hyperparameters one at a time on \texttt{PointMaze-Giant-Stitch} and \texttt{AntMaze-Giant-Stitch}, fixing the remaining ones at the defaults listed in \cref{tab:impl_hyperparams}. Reported numbers follow the same evaluation protocol as the main results (\cref{app:eval_protocol}). Success rates are averaged over $5$ random seeds, each evaluated on $5$ start-goal pairs with $20$ episodes per pair. The defaults ($w{=}0.25$, $\lambda_\mathrm{ov}{=}0.5$, $s/T{=}0.4$) are highlighted with \textbf{bold} values.

\paragraph{Ablation on guidance weight $w$.}
\cref{tab:abl_weight} varies the RCD guidance weight $w \in \{0, 0.25, 0.5\}$. Setting $w{=}0$ disables RCD guidance entirely and recovers a score close to plain compositional denoising, while both non-zero weights lead to large improvements on both environments. The default $w{=}0.25$ attains the best average success rate.

\begin{table}[H]
	\centering
	\caption{\textbf{Ablation on guidance weight $w$.}}
	\label{tab:abl_weight}
	\resizebox{\textwidth}{!}{%
		\begin{tabular}{lll ll lll}
			\toprule
			\multirow{2}{*}[-0.7ex]{\textbf{Env}} & \multirow{2}{*}[-0.7ex]{\textbf{Type}} & \multirow{2}{*}[-0.7ex]{\textbf{Size}}
			                                      & \multicolumn{2}{c}{\textbf{Baselines}} & \multicolumn{3}{c}{\textbf{RCD}}                                                                                                             \\
			\cmidrule(lr){4-5} \cmidrule(lr){6-8}
			                                      &                                        &                                        & \textbf{CD}     & \textbf{CDGS}  & \boldmath$w{=}0$ & \boldmath$w{=}0.25$ \textbf{(default)} & \boldmath$w{=}0.5$ \\
			\midrule
			\multirow[c]{1}{*}{\texttt{pointmaze}}
			                                      & \multirow[c]{1}{*}{\texttt{stitch}}    & \texttt{Giant}
			                                      & \valstd{69}{3}                         & \valstd{74}{3}                         & \valstd{59}{5}  & \valstd{\mathbf{100}}{0} & \valstd{100}{0}                        \\
			\multirow[c]{1}{*}{\texttt{antmaze}}
			                                      & \multirow[c]{1}{*}{\texttt{stitch}}    & \texttt{Giant}
			                                      & \valstd{67}{3}                         & \valstd{83}{3}                         & \valstd{56}{4}  & \valstd{\mathbf{89}}{2}  & \valstd{87}{3}                         \\
			\midrule
			\multicolumn{3}{c}{\textbf{Average}}  & $68.0$                                 & $78.5$                                 & $57.5$          & $\mathbf{94.5}$          & $93.5$                                                      \\
			\bottomrule
		\end{tabular}%
	}
\end{table}

\paragraph{Ablation on overlap consistency weight $\lambda_\mathrm{ov}$.}
\cref{tab:abl_overlap} varies the overlap consistency weight $\lambda_\mathrm{ov} \in \{0, 0.25, 0.5\}$. Even with $\lambda_\mathrm{ov}{=}0$ (reconstruction error only) the guided sampler already outperforms the compositional baseline, but adding the overlap consistency term consistently improves \texttt{AntMaze-Giant-Stitch} and yields the best overall success rate at the default $\lambda_\mathrm{ov}{=}0.5$.

\begin{table}[H]
	\centering
	\caption{\textbf{Ablation on overlap consistency weight $\lambda_\mathrm{ov}$.}}
	\label{tab:abl_overlap}
	\resizebox{\textwidth}{!}{%
		\begin{tabular}{lll ll lll}
			\toprule
			\multirow{2}{*}[-0.7ex]{\textbf{Env}} & \multirow{2}{*}[-0.7ex]{\textbf{Type}} & \multirow{2}{*}[-0.7ex]{\textbf{Size}}
			                                      & \multicolumn{2}{c}{\textbf{Baselines}} & \multicolumn{3}{c}{\textbf{RCD}}                                                                                                                                                                   \\
			\cmidrule(lr){4-5} \cmidrule(lr){6-8}
			                                      &                                        &                                        & \textbf{CD}                        & \textbf{CDGS}                         & \boldmath$\lambda_\mathrm{ov}{=}0$ & \boldmath$\lambda_\mathrm{ov}{=}0.25$ & \boldmath$\lambda_\mathrm{ov}{=}0.5$ \textbf{(default)} \\
			\midrule
			\multirow[c]{1}{*}{\texttt{pointmaze}}
			                                      & \multirow[c]{1}{*}{\texttt{stitch}}    & \texttt{Giant}
			                                      & \valstd{69}{3}                         & \valstd{74}{3}                         & \valstd{95}{2}                     & \valstd{92}{3}                        & \valstd{\mathbf{100}}{0}                                \\
			\multirow[c]{1}{*}{\texttt{antmaze}}
			                                      & \multirow[c]{1}{*}{\texttt{stitch}}    & \texttt{Giant}
			                                      & \valstd{67}{3}                         & \valstd{83}{3}                         & \valstd{79}{3}                     & \valstd{85}{3}                        & \valstd{\mathbf{89}}{2}                                 \\
			\midrule
			\multicolumn{3}{c}{\textbf{Average}}  & $68.0$                                 & $78.5$                                 & $89.0$                             & $88.5$                                & $\mathbf{94.5}$                                                                                  \\
			\bottomrule
		\end{tabular}%
	}
\end{table}

\paragraph{Ablation on probe level $s/T$.}
\cref{tab:abl_probe} varies the probe level $s/T \in \{0.1, 0.2, 0.3, 0.4, 0.5, 0.6\}$. Performance is stable across a wide range of probe levels, and we adopt $s/T{=}0.4$ as the default across all environments.

\begin{table}[H]
	\centering
	\caption{\textbf{Ablation on probe level $s/T$.}}
	\label{tab:abl_probe}
	\resizebox{\textwidth}{!}{%
		\begin{tabular}{lll ll llllll}
			\toprule
			\multirow{2}{*}[-0.7ex]{\textbf{Env}} & \multirow{2}{*}[-0.7ex]{\textbf{Type}} & \multirow{2}{*}[-0.7ex]{\textbf{Size}}
			                                      & \multicolumn{2}{c}{\textbf{Baselines}} & \multicolumn{6}{c}{\textbf{RCD}}                                                                                                                                                                                      \\
			\cmidrule(lr){4-5} \cmidrule(lr){6-11}
			                                      &                                        &                                        & \textbf{CD}    & \textbf{CDGS}  & \boldmath$s/T{=}0.1$ & \boldmath$0.2$ & \boldmath$0.3$ & \boldmath$0.4$ \textbf{(default)} & \boldmath$0.5$ & \boldmath$0.6$ \\
			\midrule
			\multirow[c]{1}{*}{\texttt{pointmaze}}
			                                      & \multirow[c]{1}{*}{\texttt{stitch}}    & \texttt{Giant}
			                                      & \valstd{69}{3}                         & \valstd{74}{3}                         & \valstd{98}{2} & \valstd{98}{2} & \valstd{98}{2}       & \valstd{\mathbf{100}}{0} & \valstd{98}{2} & \valstd{98}{2}                                                     \\
			\multirow[c]{1}{*}{\texttt{antmaze}}
			                                      & \multirow[c]{1}{*}{\texttt{stitch}}    & \texttt{Giant}
			                                      & \valstd{67}{3}                         & \valstd{83}{3}                         & \valstd{81}{4} & \valstd{86}{3} & \valstd{84}{3}       & \valstd{\mathbf{89}}{2}  & \valstd{\mathbf{89}}{2} & \valstd{83}{3}                                                     \\
			\midrule
			\multicolumn{3}{c}{\textbf{Average}}  & $68.0$                                 & $78.5$                                 & $89.5$         & $92.0$         & $91.0$               & $\mathbf{94.5}$          & $93.5$         & $90.5$                                                                              \\
			\bottomrule
		\end{tabular}%
	}
\end{table}

\subsection{Reconstruction Error and Plan Feasibility}
\label{app:proxy_feasibility}

We further evaluate whether the RCD guidance objective reflects the physical feasibility of composed plans. We sample 100 plans from unguided CD in each of \texttt{AntMaze-Giant-Stitch} and \texttt{PointMaze-Giant-Stitch}, and compute the self-reconstruction error $\mathcal{E}_{\mathrm{recon}}$, overlap consistency $\mathcal{E}_{\mathrm{ov}}$, and combined objective $\mathcal{E}_{\mathrm{RCD}}$ for each plan. Following \cref{sec:plan_quality}, a plan is valid if it does not pass through maze walls.

\begin{table}[ht]
\centering
\caption{\textbf{RCD guidance components and plan feasibility.} AUROC for detecting wall collisions on 100 unguided CD plans per environment. Colliding plans form the positive class.}
\label{tab:proxy_collision}
\begin{tabular}{lccc}
\toprule
\textbf{Environment} & $\mathcal{E}_{\mathrm{recon}}$ & $\mathcal{E}_{\mathrm{ov}}$ & $\mathcal{E}_{\mathrm{RCD}}$ \\
\midrule
\texttt{antmaze-Giant} & 0.83 & 0.92 & 0.92 \\
\texttt{pointmaze-Giant} & 0.76 & 0.88 & 0.85 \\
\bottomrule
\end{tabular}
\end{table}

\begin{table}[ht]
\centering
\caption{\textbf{Valid plan rates by RCD objective value.} Plans are sorted by increasing $\mathcal{E}_{\mathrm{RCD}}$ and divided into five groups of 20 per environment. Rates are reported as percentages.}
\label{tab:proxy_quintiles}
\begin{tabular}{lccccc}
\toprule
\textbf{Environment} & \textbf{Lowest 20\%} & \textbf{20--40\%} & \textbf{40--60\%} & \textbf{60--80\%} & \textbf{Highest 20\%} \\
\midrule
\texttt{antmaze-Giant} & 95 & 70 & 25 & 10 & 5 \\
\texttt{pointmaze-Giant} & 95 & 75 & 50 & 25 & 15 \\
\bottomrule
\end{tabular}
\end{table}

\cref{tab:proxy_collision} reports AUROC for detecting wall collisions. The combined objective achieves 0.92 on \texttt{AntMaze-Giant-Stitch} and 0.85 on \texttt{PointMaze-Giant-Stitch}, indicating that infeasible plans tend to receive larger objective values. \cref{tab:proxy_quintiles} further shows that the valid plan rate decreases monotonically across groups of increasing $\mathcal{E}_{\mathrm{RCD}}$. These results support the use of the RCD objective as a signal of plan feasibility. The overlap consistency term achieves higher collision AUROC than the self-reconstruction error in both environments. When used for guidance, however, the self-reconstruction error alone gives higher success rates. In \cref{tab:ablation_components}, using only $\mathcal{E}_{\mathrm{recon}}$ achieves 79\% on \texttt{AntMaze-Giant-Stitch} and 95\% on \texttt{PointMaze-Giant-Stitch}, compared with 56\% and 59\% using only $\mathcal{E}_{\mathrm{ov}}$. Overlap consistency targets agreement at segment boundaries, which can hold even for an infeasible plan, while the self-reconstruction error guides the entire trajectory toward high-density regions. These results are consistent with the complementary roles of the two guidance components.

\subsection{RCD with Unconditional Score Composition}
\label{app:unconditional_composition}

We further evaluate RCD with unconditional score composition in the style of GSC~\cite{mishra2023generative} and DiffCollage~\cite{zhang2023diffcollage}. We compose local denoisers without conditioning on noisy neighboring segments and use the same RCD hyperparameters as in the OGBench experiments (\cref{app:rcd_hyperparams}): $w=0.25$, $\lambda_{\mathrm{ov}}=0.5$, and $s/T=0.4$. Success rates are averaged over 5 random seeds, with 100 evaluation episodes per seed.

\begin{table}[ht]
\centering
\caption{\textbf{RCD with unconditional score composition.} Success rates (mean $\pm$ std) with and without RCD guidance.}
\label{tab:unconditional_rcd}
\begin{tabular}{llcc}
\toprule
\textbf{Environment} & \textbf{Size} & \textbf{Without RCD} & \textbf{With RCD} \\
\midrule
\multirow{3}{*}{\texttt{pointmaze}} & \texttt{Medium} & \valstd{100}{0} & \valstd{100}{0} \\
& \texttt{Large} & \valstd{100}{0} & \valstd{100}{0} \\
& \texttt{Giant} & \valstd{28}{3} & \valstd{\mathbf{51}}{3} \\
\midrule
\multirow{3}{*}{\texttt{antmaze}} & \texttt{Medium} & \valstd{95}{2} & \valstd{94}{2} \\
& \texttt{Large} & \valstd{70}{2} & \valstd{\mathbf{90}}{2} \\
& \texttt{Giant} & \valstd{22}{1} & \valstd{\mathbf{59}}{2} \\
\bottomrule
\end{tabular}
\end{table}

\cref{tab:unconditional_rcd} reports success rates with and without RCD guidance. RCD improves success rates by 23 percentage points on \texttt{PointMaze-Giant-Stitch}, 20 on \texttt{AntMaze-Large-Stitch}, and 37 on \texttt{AntMaze-Giant-Stitch}. The remaining tasks show little change, with a decrease of one percentage point on \texttt{AntMaze-Medium-Stitch}. These results demonstrate that RCD can improve compositional planning without bidirectional neighbor conditioning.

\subsection{Composition with Out-of-Distribution Constraints}
\label{app:ood_composition}

The OGBench \texttt{stitch} tasks require composing familiar local segments into globally novel trajectories. We further evaluate RCD under a local distribution shift using the bimodal toy domain in \cref{fig:toy_insight}, where the training modes are centered at $-1$ and $+1$.

We clamp one interior variable to 0.0, 1.5, or 2.0 throughout sampling and generate 200 trajectories at horizon 12 for each condition. A sample is valid if all free interior variables lie within three training standard deviations of the same mode. We exclude the clamped variable from this criterion, so the valid rate measures the consistency of the remaining variables. We use 100 denoising steps, guidance weight $w=5$, overlap consistency weight $\lambda_{\mathrm{ov}}=1$, probe level $s/T=0.35$, and four Monte Carlo noise samples per guidance step. These settings are fixed across constraints.

\begin{table}[ht]
\centering
\caption{\textbf{Toy composition with a clamped interior variable.} Valid rates are evaluated on the free interior variables. Reconstruction error is measured on CD samples, averaged over all coordinates and eight probe noise draws.}
\label{tab:ood_clamp}
\begin{tabular}{lccc}
\toprule
\textbf{Clamped value} & \textbf{CD valid rate} & \textbf{RCD valid rate} & \textbf{Reconstruction error (CD)} \\
\midrule
None & 0.27 & 0.92 & 0.32 \\
0.0 (between modes) & 0.27 & 0.92 & 0.34 \\
1.5 & 0.25 & 0.84 & 0.36 \\
2.0 & 0.23 & 0.78 & 0.42 \\
\bottomrule
\end{tabular}
\end{table}

\cref{tab:ood_clamp} shows that RCD achieves higher valid rates than CD under every constraint. At a clamped value of 2.0, RCD retains a valid rate of 0.78, compared with 0.23 for CD. The mean reconstruction error on CD samples increases from 0.32 without an interior constraint to 0.42 at a clamped value of 2.0. These results show that RCD can improve the consistency of the free variables even when one interior variable is constrained away from the training modes.

\subsection{Failure Analysis}
\label{app:failure_analysis}

We analyze failure cases in \texttt{Visual-AntMaze}, \texttt{Cube}, and \texttt{AntSoccer}, focusing on low-level execution and candidate selection in the visual setting.

\paragraph{Low-level execution.}
To evaluate execution independently of generated plan quality, we provide the low-level controllers (\cref{app:low_level}) with ground-truth trajectory segments. We measure the fraction of segments for which the agent reaches within one spatial unit of the demonstrated endpoint. For state-based and visual \texttt{AntMaze}, the reference paths are approximately 15 units long. For \texttt{AntSoccer}, we evaluate the ball position, using dribbling segments with demonstrated ball displacements of 2--5 units.

\begin{table}[ht]
\centering
\caption{\textbf{Low-level execution of ground-truth trajectories.} Endpoint tracking rates with dataset trajectories as references.}
\label{tab:controller_tracking}
\begin{tabular}{llc}
\toprule
\textbf{Setting} & \textbf{Endpoint criterion} & \textbf{Rate (\%)} \\
\midrule
State-based \texttt{AntMaze} & Agent within 1 unit & 93 \\
Visual \texttt{AntMaze} (16D latent) & Agent within 1 unit & 30 \\
\texttt{AntSoccer} & Ball within 1 unit & 33 \\
\bottomrule
\end{tabular}
\end{table}

\cref{tab:controller_tracking} reports the endpoint tracking rates. The state-based \texttt{AntMaze} controller achieves 93\%, while the visual controller achieves 30\% on paths of comparable length. The visual setting uses inverse dynamics in a 16-dimensional VAE latent space, where execution can fail even when the reference trajectory is feasible. On \texttt{AntSoccer}, the controller reaches the demonstrated ball endpoint in only 33\% of cases, despite the short reference displacements. Qualitative inspection of \texttt{Cube} rollouts similarly shows that goal-reaching plans can fail during execution when the controller cannot reach, grasp, or place the target cube. These results indicate that low-level execution limits how improvements in plan quality translate into task success.

\paragraph{Stability of latent-space rankings.}
We evaluate the stability of candidate rankings under the random noise used to compute the density proxy. For each of 5 batches of 20 candidates, we compute the scores twice with independent noise and measure the fraction of top-2 candidates retained. The retention rate is 40\% in the 16-dimensional VAE latent space, compared with 70\% in the state space. The lower retention rate indicates that candidate selection is more sensitive to evaluation noise in the visual setting. Since CDGS uses these scores to prune candidates, this sensitivity may contribute to its low success rates in \cref{tab:pixel}. RCD uses gradient guidance to refine trajectories throughout denoising. Its gains over CD in \cref{tab:pixel} suggest that the signal remains useful in the latent space, although its ability to reflect physical feasibility depends on the information preserved by the VAE.

\section{Baseline Result Sources}
\label{app:baseline_sources}

\subsection{Locomotion}
We obtain scores for GCBC, GCIVL, GCIQL, QRL, CRL, and HIQL from Table~2 in \cite{park2025ogbench}. Scores for GSC is taken from Tables~1 and 2 in \cite{luo2025generative}. CDGS and RCD are from our experiments using the same pretrained models as CompDiffuser (CD).

\subsection{Object Manipulation}
For \texttt{antsoccer}, scores for GCBC through HIQL are from Table~2 in \cite{park2025ogbench}, and GSC from Table~3 in \cite{luo2025generative}. For \texttt{cube}, GCBC through HIQL are from Table~2 in \cite{park2025ogbench}.

\begin{figure}[h!]
	\centering
	\setlength{\tabcolsep}{1.5pt}
	\begin{tabular}{cccc}
		\includegraphics[width=0.23\linewidth]{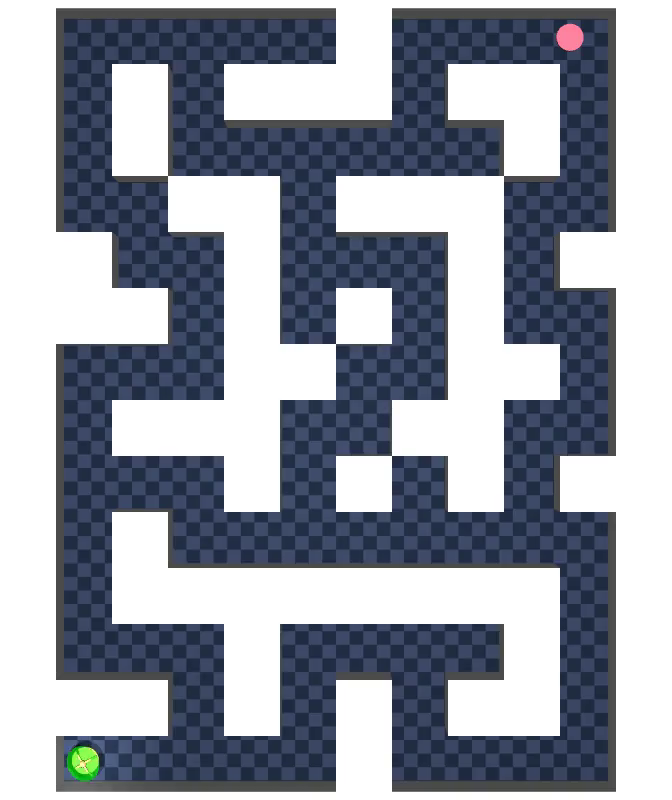} &
		\includegraphics[width=0.23\linewidth]{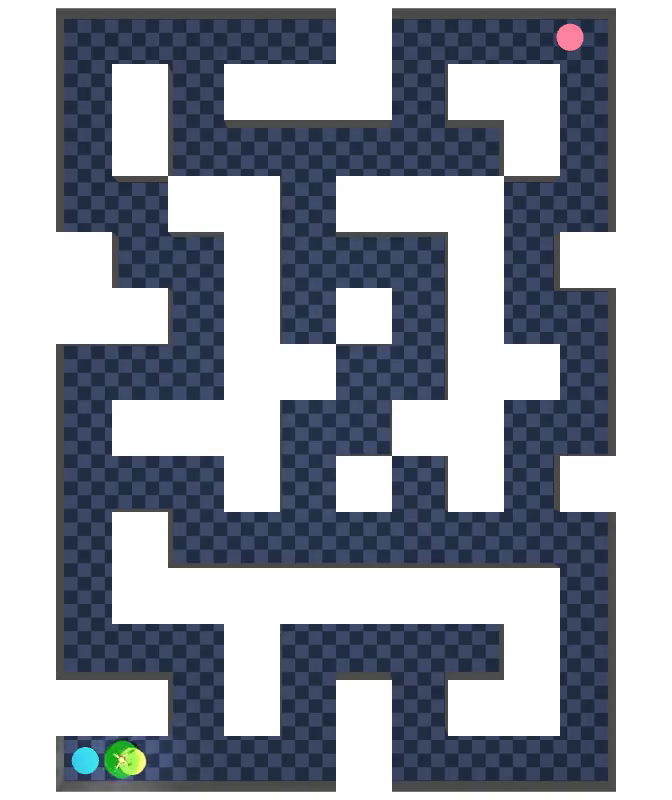} &
		\includegraphics[width=0.23\linewidth]{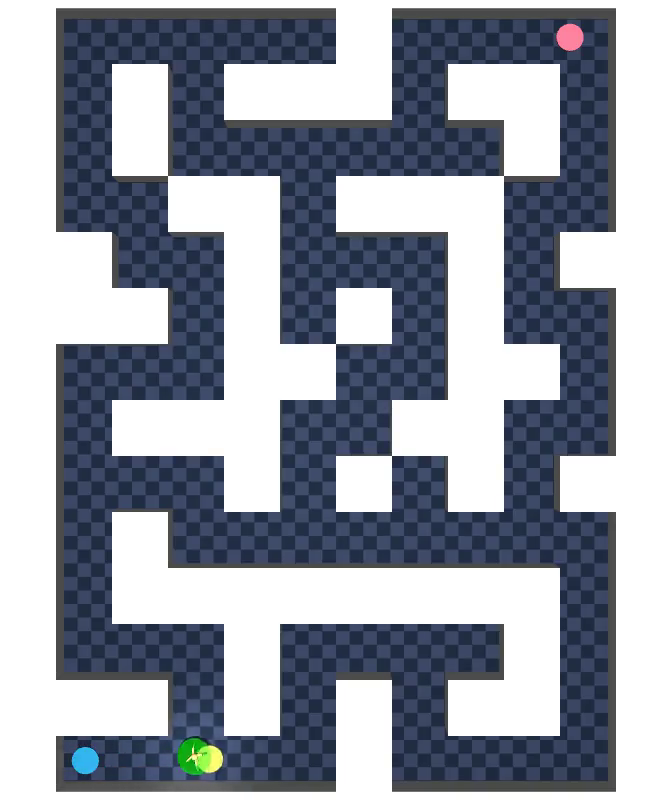} &
		\includegraphics[width=0.23\linewidth]{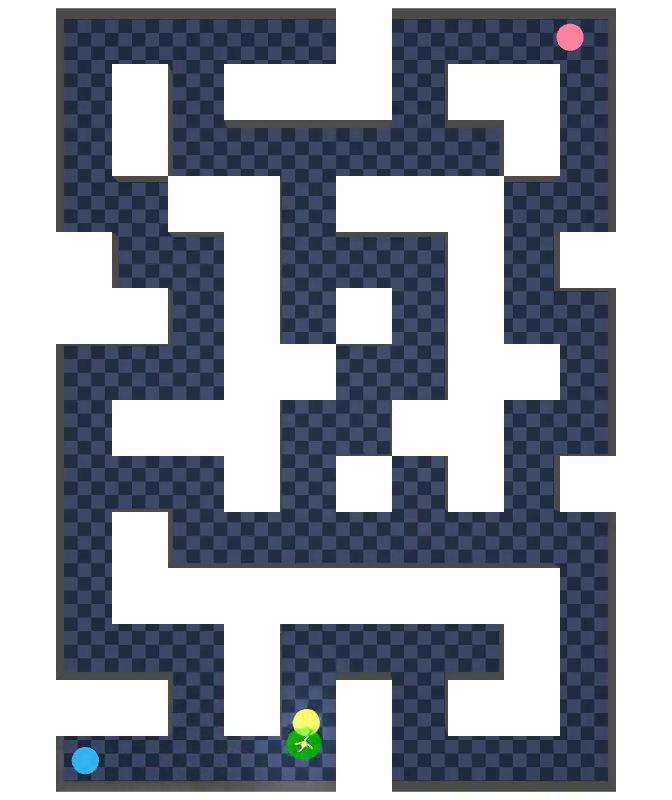} \\
		\addlinespace[1ex]
		\includegraphics[width=0.23\linewidth]{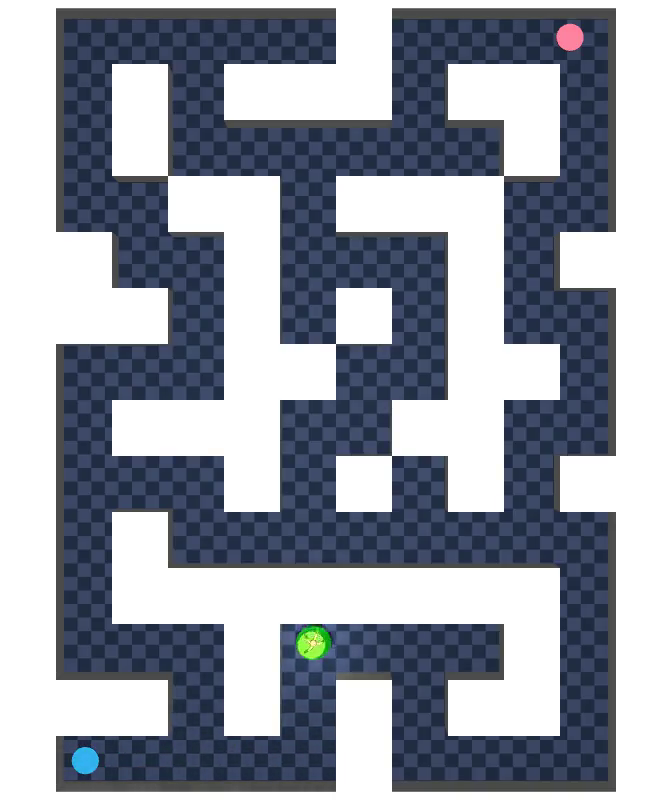} &
		\includegraphics[width=0.23\linewidth]{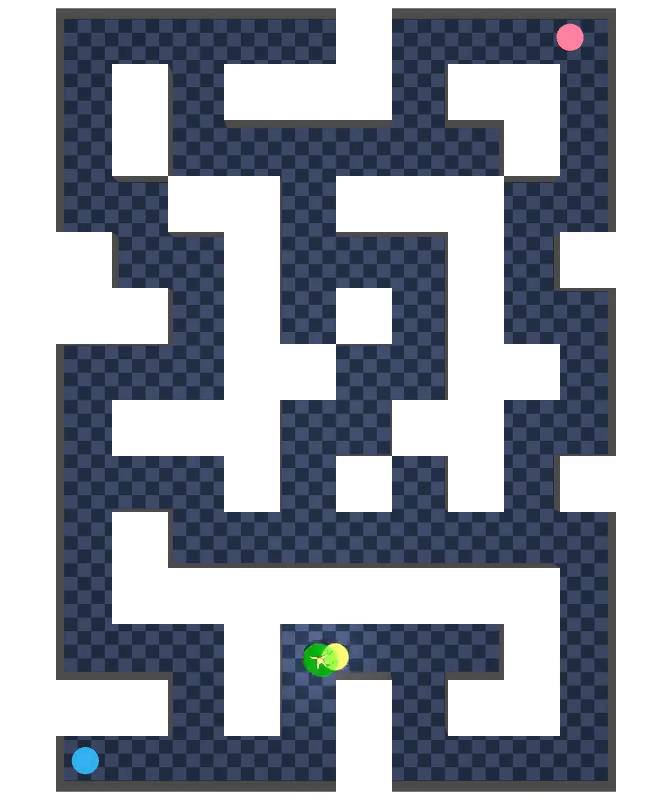} &
		\includegraphics[width=0.23\linewidth]{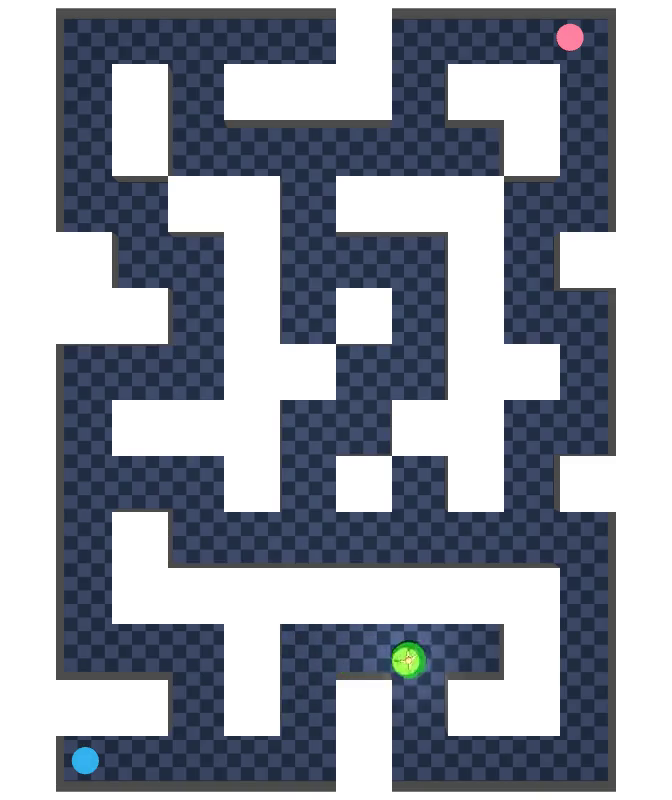} &
		\includegraphics[width=0.23\linewidth]{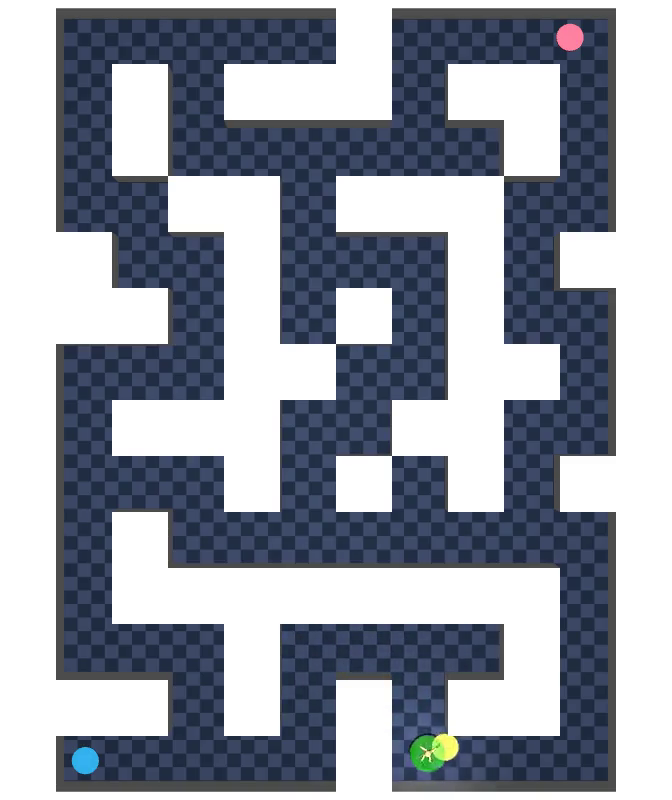} \\
		\addlinespace[1ex]
		\includegraphics[width=0.23\linewidth]{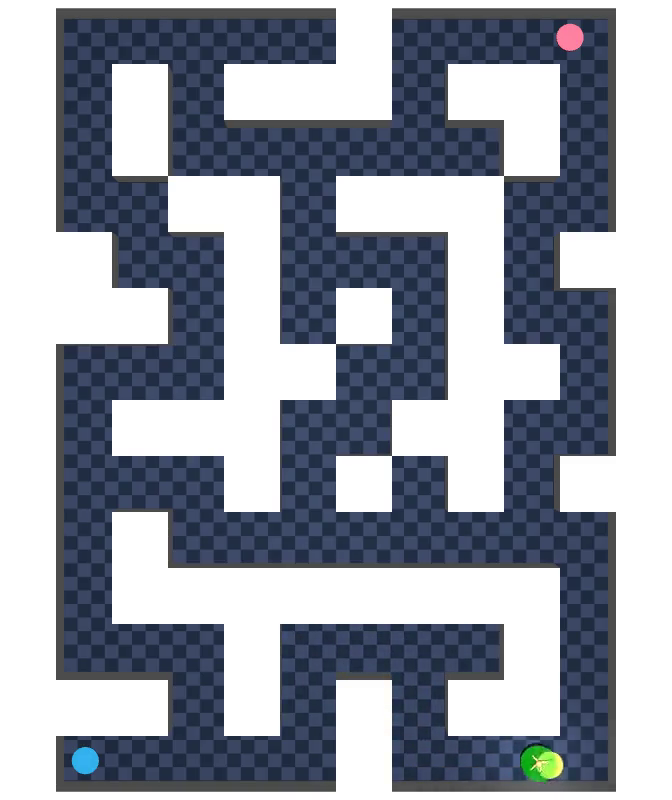} &
		\includegraphics[width=0.23\linewidth]{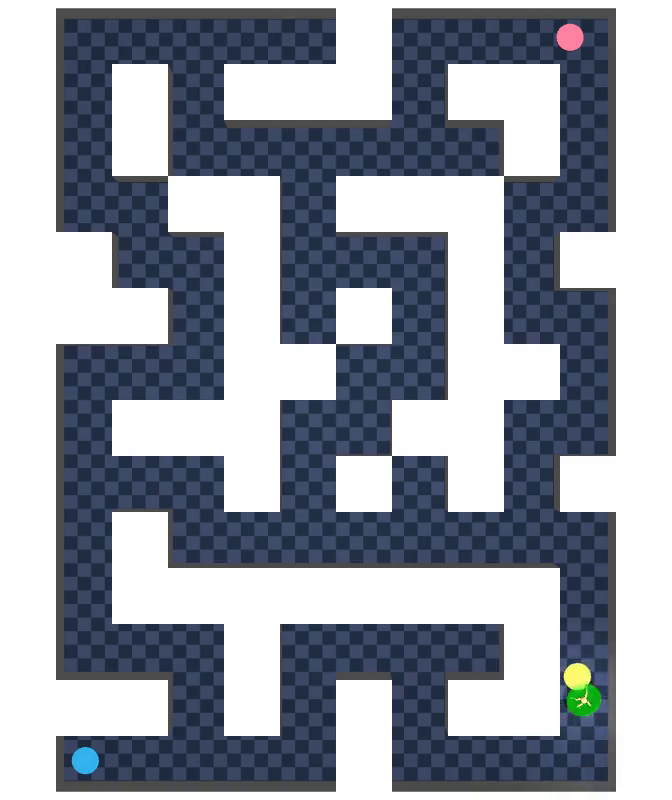} &
		\includegraphics[width=0.23\linewidth]{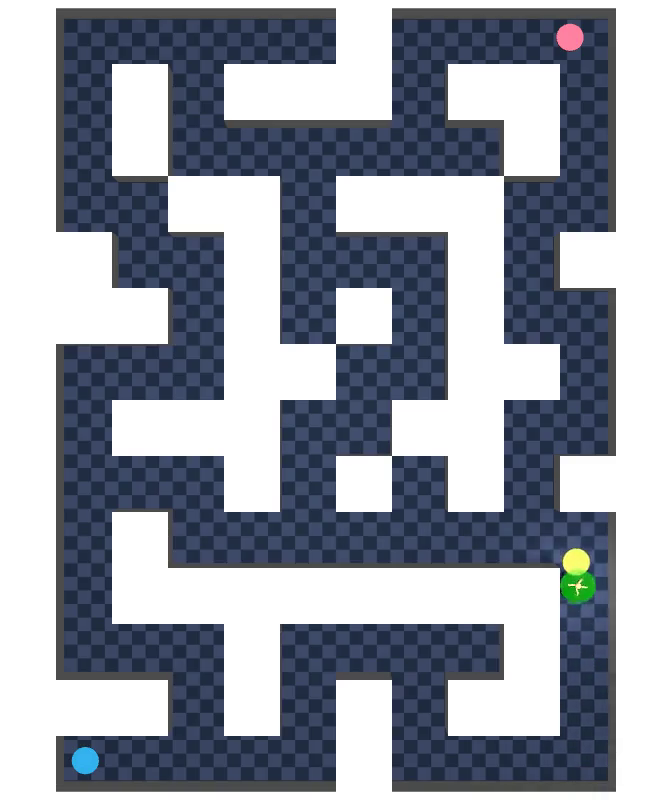} &
		\includegraphics[width=0.23\linewidth]{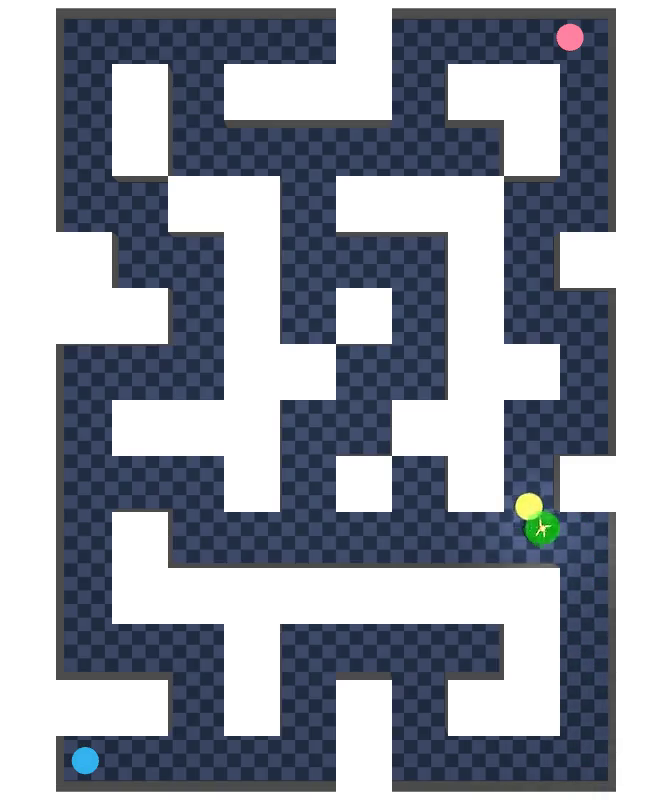} \\
		\addlinespace[1ex]
		\includegraphics[width=0.23\linewidth]{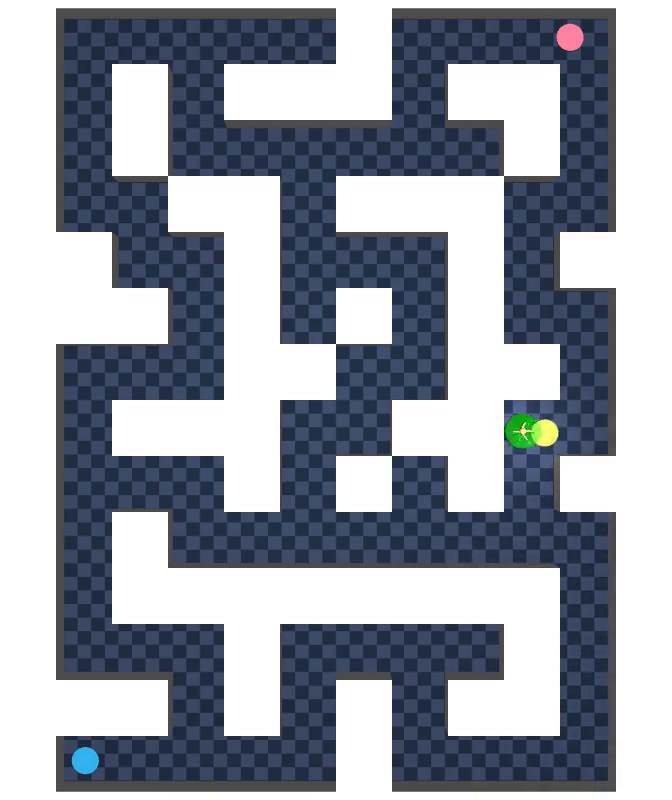} &
		\includegraphics[width=0.23\linewidth]{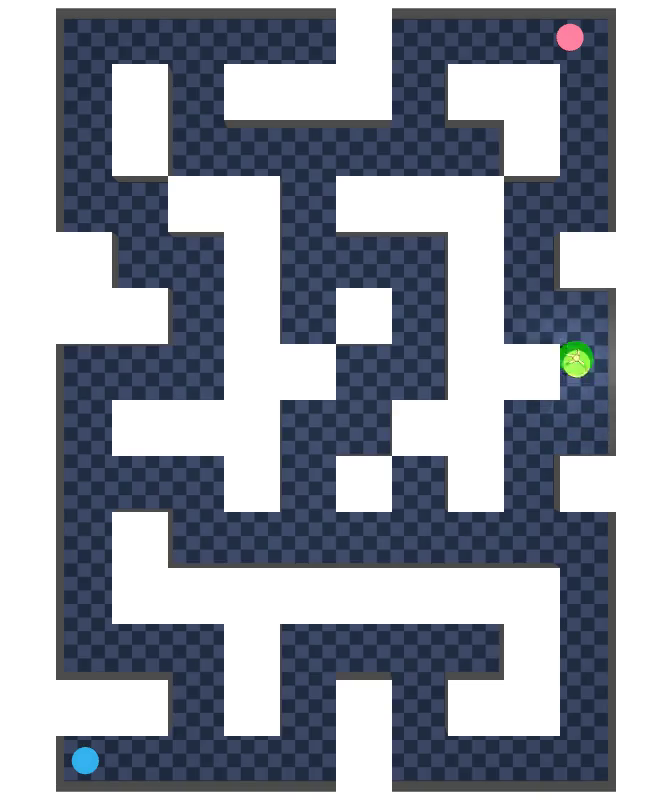} &
		\includegraphics[width=0.23\linewidth]{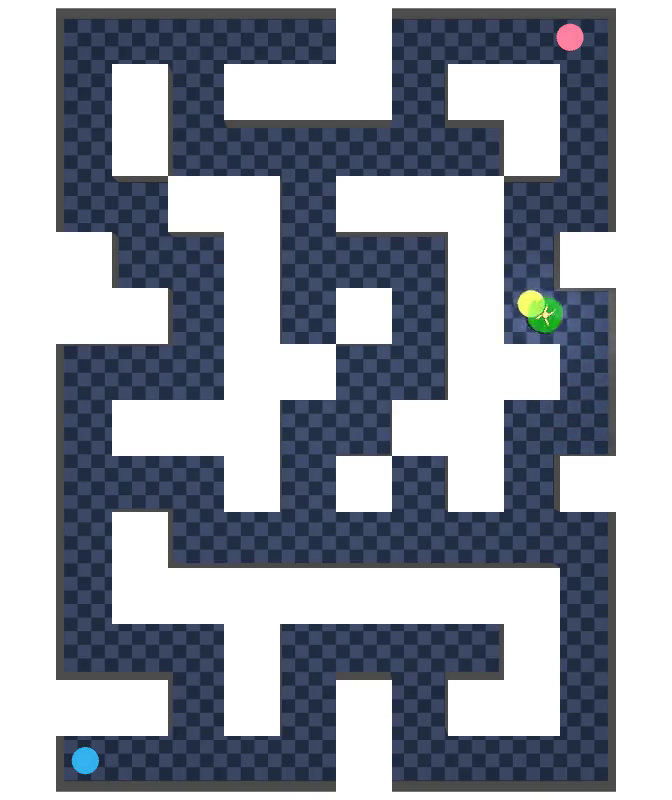} &
		\includegraphics[width=0.23\linewidth]{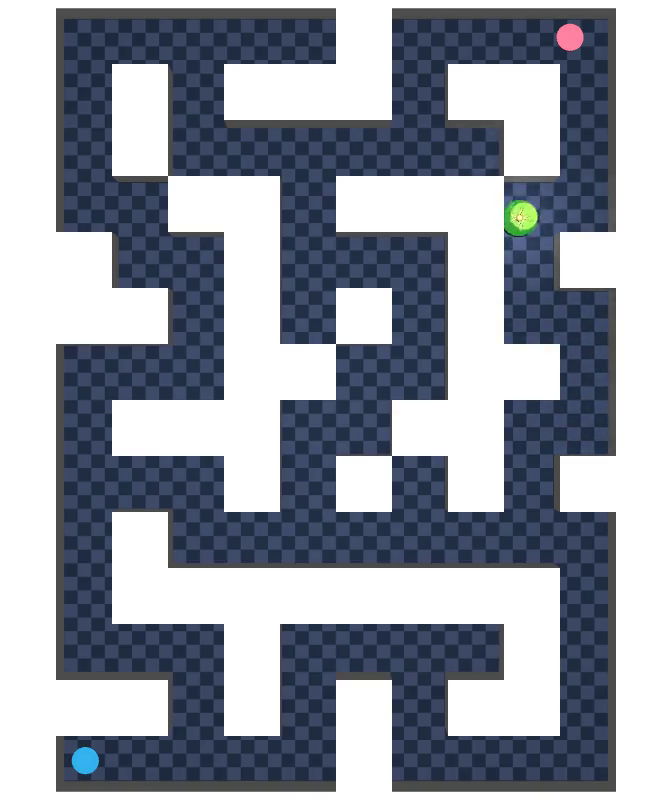} \\
		\addlinespace[1ex]
		\includegraphics[width=0.23\linewidth]{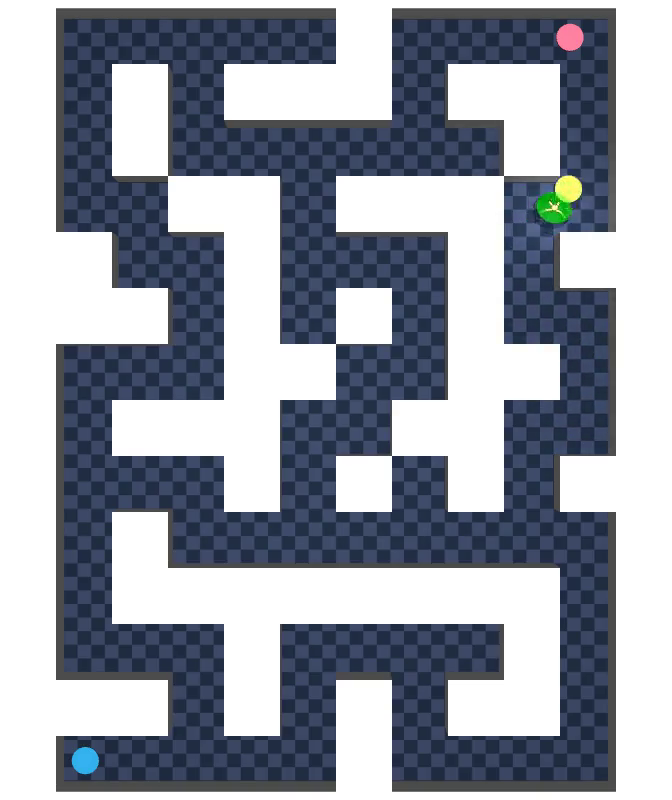} &
		\includegraphics[width=0.23\linewidth]{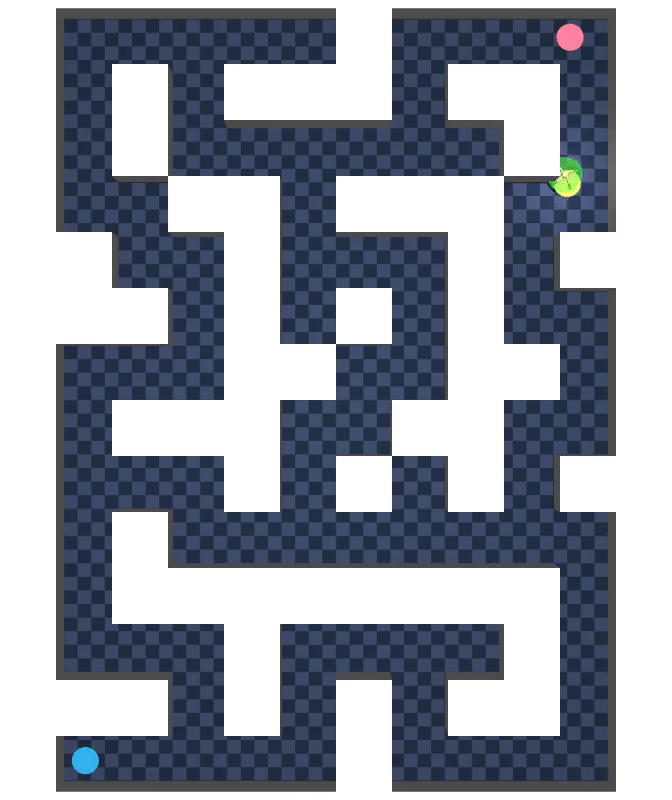} &
		\includegraphics[width=0.23\linewidth]{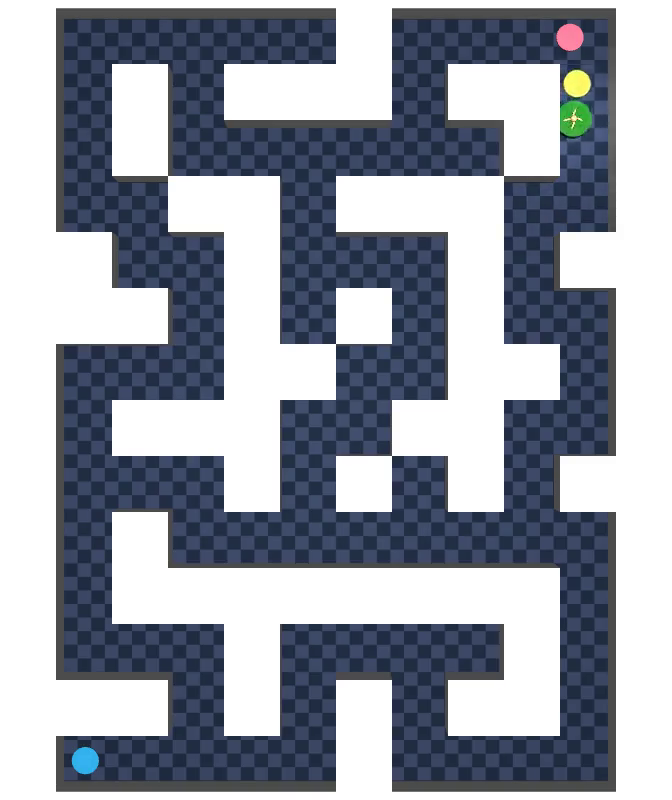} &
		\includegraphics[width=0.23\linewidth]{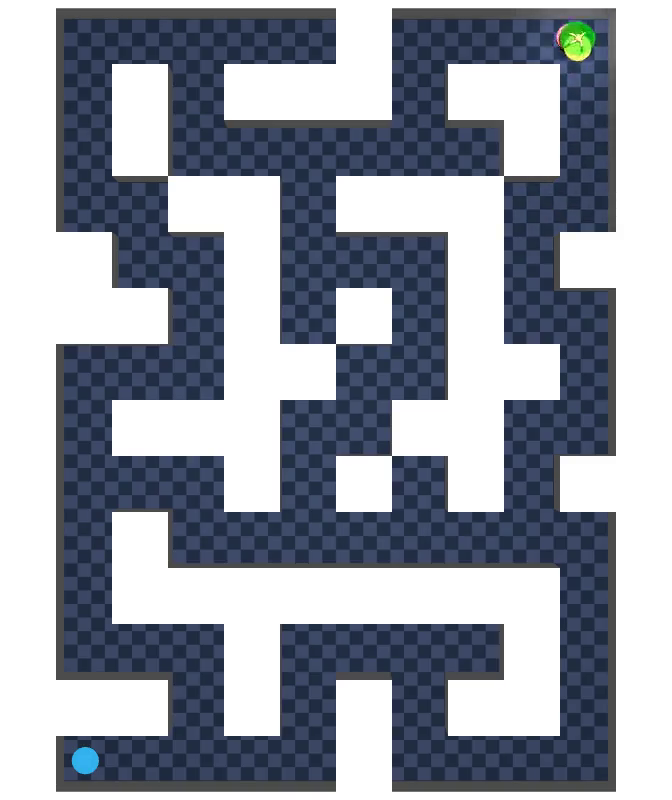} \\
	\end{tabular}
	\caption{\textbf{Visualization of RCD rollout execution on \texttt{AntMaze-Giant-Stitch}.} The ant agent navigates from the starting region to the pink goal.}
	\label{fig:rollout_antmaze_giant}
\end{figure}

\begin{figure}[h!]
	\centering
	\setlength{\tabcolsep}{1.5pt}
	\begin{tabular}{cccc}
		\includegraphics[width=0.23\linewidth]{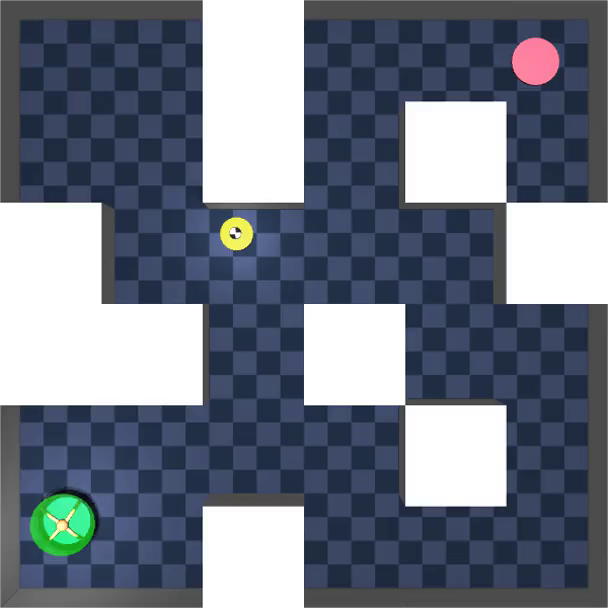} &
		\includegraphics[width=0.23\linewidth]{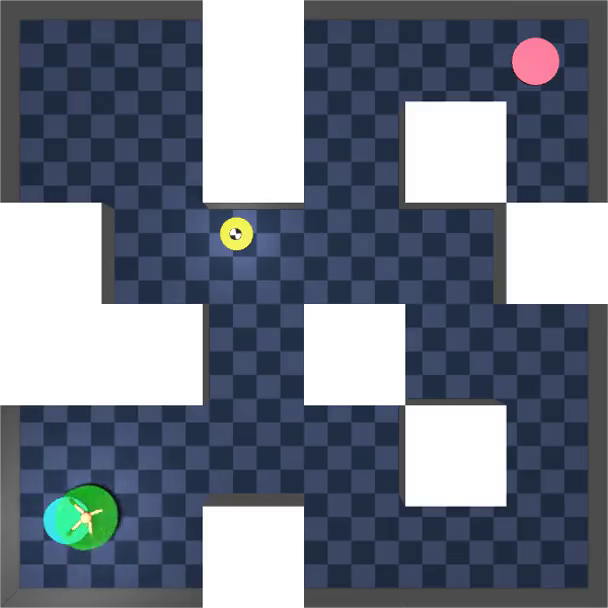} &
		\includegraphics[width=0.23\linewidth]{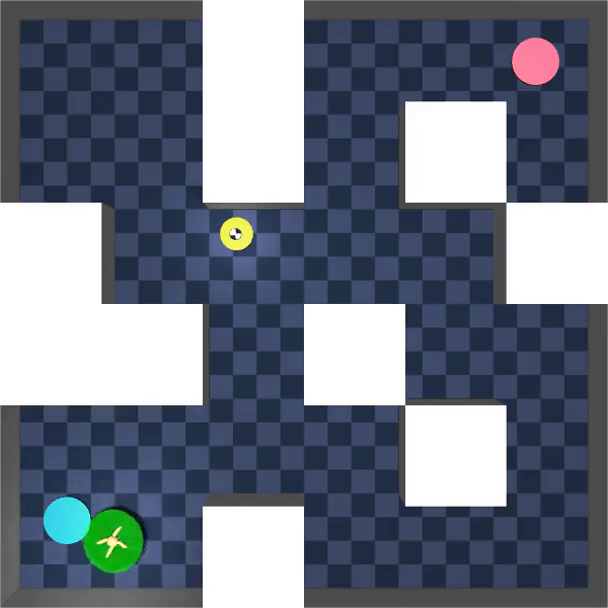} &
		\includegraphics[width=0.23\linewidth]{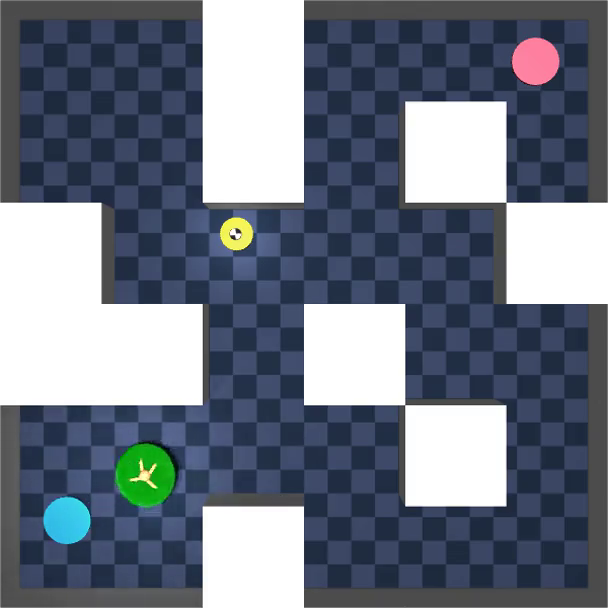} \\
		\addlinespace[1ex]
		\includegraphics[width=0.23\linewidth]{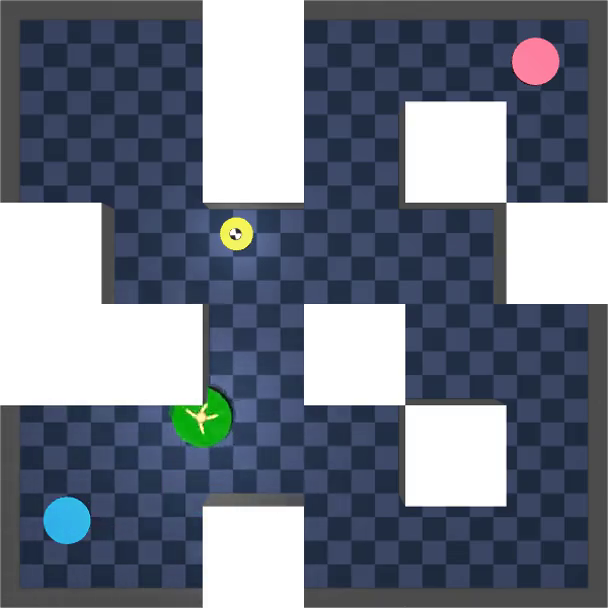} &
		\includegraphics[width=0.23\linewidth]{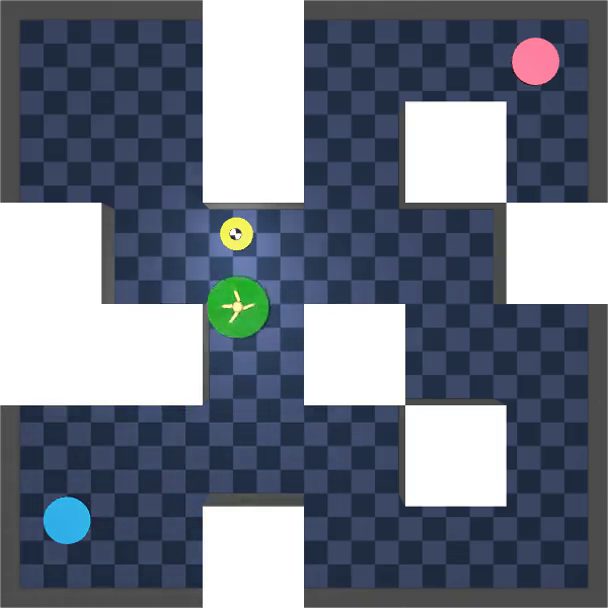} &
		\includegraphics[width=0.23\linewidth]{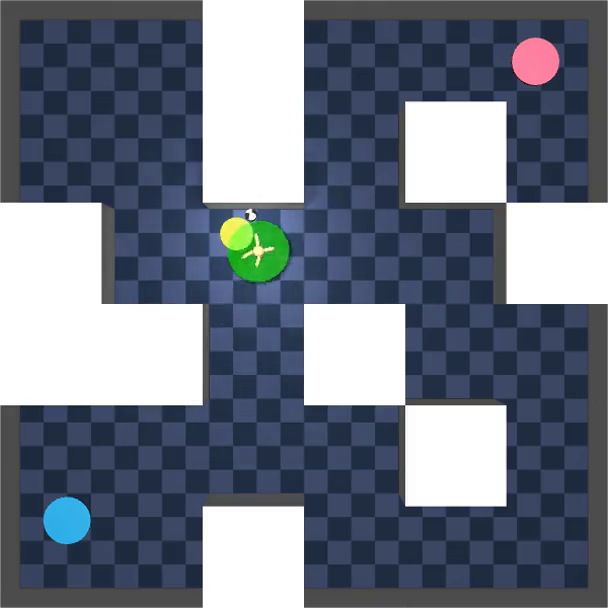} &
		\includegraphics[width=0.23\linewidth]{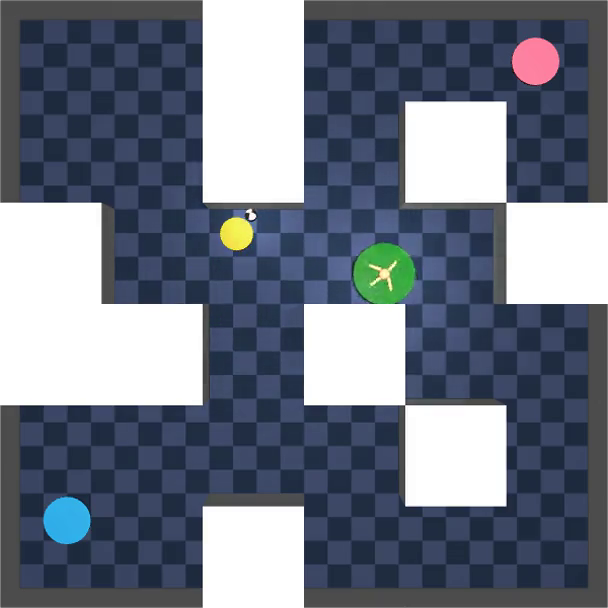} \\
		\addlinespace[1ex]
		\includegraphics[width=0.23\linewidth]{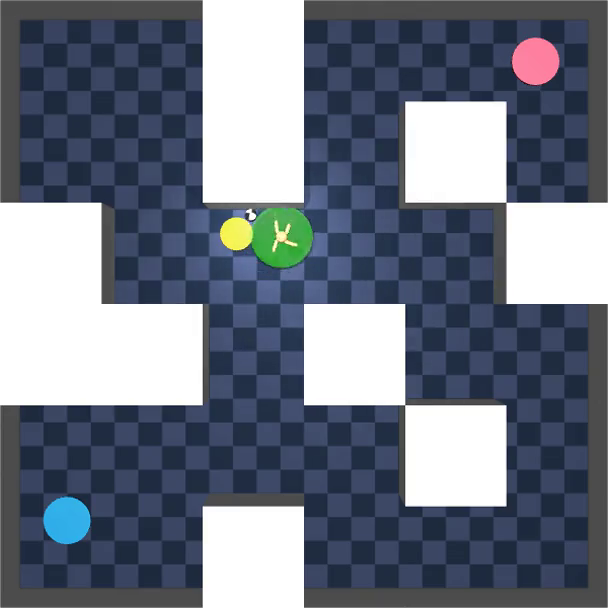} &
		\includegraphics[width=0.23\linewidth]{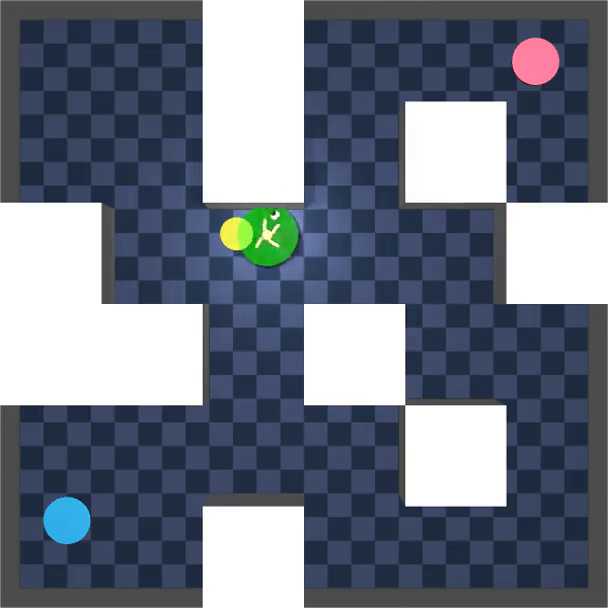} &
		\includegraphics[width=0.23\linewidth]{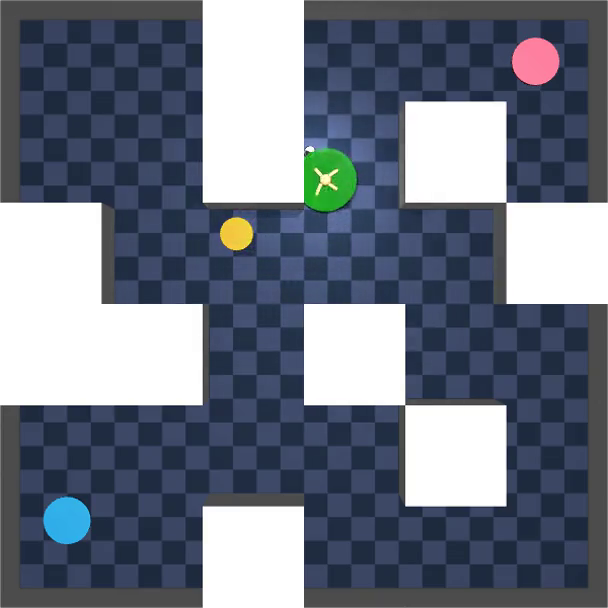} &
		\includegraphics[width=0.23\linewidth]{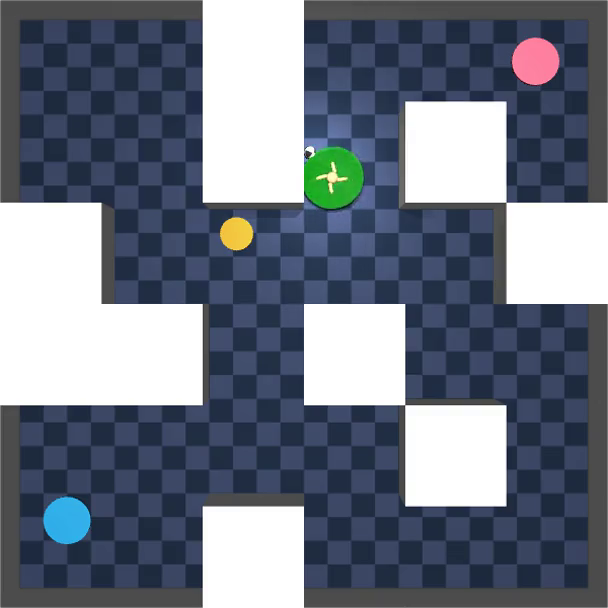} \\
		\addlinespace[1ex]
		\includegraphics[width=0.23\linewidth]{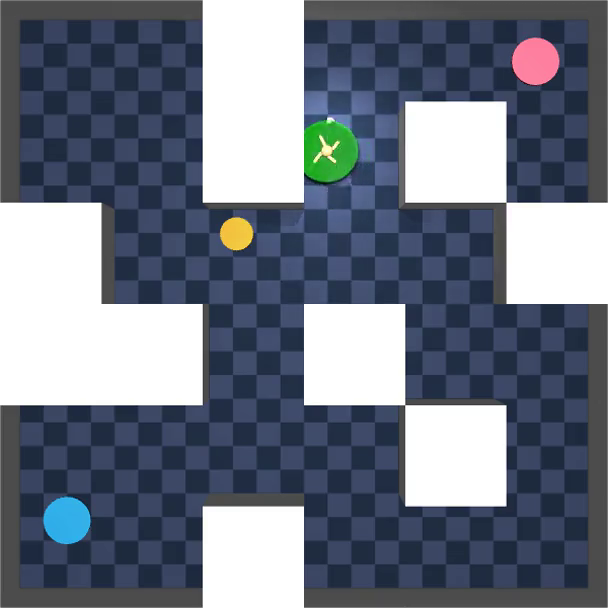} &
		\includegraphics[width=0.23\linewidth]{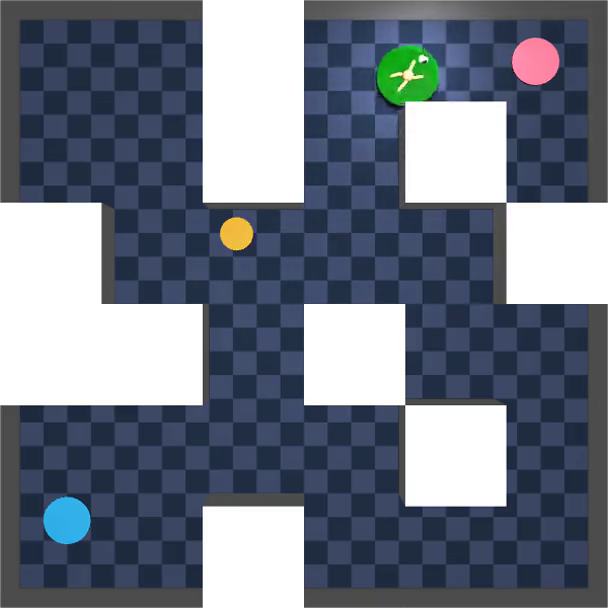} &
		\includegraphics[width=0.23\linewidth]{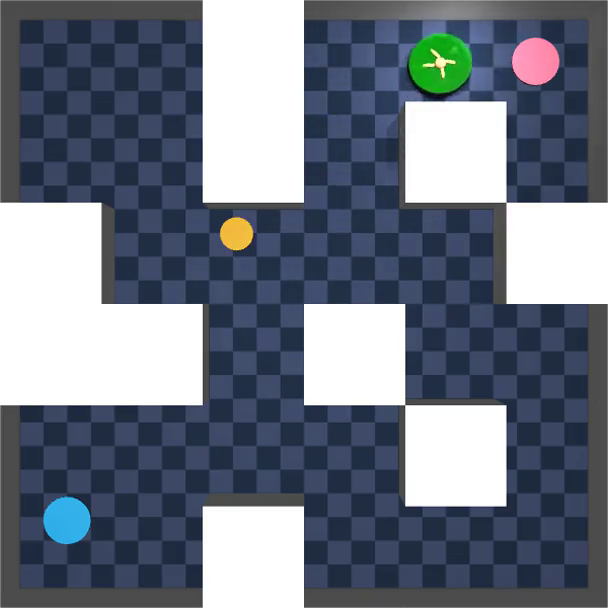} &
		\includegraphics[width=0.23\linewidth]{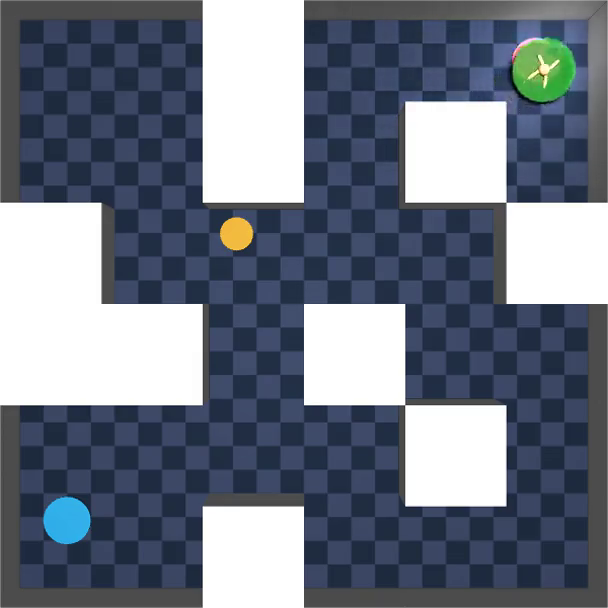} \\
	\end{tabular}
	\caption{\textbf{Visualization of RCD rollout execution on \texttt{AntSoccer-Medium-Stitch}.} The ant agent dribbles the soccer ball toward the goal location.}
	\label{fig:rollout_antsoccer_medium}
\end{figure}

\begin{figure}[h!]
	\centering
	\setlength{\tabcolsep}{1.5pt}
	\begin{tabular}{cccc}
		\includegraphics[width=0.23\linewidth]{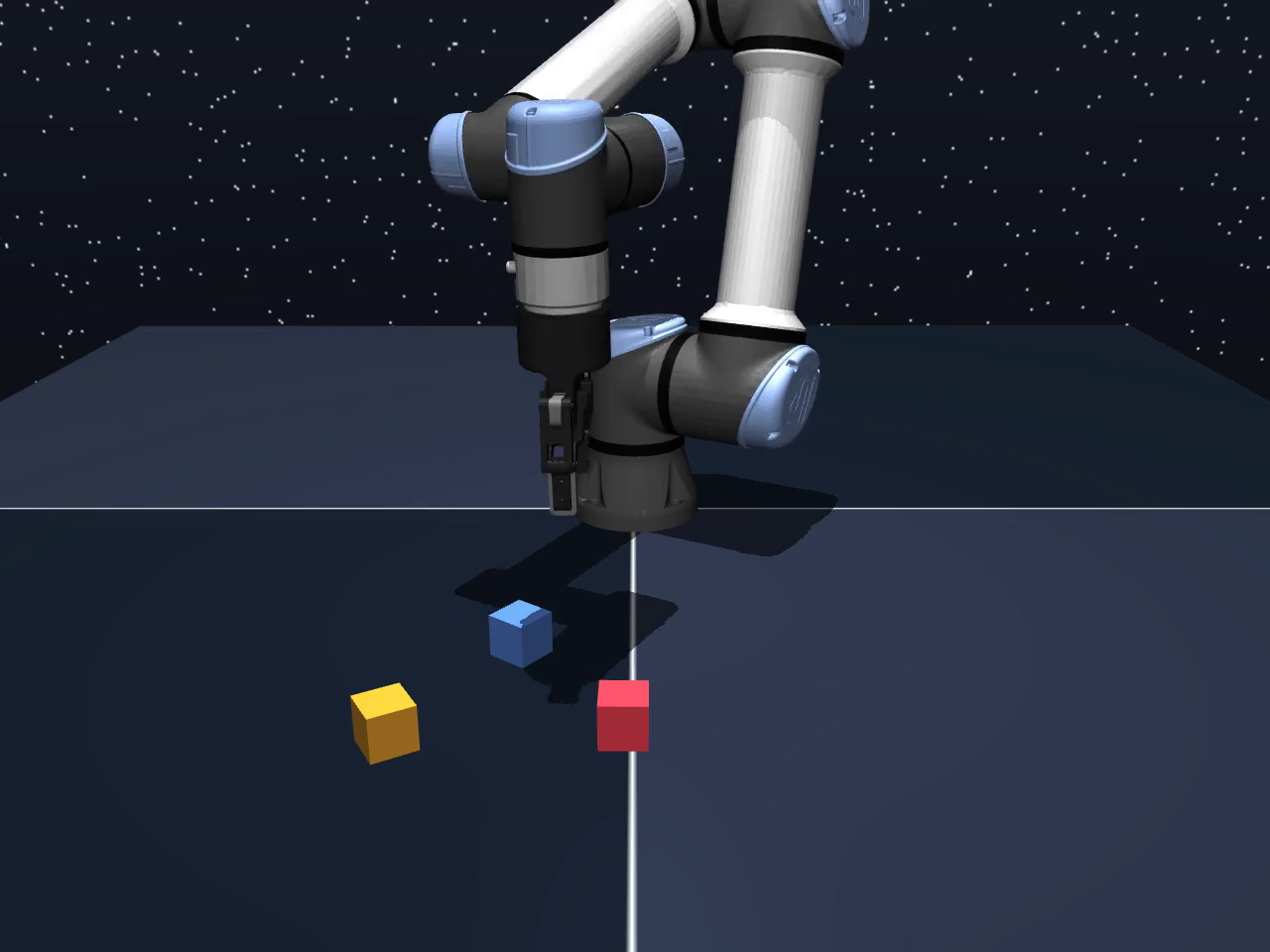} &
		\includegraphics[width=0.23\linewidth]{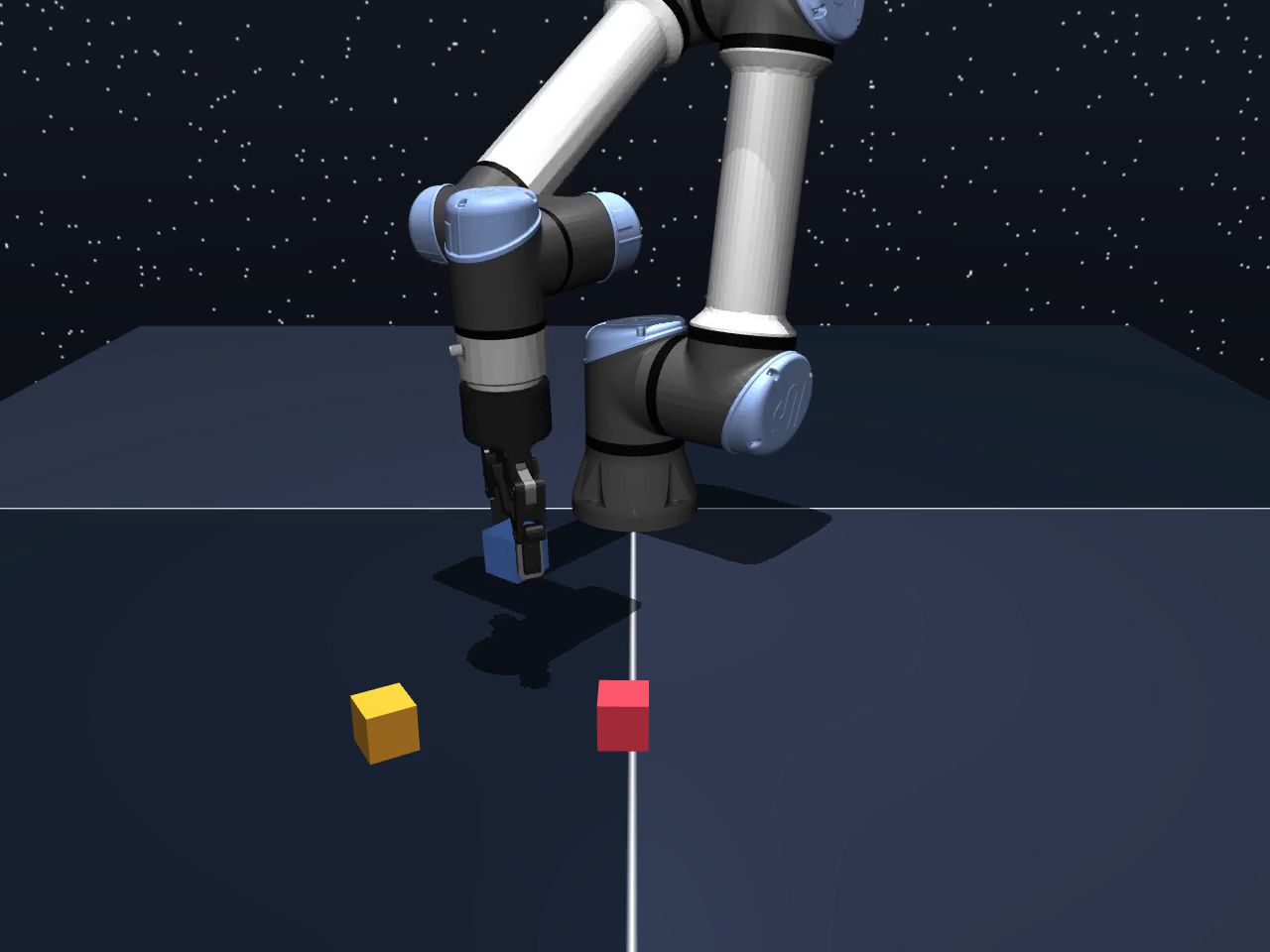} &
		\includegraphics[width=0.23\linewidth]{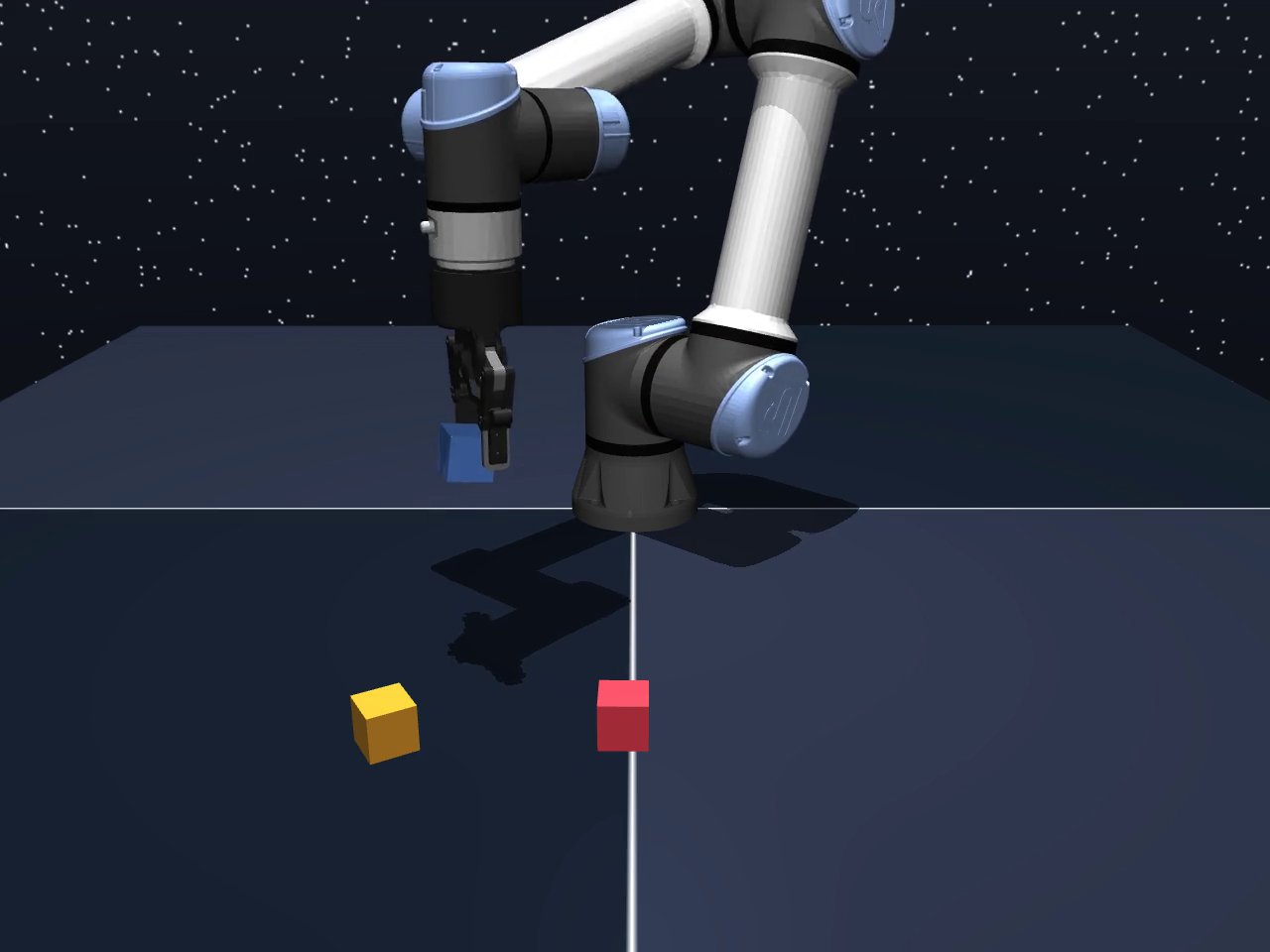} &
		\includegraphics[width=0.23\linewidth]{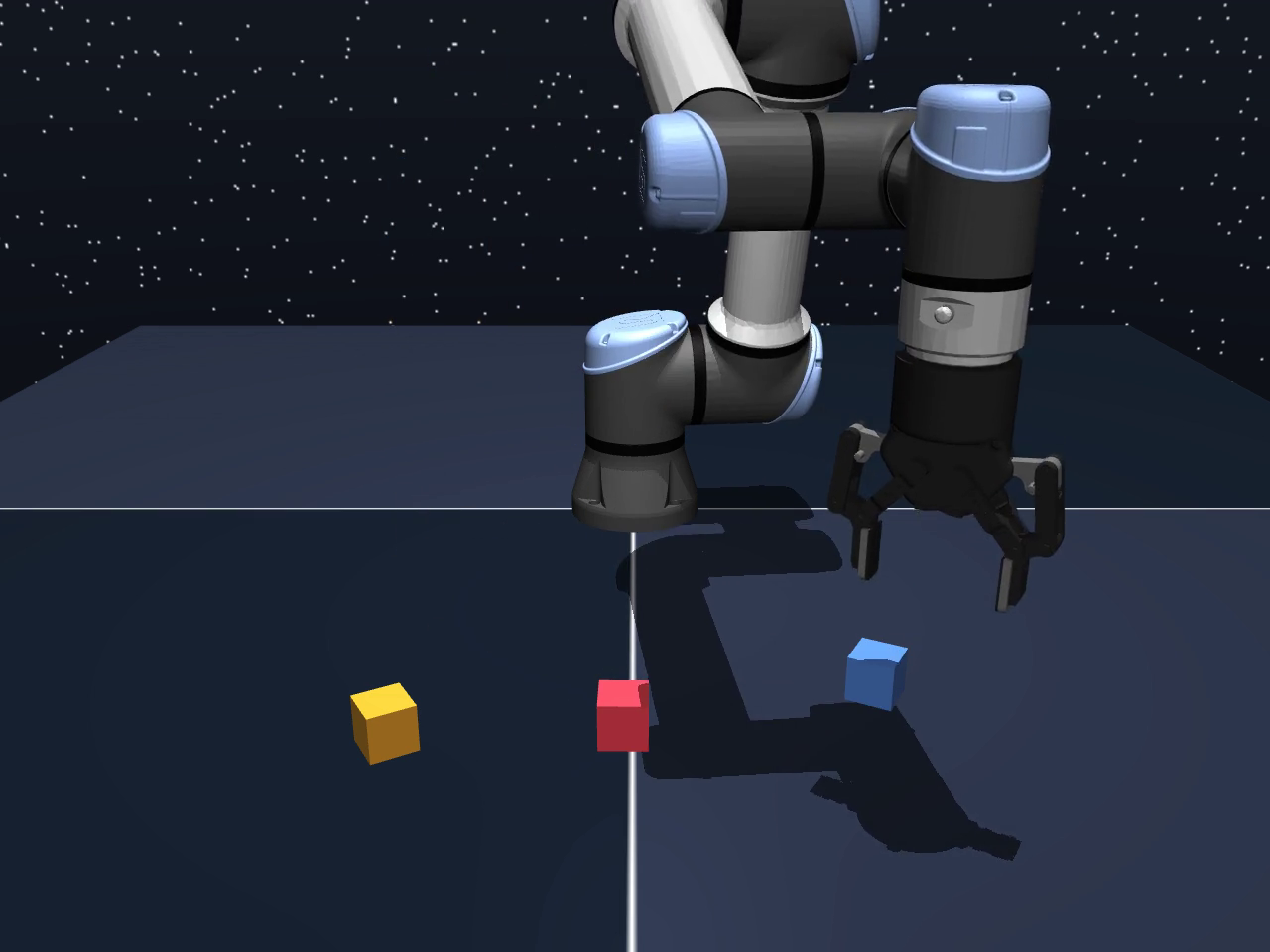} \\
		\addlinespace[1ex]
		\includegraphics[width=0.23\linewidth]{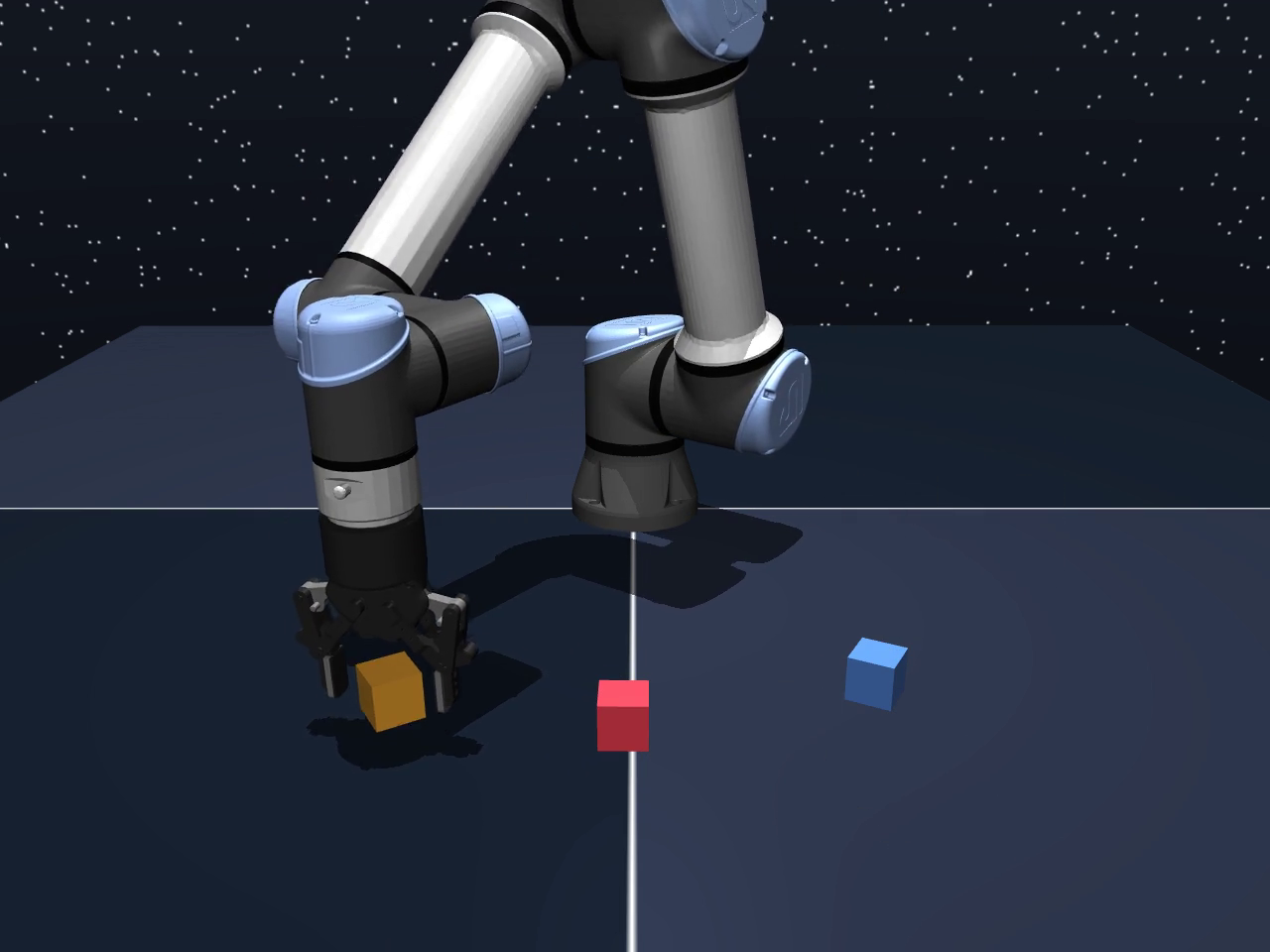} &
		\includegraphics[width=0.23\linewidth]{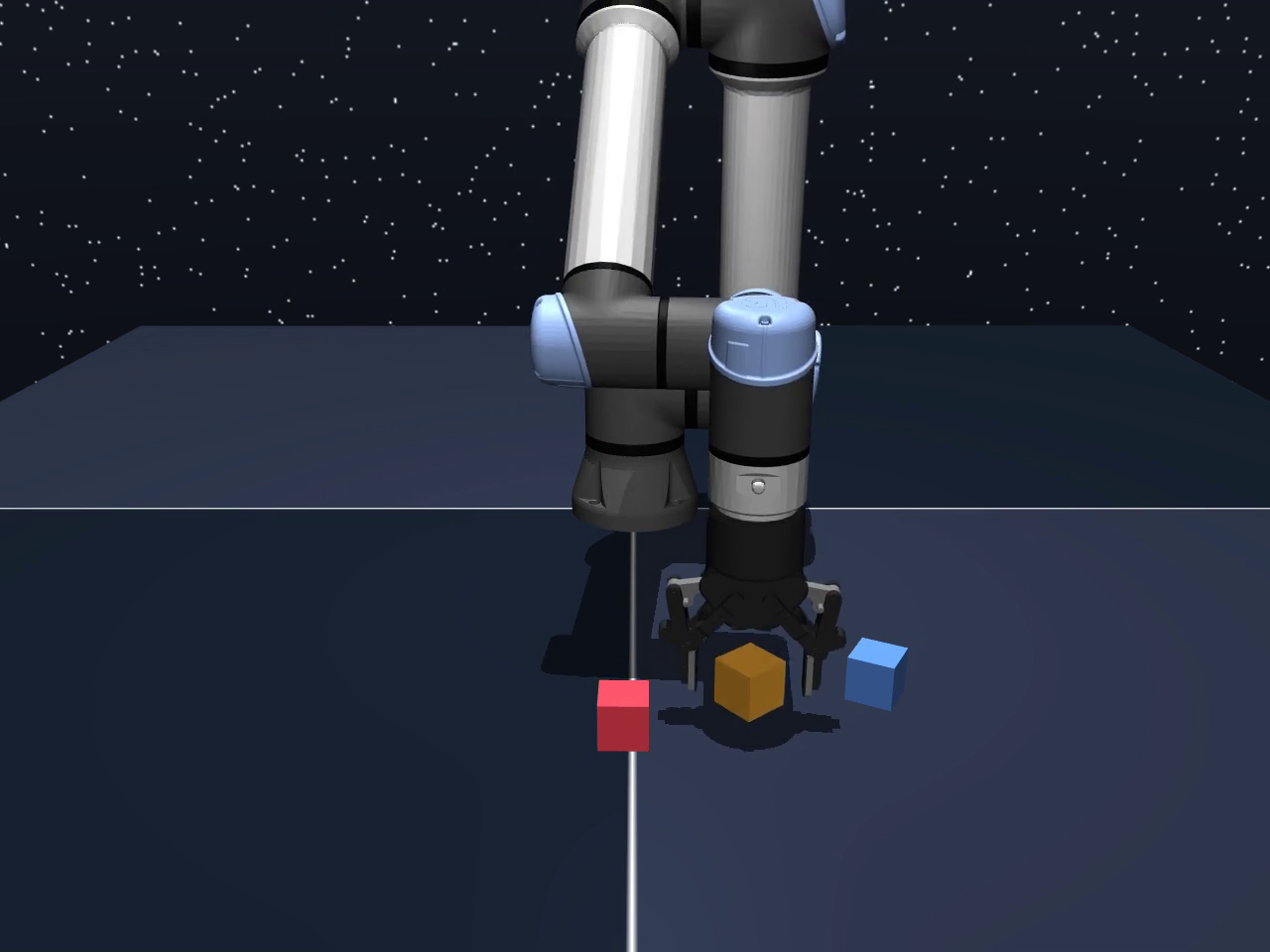} &
		\includegraphics[width=0.23\linewidth]{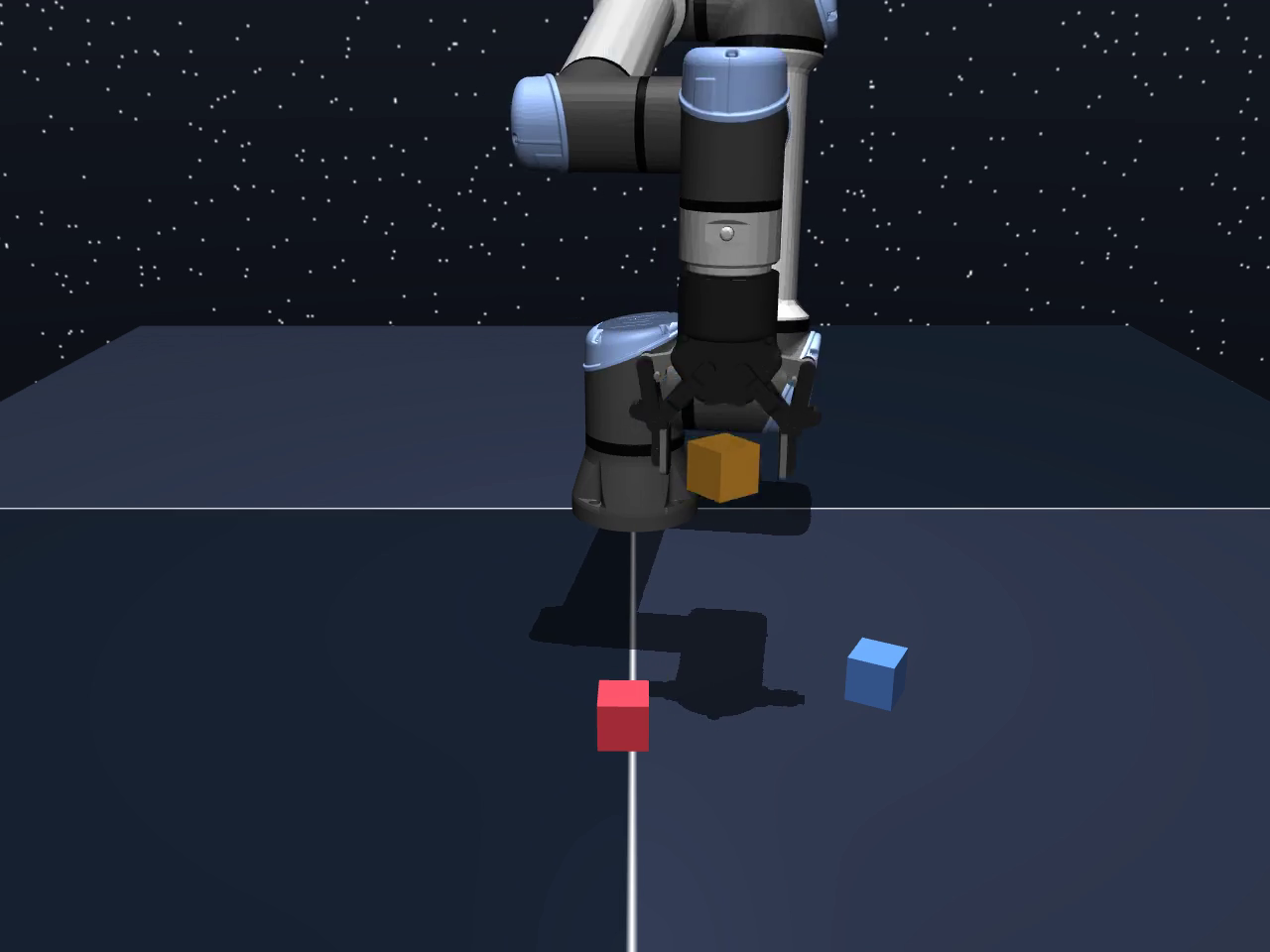} &
		\includegraphics[width=0.23\linewidth]{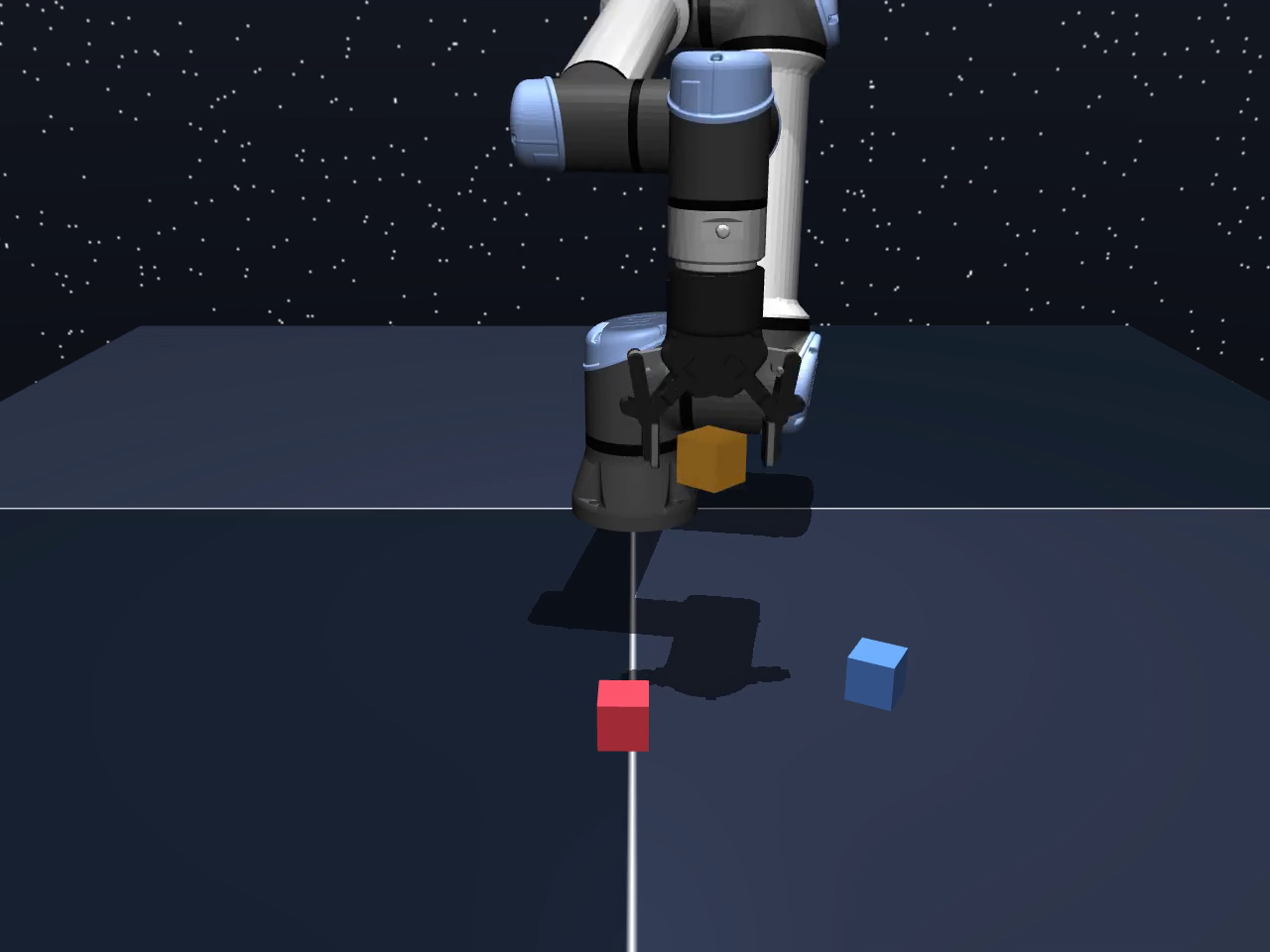} \\
		\addlinespace[1ex]
		\includegraphics[width=0.23\linewidth]{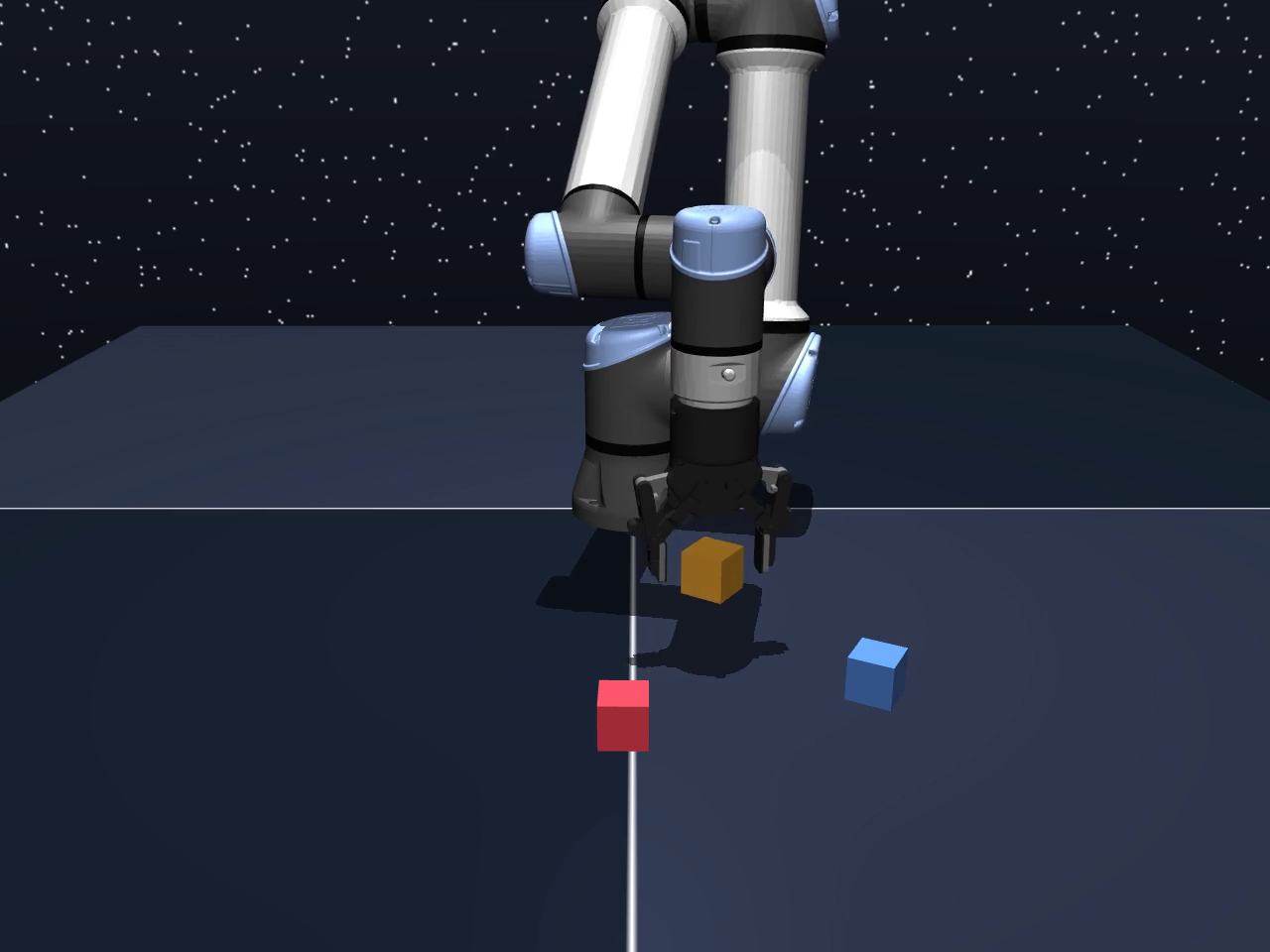} &
		\includegraphics[width=0.23\linewidth]{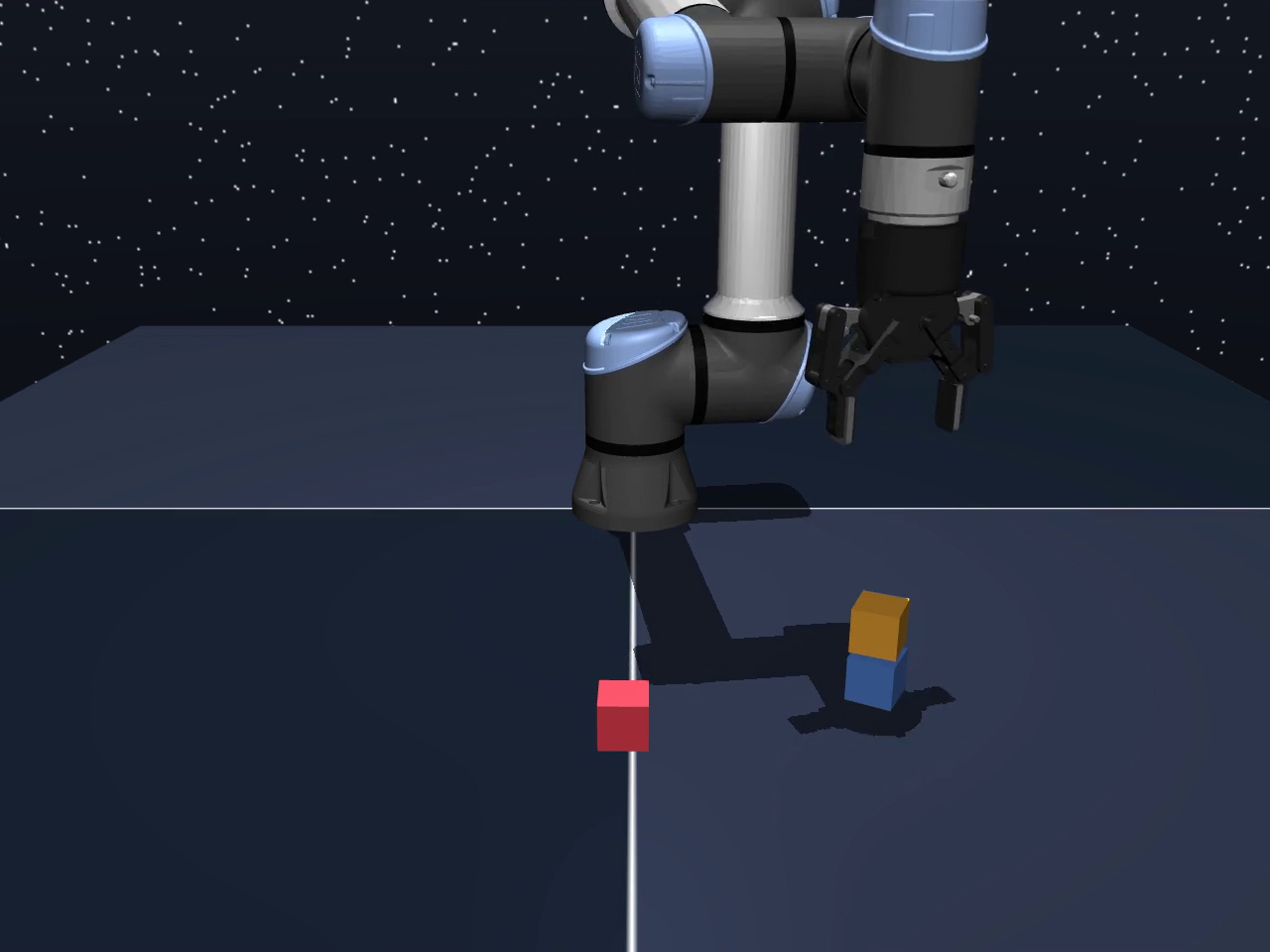} &
		\includegraphics[width=0.23\linewidth]{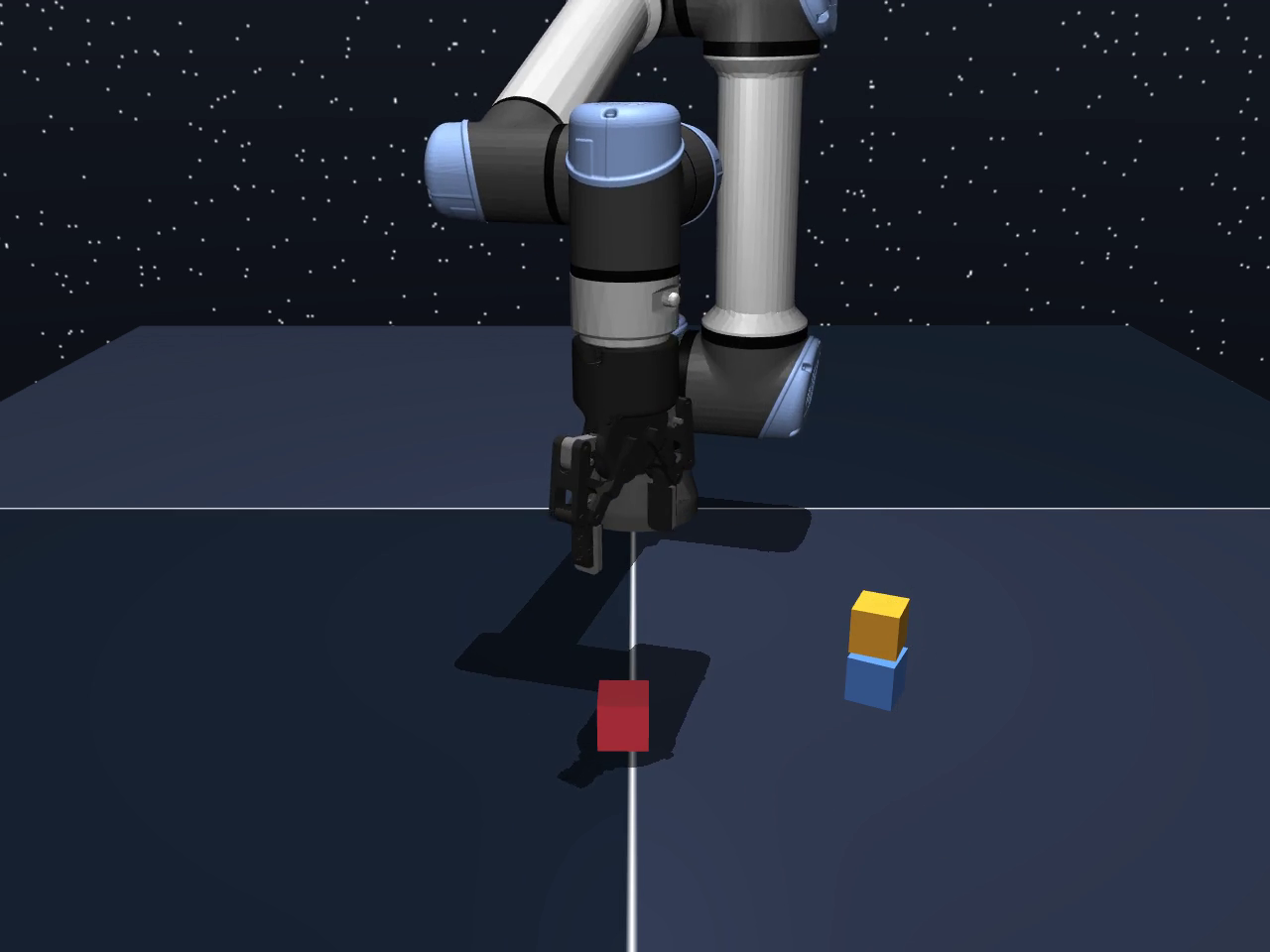} &
		\includegraphics[width=0.23\linewidth]{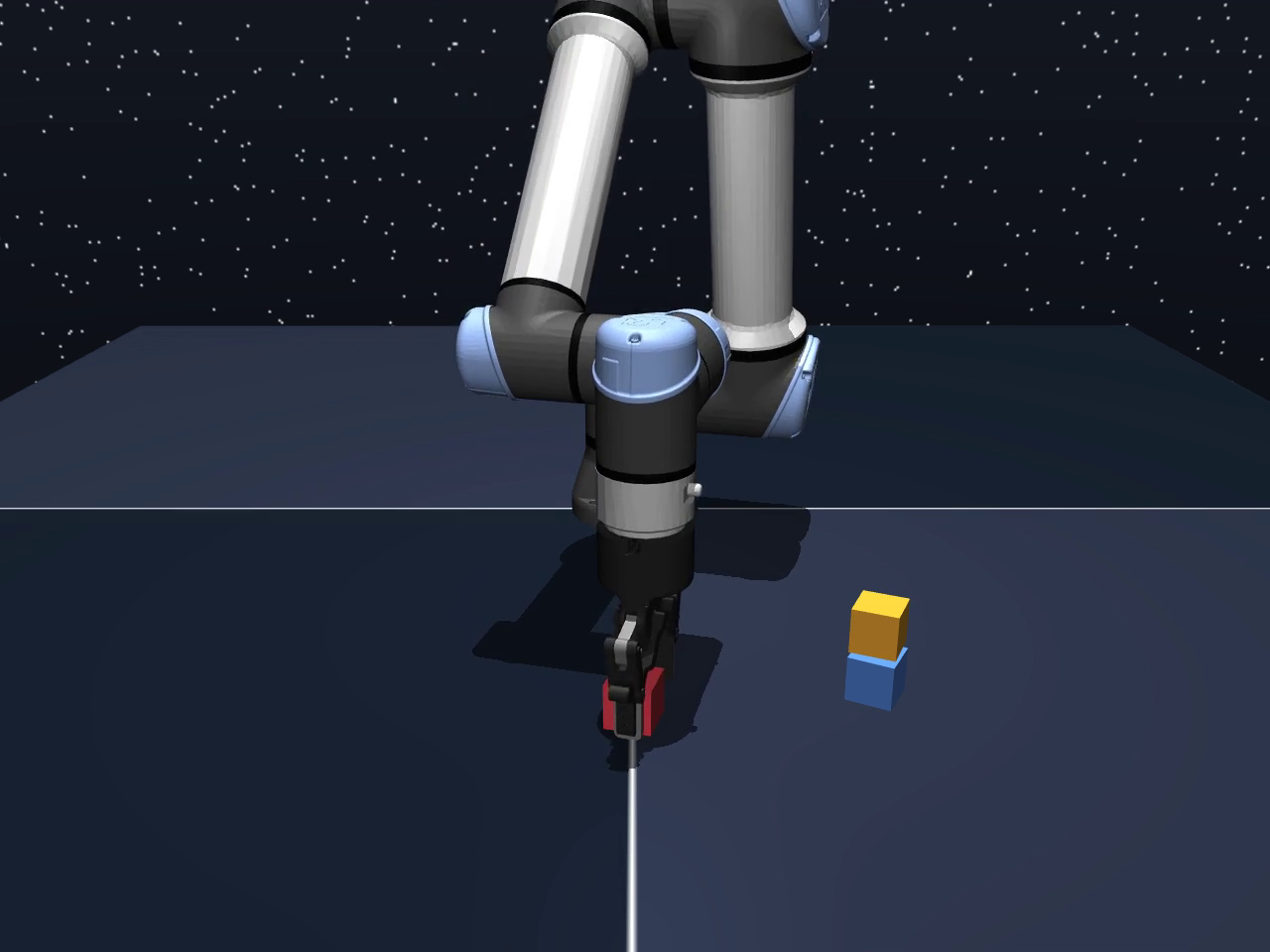} \\
		\addlinespace[1ex]
		\includegraphics[width=0.23\linewidth]{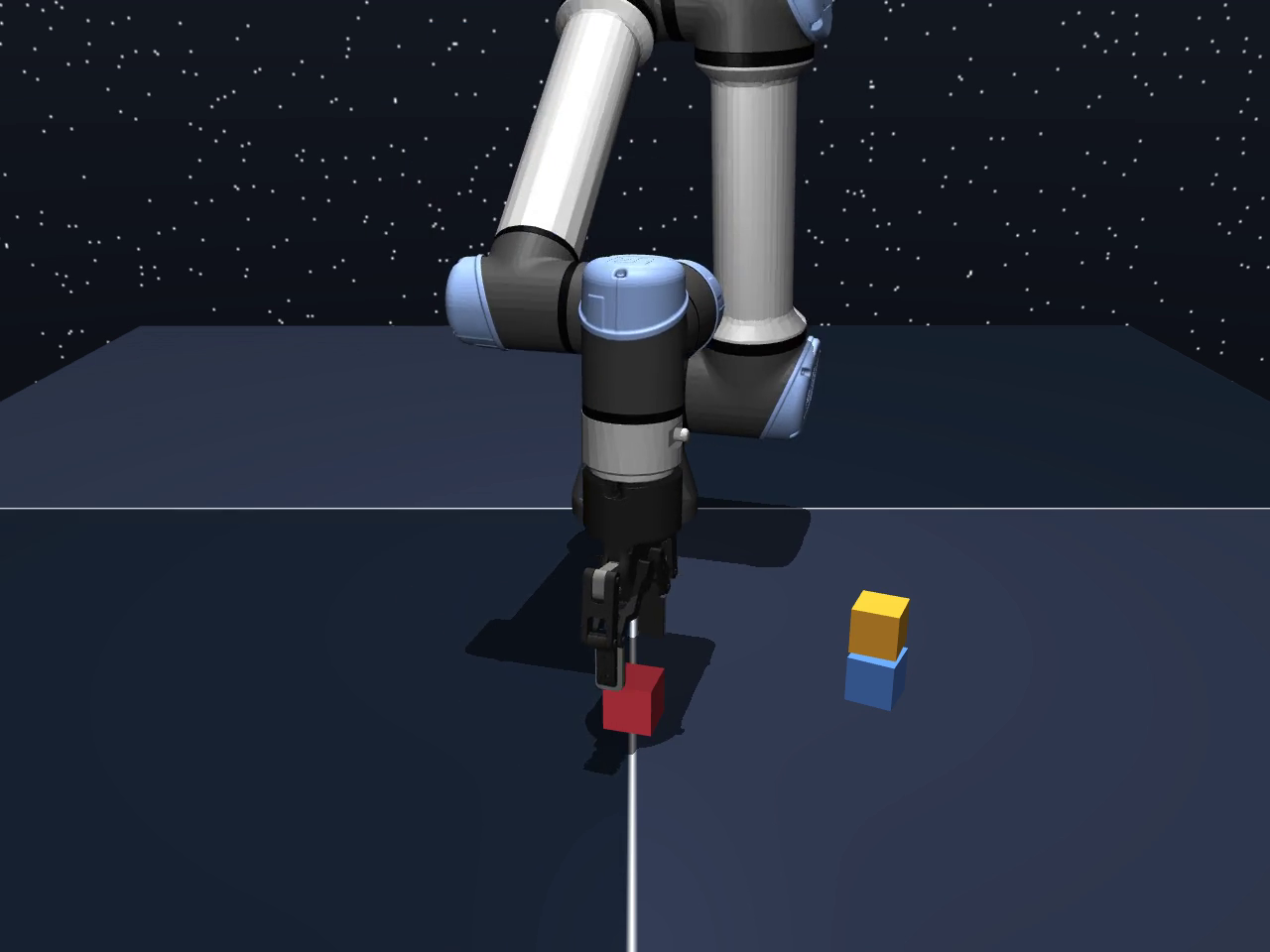} &
		\includegraphics[width=0.23\linewidth]{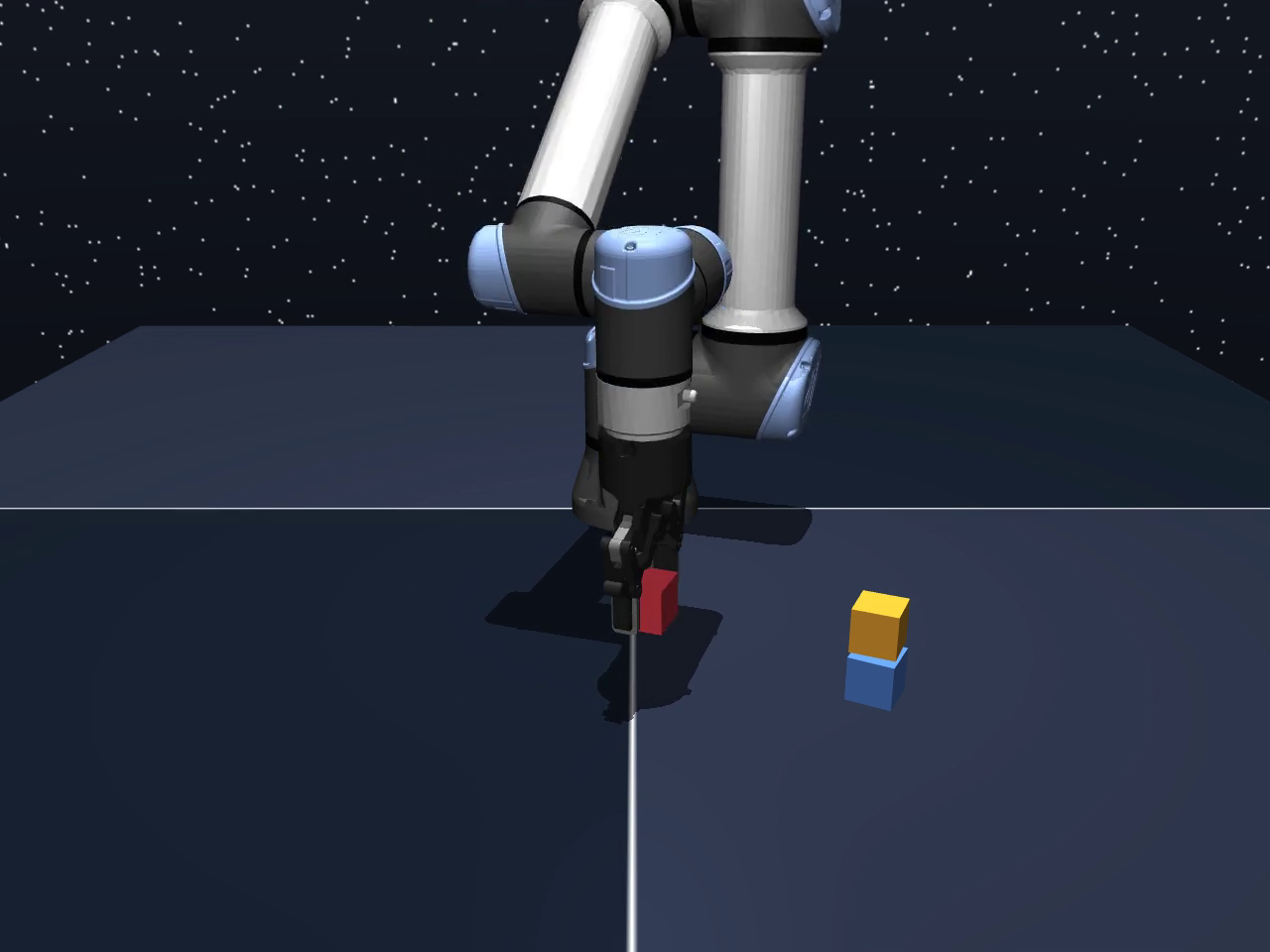} &
		\includegraphics[width=0.23\linewidth]{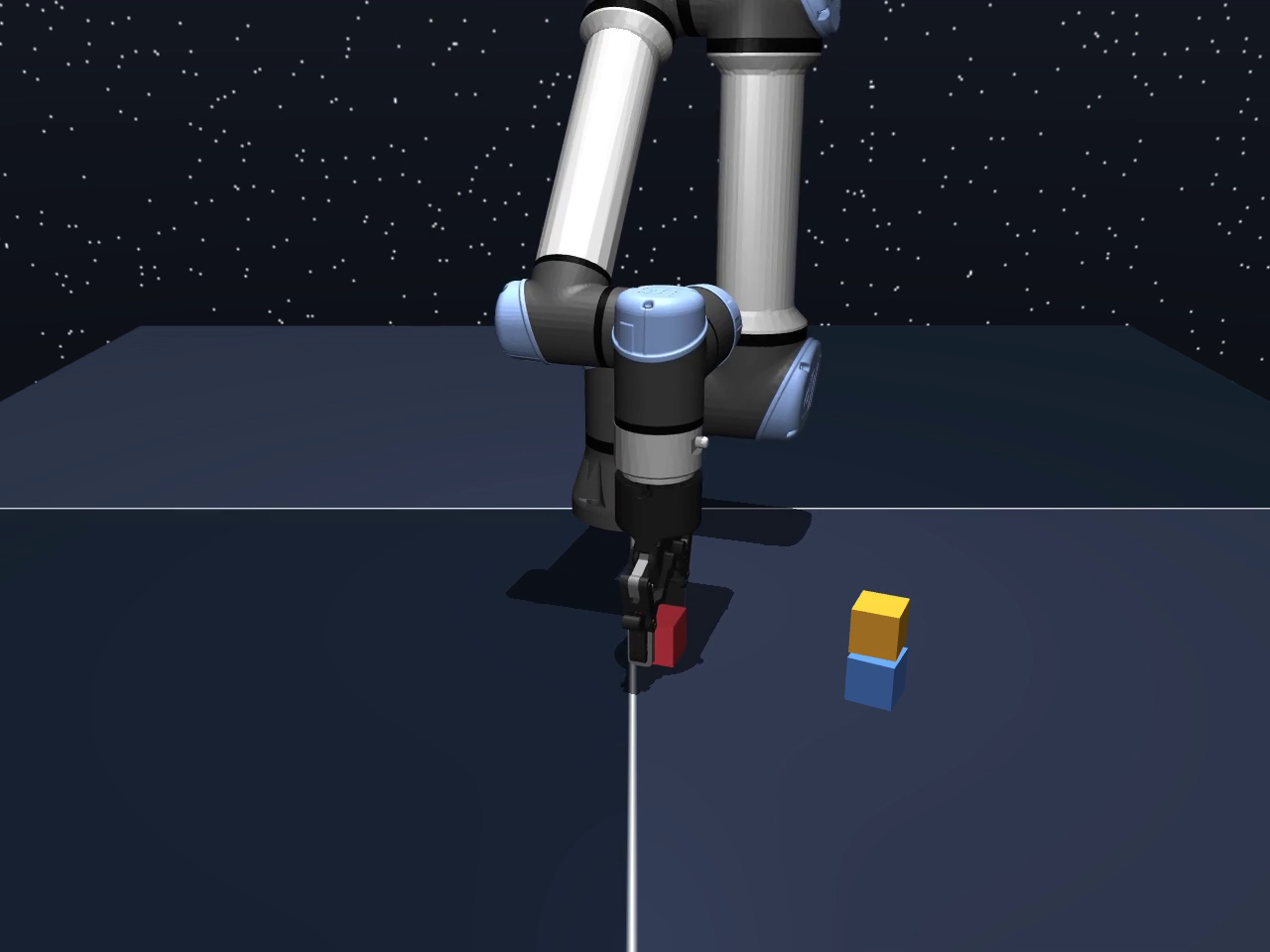} &
		\includegraphics[width=0.23\linewidth]{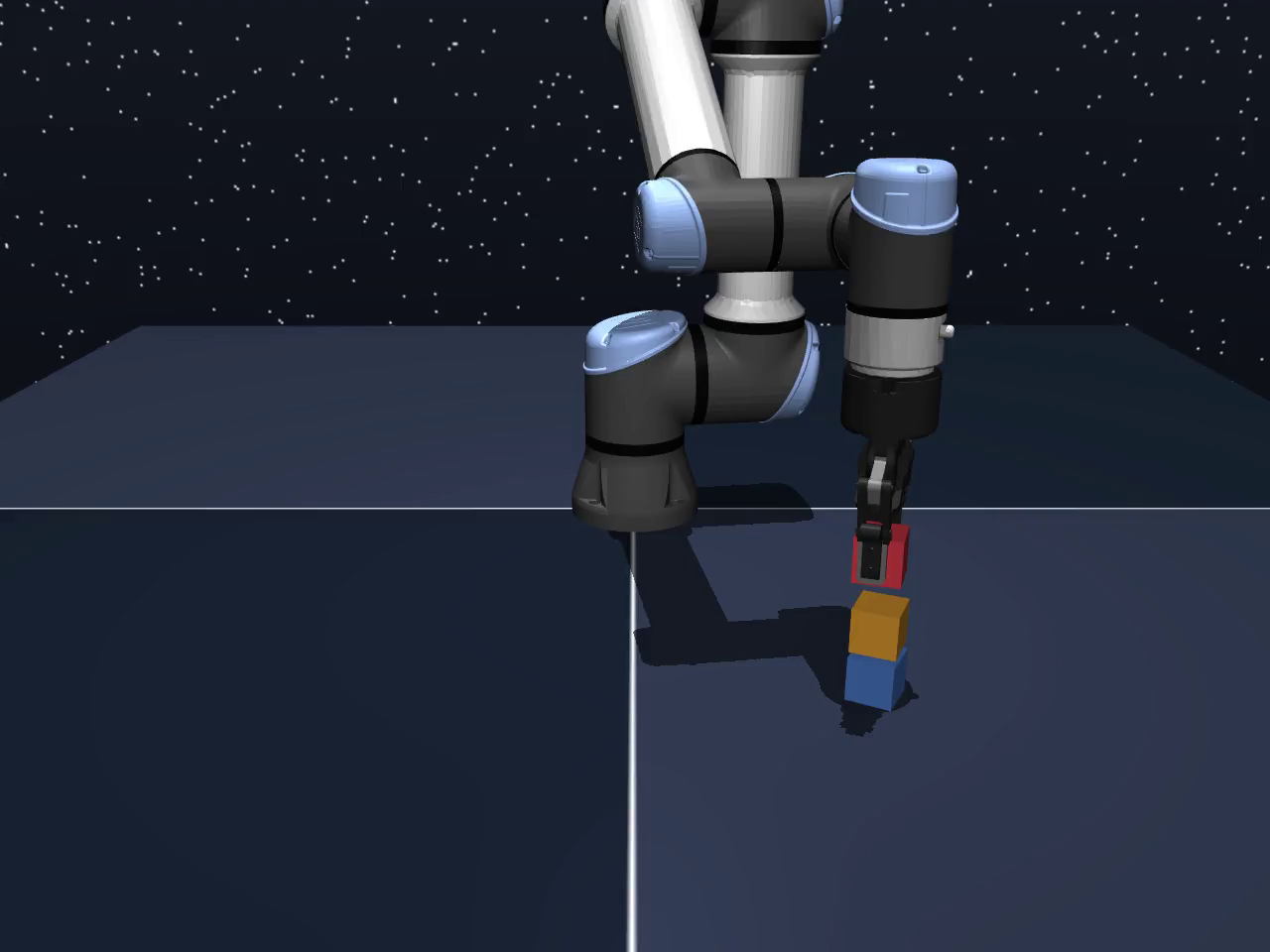} \\
	\end{tabular}
	\caption{\textbf{Visualization of RCD rollout execution on \texttt{Cube-Triple}.} The 6-DoF UR5e robot arm arranges three cubes into their target configuration via pick-and-place.}
	\label{fig:rollout_cube_triple}
\end{figure}

\begin{figure}[ht]
	\centering
	\includegraphics[width=0.98\linewidth]{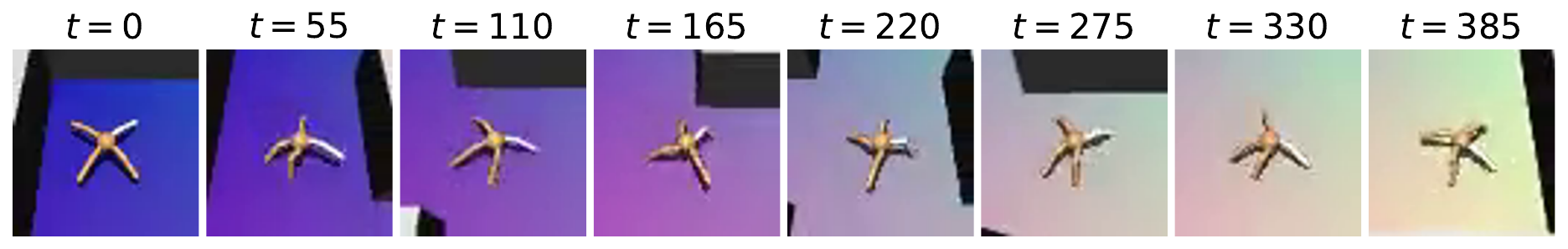}
	\caption{\textbf{Visualization of RCD rollout execution on \texttt{Visual-AntMaze-Medium-Stitch}.} The ant agent navigates from the starting region to the goal in a long-horizon task where start and goal are at opposite ends of the maze. Frames are the env-rendered $64{\times}64$ pixel observations. Planning is performed in a learned $16$-dimensional VAE latent space derived from these pixel renders.}
	\label{fig:rollout_visual_antmaze}
\end{figure}

%%%%%%%%%%%%%%%%%%%%%%%%%%%%%%%%%%%%%%%%%%%%%%%%%%%%%%%%%%%%
% \clearpage
% \newpage
% \input{checklist.tex}

\end{document}